\documentclass{article}

\usepackage[preprint]{neurips_2026}

\usepackage[utf8]{inputenc} 
\usepackage[T1]{fontenc}    
\usepackage{hyperref}       
\usepackage{url}            
\usepackage{booktabs}       
\usepackage{amsmath}
\usepackage{amssymb}
\usepackage{mathtools}
\usepackage{amsthm}
\usepackage{amsfonts}       
\usepackage{nicefrac}       
\usepackage{microtype}      
\usepackage{xcolor}         
\usepackage{graphicx}
\usepackage{svg}
\usepackage{enumitem}
\usepackage{subcaption}
\usepackage{thm-restate}
\usepackage{textgreek}
\usepackage{multirow}
\usepackage{algorithm}
\usepackage{algpseudocode}
\usepackage{xspace}
\usepackage[capitalize,noabbrev]{cleveref}

\usepackage[textsize=tiny]{todonotes}

\makeatletter
\DeclareRobustCommand\onedot{\futurelet\@let@token\@onedot}
\def\@onedot{\ifx\@let@token.\else.\null\fi\xspace}

\def\iid{{i.i.d}\onedot}
 
\def\ie{{i.e}\onedot}

\makeatother

\DeclareUnicodeCharacter{2212}{\ensuremath{-}}

\theoremstyle{plain}

\declaretheorem[name=Corollary]{corollary}
\declaretheorem[name=Lemma]{lemma}
\declaretheorem[name=Definition]{definition}
\declaretheorem[name=Assumption]{assumption}
\theoremstyle{remark}

\title{Matrix Factorization for Practical Continual Mean Estimation
Under User-Level Differential Privacy}

\author{%
  Nikita P.~Kalinin
  \thanks{Correspondence
to: Nikita P. Kalinin \texttt{nikita.kalinin@ist.ac.at}}
  \quad
  Ali Najar \thanks{This work was done while the author was an intern at ISTA.}
  \quad
  Valentin Roth$\;^\dagger$
  \quad
  Christoph H.~Lampert\\ 
  Institute of Science and Technology Austria (ISTA)
}

\begin{document}

\maketitle

\begin{abstract}
  We study continual mean estimation, where data vectors arrive sequentially and the goal is to maintain accurate estimates of the running mean. We address this problem under user-level differential privacy, which protects each user’s entire dataset even when they contribute multiple data points. Previous work on this problem has focused on pure differential privacy. While important, this approach limits applicability, as it leads to overly noisy estimates. In contrast, we analyze the problem under approximate differential privacy, adopting recent advances in the Matrix Factorization mechanism. We introduce a novel mean estimation specific factorization, which is both efficient and accurate, achieving asymptotically lower mean-squared error bounds in continual mean estimation under user-level differential privacy.
\end{abstract}

\section{Introduction}

The mean is one of the most fundamental statistical quantities and a cornerstone for more complex data analysis tasks. In many applications, data is collected from individuals (users), and accurate estimation of the mean provides essential insights into the underlying distribution. For example, medical institutions may track patient measurements across repeated visits, online platforms may monitor behavior over time, and federated learning systems often rely on mean statistics to aggregate model updates. In such settings, data typically arrives sequentially, requiring continual updates of the estimate as the stream evolves. Each user may contribute multiple data points, forming a dataset associated with that individual. Since these datasets often contain sensitive information, it is essential to ensure that estimates can be released without compromising privacy.

Differential Privacy (DP) \citep{dwork2006differential} is a principal framework for protecting privacy. In this work, we focus specifically on user-level privacy~\citep{amin2019bounding}, where the objective is to protect the entire dataset in relation to individual users' contributions, rather than protecting individual data points within a user's data. This scenario is highly relevant in sensitive domains such as collaborations between medical institutions, where direct sharing of raw, individual-level data is often restricted, and privatization is necessary.

User-level differential privacy (ULDP) has been actively studied, with prior work addressing both theoretical bounds and algorithmic design \citep{amin2019bounding, liu2020learning, liu2024user, chua2024mind, kato2024uldp}.
Within this line of research, mean estimation under ULDP has received particular attention, with recent results tightening bounds across a variety of domains \citep{narayanan2022tight, cummings2022mean, acharya2023discrete, pla2024distribution, ghazi2023user, agarwal2025private, du2025privacy}.

Sequential data corresponds to the continual release setting, which has been extensively studied
~\citep{dwork2010differential, bolot2013private, mcmahan2022federated, andersson2023smooth, andersson2024improved, edmonds2020power, cohen2024lower, kalinin2025continual}. 
To protect privacy in this setting, we will use a specific mechanism, the \emph{Matrix Factorization} mechanism \citep{li2015matrix}, which adds correlated Gaussian noise to all released weighted sums, reducing the overall error in the estimation problem.

Continual release with user-level differential privacy has been studied in the context of machine learning, including applications in federated learning where users participate repeatedly \citep{wei2021user, huang2023federated, charles2024fine, zhao2024learning, lowy2024faster, zhang2025personalized, kaiseruser}, as well as in continual histograms \citep{dong2023continual}.

Conceptually, the work closest to ours is Joint Moment Estimation (JME) \citep{kalinin2025continual}, which introduced the problem of mean estimation as a matrix factorization and defined the average workload, but used a suboptimal trivial factorization equivalent to the Gaussian mechanism for mean estimation and item-level differential privacy. A tangentially related approach is presented in \citet{choquette2023correlated}, which introduces a correlation matrix referred to as $\nu$-DP-FTRL but in a different setting, namely mean estimation as an instance of empirical risk minimization using the gradient descent method. We nevertheless found their proposed factorization to be highly competitive in our setting and therefore report a numerical comparison.

We also compare our approach with the study on continual mean release by \citet{george2024continual}, which focused on pure $\epsilon$-differential privacy and relied on modifications of Binary Tree mechanisms \citep{dwork2010differential, chan2011private} that require adding a relatively large amount of noise. In contrast, we build on the Gaussian mechanism in the form of the matrix factorization mechanism, which ensures Gaussian differential privacy and allows for significantly reduced noise, especially in the multi-dimensional setting. The matrix factorization mechanism \citep{denisov2022improved} has been shown to improve accuracy for continual estimation problems, such as model training, where gradients are computed sequentially by design. The problem of private mean estimation with matrix factorization was introduced in \citet{chen2024improved}, where it was studied under a sparsification constraint in the item-level differential privacy setting, making their results not directly comparable to ours.

Matrix factorization has also been studied in the ``multi-participation'' (or multi-epoch or user-level) setting \citep{choquette2023multi,choquette-choo2023amplified, kalinin2024banded, kalinin2026dp} for centralized and federated learning, where users or data points contribute updates repeatedly over time. We adapt this ``multi-participation'' factorization framework to the problem of user-level private mean estimation.

\textbf{Contributions}: 
Our main contribution in this work is to introduce \textbf{an efficient and accurate method for continual mean estimation under user-level differential privacy}, by using the Matrix Factorization mechanism in the ``multi-participation'' setting and tailoring it to the mean estimation problem.
Specifically, we design a new correlation matrix for the Matrix Factorization mechanism, specifically tailored to mean estimation, which yields a lower root mean square error (RMSE) compared to existing approaches. 
We establish theoretical guarantees, including RMSE bounds and high-probability concentration bounds for continual mean estimation in user-level differential privacy.

\section{Background} \label{sec:background}

Let $\mathbf{X} = (\mathbf{x}_1, \mathbf{x}_2, \dots, \mathbf{x}_n)^\top \in \mathbb{R}^{n\times d}$ be a private dataset.
In the continual‐release model, at each round \(t=1,\dots,n\), a new data vector \(\mathbf{x}_t\in\mathbb{R}^d\) arrives, and we must release a private estimate of the running mean
\begin{align}
  \mathbf{y}_t \;=\;\frac1t\sum_{j=1}^t \mathbf{x}_j =(\mathbf{AX})_t, \quad \text{where} \quad \mathbf{A}_{i,j} := 
  \begin{cases}
    \frac{1}{i}, & j \le i,\\
    0,         & j > i.
  \end{cases}
\end{align}
Rewriting the running means in a matrix form as $\mathbf{Y}=(\mathbf{y}_1,\dots,\mathbf{y}_n)\in\mathbb{R}^{n\times d}$, we have $\mathbf{Y} = \mathbf{A}\,\mathbf{X}$.

Let $\mathbf{X},\mathbf{X}'\in \mathbb{R}^{n\times d}$ be neighboring datasets that differ in the contribution of one user. In this work, we consider a ``replace with zero'' notion of neighboring\footnote{The ``replace with zero'' notion is almost semantically equivalent to the add/remove model \citep{ponomareva2023dp}, but it is more technically convenient, as we do not allow the size of the dataset to change.} \citep{erlingsson2020encode, denisov2022improved}, where sets $\mathbf{X}$, $\mathbf{X}'$ differ by the contribution of a single user.  It is further assumed that all  row vectors are bounded as $\|\mathbf{x}_j\|_2 \le \xi$.

Our goal is to make $\mathbf{Y} = \mathbf{AX}$ differentially private. Instead of adding noise to $\mathbf{AX}$ directly, we apply the \textit{Matrix Factorization mechanism} \citep{li2015matrix}. Consider an arbitrary factorization $\mathbf{A} = \mathbf{BC}$ for
$\mathbf{A},\mathbf{B},\mathbf{C} \in \mathbb{R}^{n\times n}$.
To privately estimate $\mathbf{Y}$, we first make the product $\mathbf{CX}$ private by adding Gaussian noise, and then multiply by the matrix $\mathbf{B}$ as a post-processing step. Formally,
\begin{align}
      \widehat{\mathbf{Y}}
      = \mathbf{B}\,(\mathbf{C}\,\mathbf{X} + \mathbf{Z})
      = \mathbf{A}\,\mathbf{X} + \mathbf{B}\,\mathbf{Z}, \label{eq:matrix_factorization_step}
\end{align}
where $\mathbf{Z}\sim \mathcal{N}(0,\sigma^2_{\varepsilon,\delta}\cdot \xi^2 \cdot \mathrm{sens^2(\mathbf{C})})^{n\times d}$ is an appropriately scaled Gaussian noise that makes $\mathbf{CX}$ private. 
The privacy guarantees follow from the Gaussian mechanism (see Lemma~\ref{lem:GaussianMechanism}), which ensures privacy for any multi-dimensional function with bounded $\ell_2$-sensitivity, and can be applied to the product $\mathbf{CX}$ with an appropriate notion of sensitivity.
\begin{restatable}[\citet{dwork2014algorithmic}]{lemma}{GaussianMechanism}\label{lem:GaussianMechanism}
Let $f:\mathcal{X}^n \to \mathbb{R}^d$ have $\ell_2$-sensitivity
$\Delta_2(f) := \sup_{x\sim x'} \|f(x)-f(x')\|_2$, where $x$ and $x'$ differ in one individual's data.
The mechanism
\begin{equation}
\mathcal{M}(x)=f(x)+Z,\qquad Z\sim \mathcal{N}(0,\sigma^2 I_d)
\end{equation}
satisfies $(\varepsilon,\delta)$-differential privacy for $\varepsilon,\delta\in(0,1)$ if
\begin{equation}
\sigma \ge \frac{\Delta_2(f)\,\sqrt{2\ln(1.25/\delta)}}{\varepsilon}\, .
\end{equation}
\end{restatable}

The term $\mathrm{sens}(\mathbf{C})$ is the \textit{global sensitivity} for the Gaussian mechanism with unit norm $\xi=1$ neighboring data $\mathbf{X}\sim \mathbf{X}'$, which is computed as:
\begin{align}
  \mathrm{sens}(\mathbf{C}) 
  &= \max_{\mathbf{X}\sim\mathbf{X}'}\|\mathbf{C}(\mathbf{X} - \mathbf{X}')\|_{F}.
\end{align}
In the case of single participation, or item-level privacy, this reduces to
\begin{align}
     \mathrm{sens}(\mathbf{C})  &= \max_{j} \sup_{\|\mathbf{x}_j\|_2 \le 1} \| \mathbf{C}_{:,j}\|_2 \cdot \|\mathbf{x}_j\|_2
  =  \max_j \| \mathbf{C}_{:,j}\|_2 =: \| \mathbf{C} \|_{1\to2},
\end{align}
\noindent \ie the largest $\ell_2$-column norm of $\mathbf{C}$.

In this work, we allow each user to participate in more than one data entry. However, without further restrictions, the contribution of a single user could be large, making it impossible to obtain meaningful utility guarantees.
Following \citet{choquette-choo2023amplified}, we adopt the $b$-min-separation condition, where any two participations of a single user are at least $b$ steps apart. 
Intuitively, this separation is realistic, since one can simply delay releasing a single user’s data until at least $b$ other samples have been observed, while returning the previous estimate at a new step. We set $k=\lceil \frac{n}{b}\rceil$ which denotes the maximum number of times a user can participate. 
The sensitivity in this case is bounded by \citet{choquette-choo2023amplified} as
\begin{equation}\label{eq:sens_bound}
\mathrm{sens}_{k,b}(\mathbf{C})
\;\le\;
\max_{\pi\in\Pi_{k,b}}
\sqrt{\sum_{i,j\in\pi}\bigl|(\mathbf{C}^\top \mathbf{C})_{i,j}\bigr|}
\end{equation}
where $\Pi_{k,b}$ denotes the family of subsets $\pi \subseteq \{1,\dots,n\}$ with $|\pi|\le k$, such that any two distinct elements $i,j \in \pi$ satisfy $|i-j|\ge b$.  
Moreover, if every entry of $\mathbf{C}^\top \mathbf{C}$ is nonnegative, then the inequality in \eqref{eq:sens_bound} is tight. 
For lower triangular Toeplitz matrices with decreasing non-negative elements, the sensitivity can be computed exactly. 

\begin{restatable}[Theorem~2 from \citet{kalinin2024banded}]{theorem}{sensTheorem}\label{thm:sensTheorem}
Let $\mathbf{C}$ be a lower triangular Toeplitz matrix with decreasing non-negative entries $c_0\ge c_1\ge \cdots \ge c_{n-1} \ge 0$. Then the sensitivity in the setting of $b$-min-separation is 
\begin{align}
    \mathrm{sens}_{k,b} (\mathbf{C}) = \Big\| \sum_{j=0}^{k-1} \mathbf{C}_{:,\,jb + 1} \Big\|_2 
\end{align}
  where $\mathbf{C}_{:,\,jb + 1}$ denotes the $(1+jb)$-th column of $\mathbf{C}$.
\end{restatable}

Given a lower triangular factorization of $\mathbf{A}=\mathbf{BC}$, we quantify the mean expected squared error\footnote{We generalize the RMSE of \citet{choquette2023multi} to an arbitrary intermediate step $t$, and not just the last $n$.
} up to time $t$ as
\begin{equation}
\begin{aligned}
    \mathbb{E}\left[ \frac{1}{t}\|\mathbf{Y}_{:t} - \widehat{\mathbf{Y}}_{:t}\|^2_F \right] &= \mathbb{E}\left[\frac{1}{t}\|\mathbf{B}_{:t, :t}\mathbf{Z}_{:t}\|_F^2\right] =  \frac{1}{t}\|\mathbf{B}_{:t}\|^2_F \cdot \mathrm{sens}^2_{k,b}(\mathbf{C}) \cdot \xi^2 \cdot \sigma^2_{\varepsilon, \delta} \cdot d .
\end{aligned}
\end{equation}
The term $\sigma^2_{\varepsilon, \delta} \cdot \xi^2 \cdot d$ does not depend on the factorization, so to find the factorization we need to make the product $\|\mathbf{B}_{:t}\|^2_F \cdot \mathrm{sens}^2_{k,b}(\mathbf{C})$ as small as possible. Following the line of works of \citet{li2015matrix, choquette2023multi, kalinin2024banded}, we define the error that we aim to minimize as:
\begin{align}
    \mathcal{E}_t(\mathbf{B},\mathbf{C}) := \frac{1}{\sqrt{t}} \|\mathbf{B}_{:t}\|_F \cdot \mathrm{sens}_{k,b}(\mathbf{C})
    .
\end{align}
Thus, we reduce the problem of private continual mean estimation to finding matrices $\mathbf{B},\mathbf{C}\in \mathbb{R}^{n\times n}$ such that $\mathbf{A} = \mathbf{BC}$ to minimize $\mathcal{E}_t(\mathbf{B},\mathbf{C})$.

\vspace{0.5cm}
\section{Private User Level Mean Estimation with Matrix Factorization}

\begin{figure}[t]
    \centering
    \includegraphics[width=0.5\linewidth]{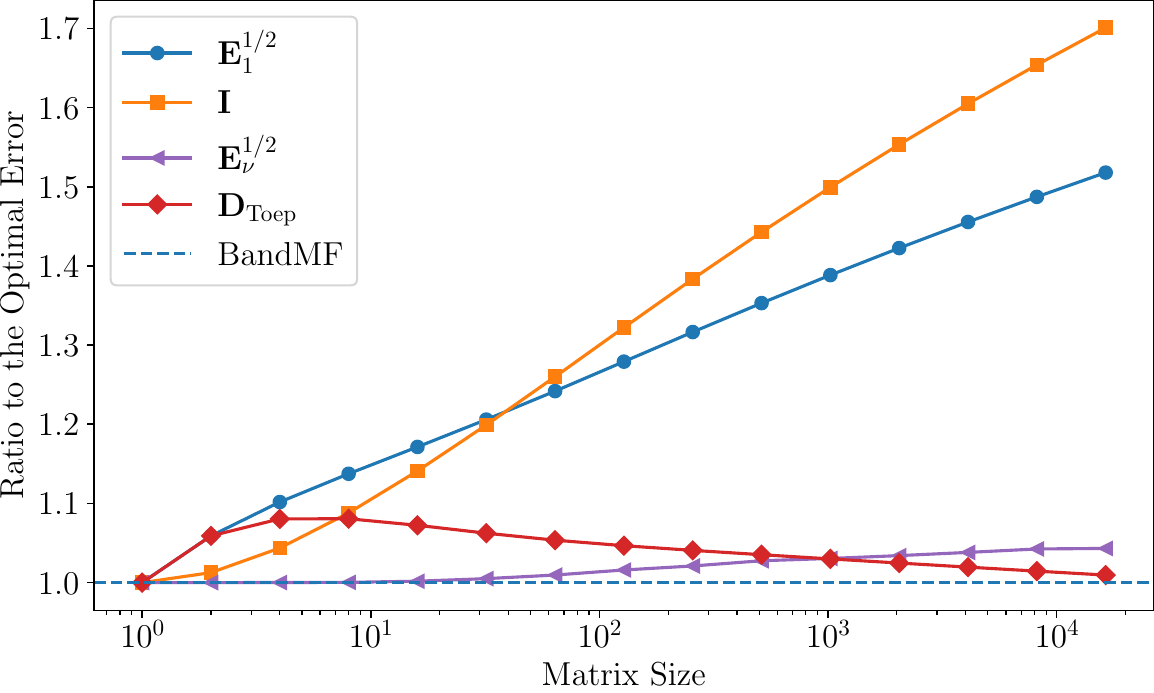}
    \caption{RMSE at step $n$ of different factorizations in the single-participation (item-level privacy) setting, presenting the error ratios of the four best-performing factorizations relative to BandMF \citep{mckenna2024scaling}.}
    \label{fig:single_rmse}
\end{figure}

In this work, we propose novel mean-specific factorizations of $\mathbf{A}$ to minimize the RMSE error. Before we present it, let us define two matrices that we will extensively use in this work. We define the diagonal normalization matrix $\mathbf{D}$ and the prefix-sum matrix $\mathbf{E}_1$ by
\begin{align}
\label{eq:D_E1_matrices}
    D_{i,j} &:= \begin{cases}
        \frac{1}{i}, & j = i, \\
        0, & j \neq i,
    \end{cases}
    &
    (E_1)_{i,j} &:= \begin{cases}
        1, & j \le i, \\
        0, & j > i.
    \end{cases}
\end{align}
The prefix-sum matrix, used together with $\mathbf{D}$, forms the running-averages matrix $\mathbf{A} = \mathbf{D}\mathbf{E}_1$.

In the work of \citet{choquette2023correlated}, the authors introduce an exponentially decayed version of $\mathbf{E}_{1}^{1/2}$, referred to as $\nu$-DP-FTRL with the correlation matrix $\mathbf{E}_{\nu}^{1/2}$, where the matrix $\mathbf{E}_{\nu}$ is given by:
\begin{align}
    (E_\nu)_{i,j} &:= \begin{cases}
        (1 - \nu)^{i - j}, & j\le i, \\
        0, & j>i.
    \end{cases}
\end{align}

For our proposed factorization, we define a lower-triangular Toeplitz matrix $\mathbf{D}_{\mathrm{Toep}}$, whose diagonals are determined by the entries of $\mathbf{D}$:
\begin{align}
\label{eq:D_Toep_matrix}
(D_\mathrm{Toep})_{i,j} &:=
\begin{cases}
\dfrac{1}{i-j+1}, & j \le i, \\
0, & j > i.
\end{cases}
\end{align}
We will use $\mathbf{C} = \mathbf{D}_{\mathrm{Toep}}$ as the correlation matrix for single-participation (item-level) differential privacy. In the more general multi-participation setting, to account for the $b$-min-separation sensitivity~\eqref{eq:sens_bound}, 
we consider a banded-inverse variant of the matrix $\mathbf{D}_{\mathrm{Toep}}$, following \citet{kalinin2025back}. 
Specifically, we set to zero all elements below the $p$-th diagonal in $\mathbf{D}_{\mathrm{Toep}}^{-1}$ 
and then invert the resulting matrix. 
We denote this matrix by $\mathbf{C} = \mathbf{D}_{\mathrm{Toep}}^{p}$, 
where $p$ is a hyperparameter that controls the trade-off between utility and memory consumption.

\begin{definition}[Mean-Aware Matrix Factorization]
Let $\mathbf{D}_{\mathrm{Toep}}$ be the Toeplitz matrix defined in~\eqref{eq:D_Toep_matrix}.
For a bandwidth parameter $p > 0$, the banded-inverse mean-aware factorization $\mathbf{C} = \mathbf{D}_{\mathrm{Toep}}^{p}$ is defined by its inverse:
\begin{equation}
 (\mathbf{D}_{\mathrm{Toep}}^{p})^{-1} = 
\begin{cases}(D_{\mathrm{Toep}}^{-1})_{i,j}, & 0 \le i-j < p,\\[2pt]
0, & \text{otherwise},
\end{cases}
\end{equation}
where $(D_\mathrm{Toep})_{i,j} = \frac{1}{i-j+1}$  if $j \le i$ and $0$ otherwise. 
\end{definition}

We compute the coefficients of the inverse matrix $\mathbf{D}_{\mathrm{Toep}}^{-1}$ in the following lemma:

\begin{restatable}{lemma}{DToepInverse}
\label{lem:dtoep-inverse}
The Lower-Triangular Toeplitz (LTT) matrix
\(\mathbf{D}_{\mathrm{Toep}}^{-1}\) is given by:
\begin{align}
    \mathbf{D}_{\mathrm{Toep}}^{-1}
    = \begin{pmatrix}
        1        & 0        & \cdots & 0 \\
        g_{1}    & 1        & \cdots & 0 \\
        \vdots   & \vdots   & \ddots & \vdots \\
        g_{n-1}  & g_{n-2}  & \cdots & 1
      \end{pmatrix},
\end{align}
where \(g_j=-|G_j|\) for \(j\ge 1\), and \(G_j\) is the \(j\)-th Gregory coefficient \citep{Gregory1841_CollinsLetter}.
\end{restatable}

Using the identities and bounds for Gregory coefficients from \citet{BLAGOUCHINE2016404}, we state that
\begin{align}
\label{eq:greg_first_prop}
    \frac{1}{j\ln^2 j} - \frac{2}{j\ln^3 j} \le |G_j| \le \frac{1}{j\ln^2 j} - \frac{2\gamma}{j\ln^3 j}, \qquad \text{and} \qquad \sum_{j=1}^\infty |G_j|  = 1,
\end{align}
where the first inequality holds for $j \ge 5$ and \(\gamma\) is the Euler-Mascheroni constant.
Thus, for the inverse matrix  $\mathbf{D}^{-1}_{\mathrm{Toep}}$ , we obtain tight bounds on both the coefficients and their partial sums, which is a necessary step to compute the factorization error, allowing us to derive bounds on the RMSE of this factorization.

We summarize the complete mean-estimation procedure, including the computation of the Gregory coefficients, in Algorithm~\ref{alg:continual-mean-estimation}, which is deferred to the appendix for space reasons. In the following subsections, we compare the proposed mean-aware factorization with several baselines and derive the RMSE in both single- and multi-participation settings.

\begingroup
\everymath{\displaystyle}
\begin{table}[t]
  \centering
  \caption{Upper bounds of multi-participation RMSE for selected factorizations at step $t$.}
  \label{tab:multi_part_results}
  \setlength{\tabcolsep}{6pt}
  \renewcommand{\arraystretch}{1.2}
  \begin{tabular}{@{\hspace{5pt}} l @{\hspace{40pt}} l @{\hspace{40pt}} l @{\hspace{0pt}}}
    \toprule
    \multirow{2}{*}{$\mathbf{C}$} &
      \multicolumn{2}{c}{\hspace{-50pt}\textbf{$\mathcal{E}_t(\mathbf{B},\mathbf{C})$}} \\
    & \hspace{25pt}\textbf{Not Banded} & \hspace{25pt}\textbf{Banded Inverse} \\
    \midrule
    $\mathbf{E}_1^{1/2}$
      & $ \Theta\left(\sqrt{\frac{k(\ln n + k)}{t}}\right)$ & $O\left(\sqrt{\frac{k \ln(n) \ln \ln (n/k)}{t\ln (n/k)}} + \frac{k}{\sqrt{nt}}\sqrt{\ln n}\right)$ \\ [13pt]
    $\mathbf{I}$
      & $ \Theta\left(\sqrt{\frac{k\ln t}{t}}\right)$              & $\Theta\left(\sqrt{\frac{k\ln t}{t}}\right)$ \\ [13pt]
    $\mathbf{D}_{\mathrm{Toep}}$
      & $O\left(\sqrt{\frac{k}{t}} + \frac{k}{\sqrt{nt}}\sqrt{\ln k \ln n}\right)$ & $O\left(\sqrt{\frac{k}{t}} + \sqrt{\frac{k \ln k}{t \ln^2 (n/k)}}\right)$ \\ [10pt]
    \bottomrule
  \end{tabular}
\end{table}
\endgroup
\subsection{Single Participation}

We first consider a simple setting in which each user can participate only once, corresponding to item-level differential privacy. We numerically compare the proposed factorization,
$\mathbf{C} = \mathbf{D}_{\mathrm{Toep}}$,
with several explicit factorizations from the literature: input perturbation, given by $\mathbf{C} = \mathbf{I}$; the square-root factorization, $\mathbf{C} = \mathbf{E}_1^{1/2}$~\citep{fichtenberger2023constant, kalinin2025continual}; and $\nu$-DP-FTRL, $\mathbf{C} = \mathbf{E}_{\nu}^{1/2}$~\citep{choquette2023correlated}, where $\nu$ is chosen to minimize the RMSE error.

We also compare against factorizations obtained by numerical optimization. In particular, we use the scalable Toeplitz-banded \textit{BandMF} factorization of \citet{mckenna2024scaling}, which approximates
$\inf\limits_{\mathbf{B}\mathbf{C} = \mathbf{A}} \mathcal{E}_n(\mathbf{B}, \mathbf{C})$
and can be computed efficiently for large matrices.  We use the official implementation from the JAX\_privacy library~\citep{mckenna2026jax} for the experiments. Figure~\ref{fig:single_rmse} reports the RMSE errors as ratios relative to the numerically optimized BandMF error. The results show that, for sufficiently large matrix sizes, the proposed factorization
$\mathbf{C} = \mathbf{D}_{\mathrm{Toep}}$
achieves the lowest error among the explicit factorizations considered.

\subsection{Multi-participation}

We now turn to a more realistic scenario in which each user may participate multiple times. 
Under the $b$-min-separation setting, we extend our analysis beyond the four factorizations considered for single participation. In particular, we investigate the Banded and Banded-Inverse variants of these factorizations, where the correlation matrix is made $p$-banded or its inverse is made $p$-banded, as introduced in \citet{kalinin2024banded} and \citet{kalinin2025back}, respectively.

 The following Theorem~\ref{thm:RMSE_multi} provides error bounds for the factorizations with $\mathbf{C}=\mathbf{I}$, $\mathbf{C}=\mathbf{E}_1^{1/2}$, and $\mathbf{C}=\mathbf{D}_{\mathrm{Toep}}$ in both the non-banded and banded-inverse settings. The bounds are summarized in Table~\ref{tab:multi_part_results}. We see that $\mathbf{D}_{\mathrm{Toep}}$ provably achieves a better asymptotic error, and in the regime of large $k$ the bound can be further improved by adopting the banded-inverse factorization. In the proof, we also observe that the optimal bandwidth for the proposed factorization is equal to the separation parameter $b$, whereas for the prefix-sum–based factorization $\mathbf{E}_1$ it grows logarithmically in $b$. In practice, however, we find that $\mathbf{D}_{\mathrm{Toep}}$ can be used with much smaller bandwidth.

\begin{restatable}{theorem}{RMSEmulti}\label{thm:RMSE_multi}
    The RMSE at step $t$ under the $b$-min-separation condition for $b = \lceil \frac{n}{k}\rceil$ is bounded as shown in Table~\ref{tab:multi_part_results}. 
    For the banded inverse $\mathbf{E}_1^{1/2}$, the bandwidth is $p=\lceil\log_2 b\rceil$, 
    while for the other factorizations it is $p=b$.
\end{restatable}

The following theorem establishes a general lower bound on the $\mathrm{RMSE}$.

\begin{restatable}[RMSE Lower Bound]{theorem}{OptimalityMultiPart}\label{thm:OptimalityMultiPart}
    For any factorization of the matrix $\mathbf{A} = \mathbf{BC}$, the RMSE at step $n$ under the $b$-min-separation condition in the multi-participation setting is lower bounded by
    \begin{align}
        \mathcal{E}_n(\mathbf{B},\mathbf{C}) = \Omega\left(\frac{k}{n} + \frac{1}{\sqrt{n}}\right).
    \end{align}
Under the restriction that the matrix $\mathbf{C}$ is lower triangular and either (a) Toeplitz with positive coefficients, (b) Toeplitz and $b$-banded, or (c) column-normalized and $b$-banded, the following stronger lower bound holds:
    \begin{align}
        \mathcal{E}_n(\mathbf{B},\mathbf{C}) = \Omega\left(\sqrt{\frac{k}{n}}\right).
    \end{align}
\end{restatable}

In the single-participation regime ($k=1$) the theorem shows that the proposed factorization with $\mathbf{C}=\mathbf{D}_{\mathrm{Toep}}$ is asymptotically optimal among all factorizations. For the multi-participation setting, obtaining a tight bound is more challenging. We therefore focus on classes of factorizations that arise in practice, including lower-triangular positive Toeplitz factorizations, which contain our proposed factorization as well as the banded and banded-inverse versions of $\mathbf{E}_1^{1/2}$; Toeplitz-banded factorizations, corresponding to the BandMF factorization; and column-normalized banded factorizations~\citep{choquette-choo2023amplified}, which, however, do not scale to the matrix sizes considered here. Extending the tight lower bound to general factorizations appears substantially harder, as it would require a general characterization of the multi-participation sensitivity, which was shown to be NP-hard to compute~\citep{choquette-choo2023amplified}.

\paragraph{Remark on unbounded streams.}
Our proposed factorization $\mathbf{C} = \mathbf{D}_{\mathrm{Toep}}$ has bounded single-participation sensitivity, even as the matrix size $n$ tends to infinity. The recent work of \citet{jacobsen2025private} shows that infinite matrices with bounded sensitivity can be applied to unbounded streams when the noise is calibrated to the limiting sensitivity. However, no analogous theory has yet been developed for the multi-participation setting, which we view as a promising direction for future work.

\begin{figure}[t!]
    \centering
\begin{subfigure}[t]{0.3\textwidth}
    \centering
    \includegraphics[width=\linewidth]{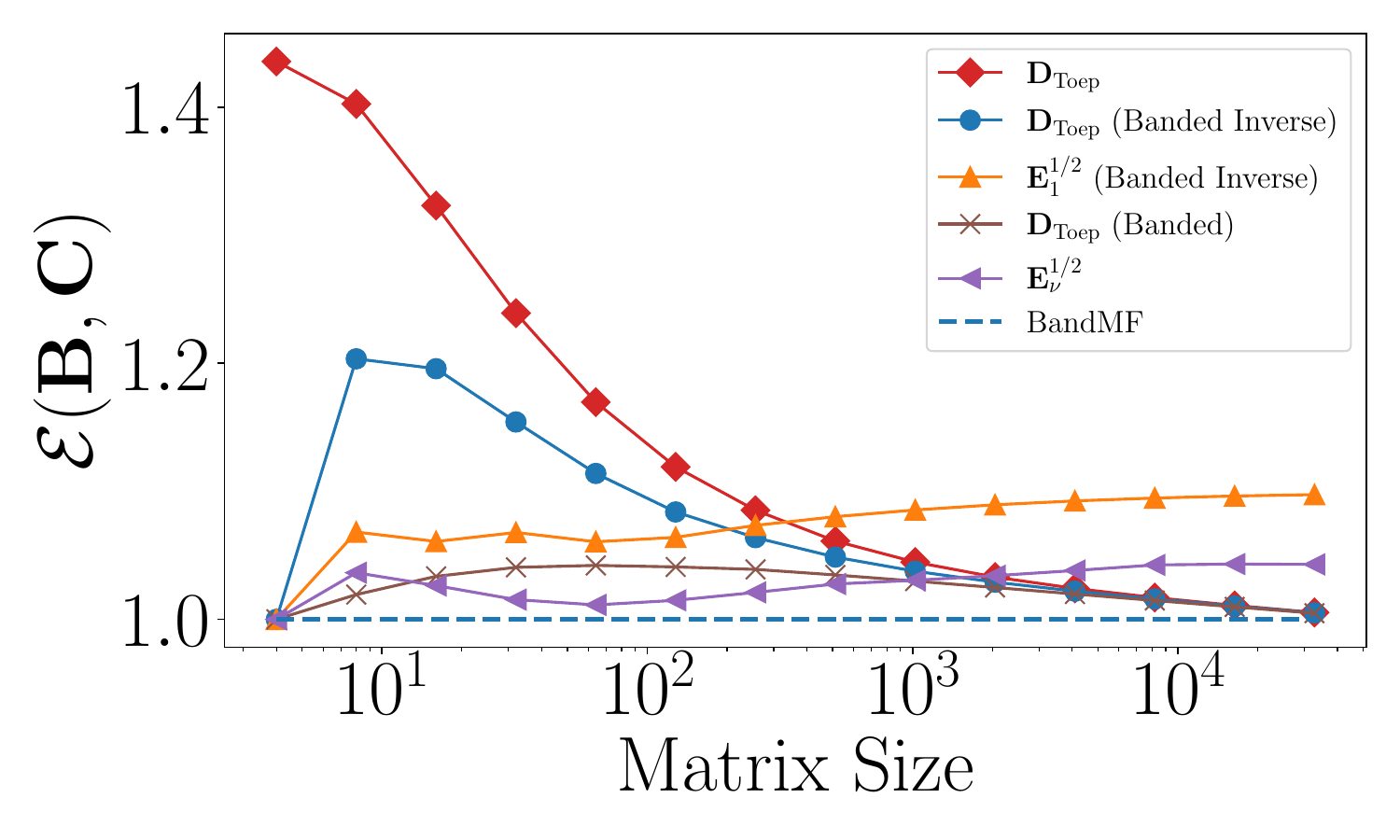}
    \caption{$k=4$}
    \label{fig:multi_k4}
  \end{subfigure}
\begin{subfigure}[t]{0.3\textwidth}
    \centering
    \includegraphics[width=\linewidth]{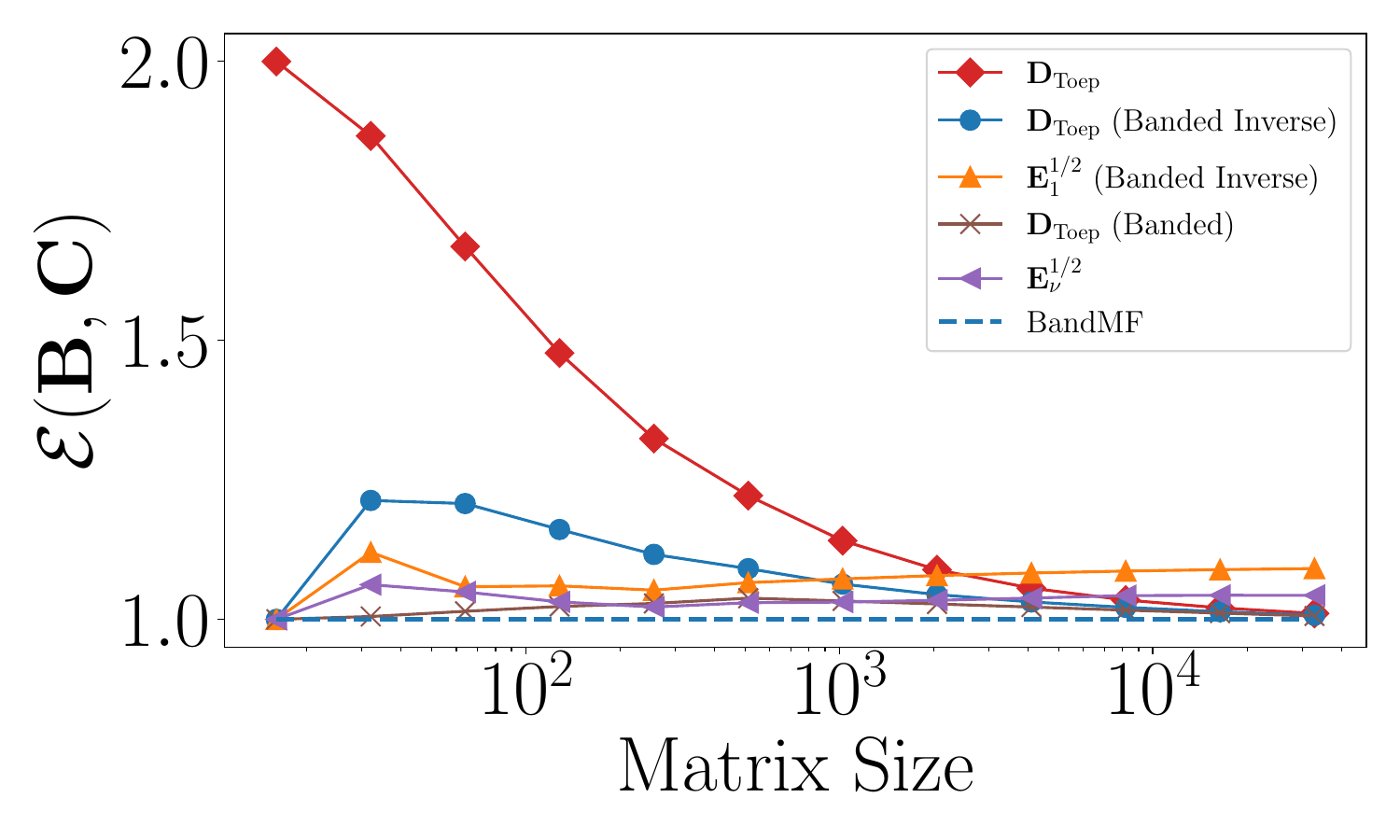}
    \caption{$k=16$}
    \label{fig:multi_k16}
\end{subfigure}
\begin{subfigure}[t]{0.3\textwidth}
    \centering
    \includegraphics[width=\linewidth]{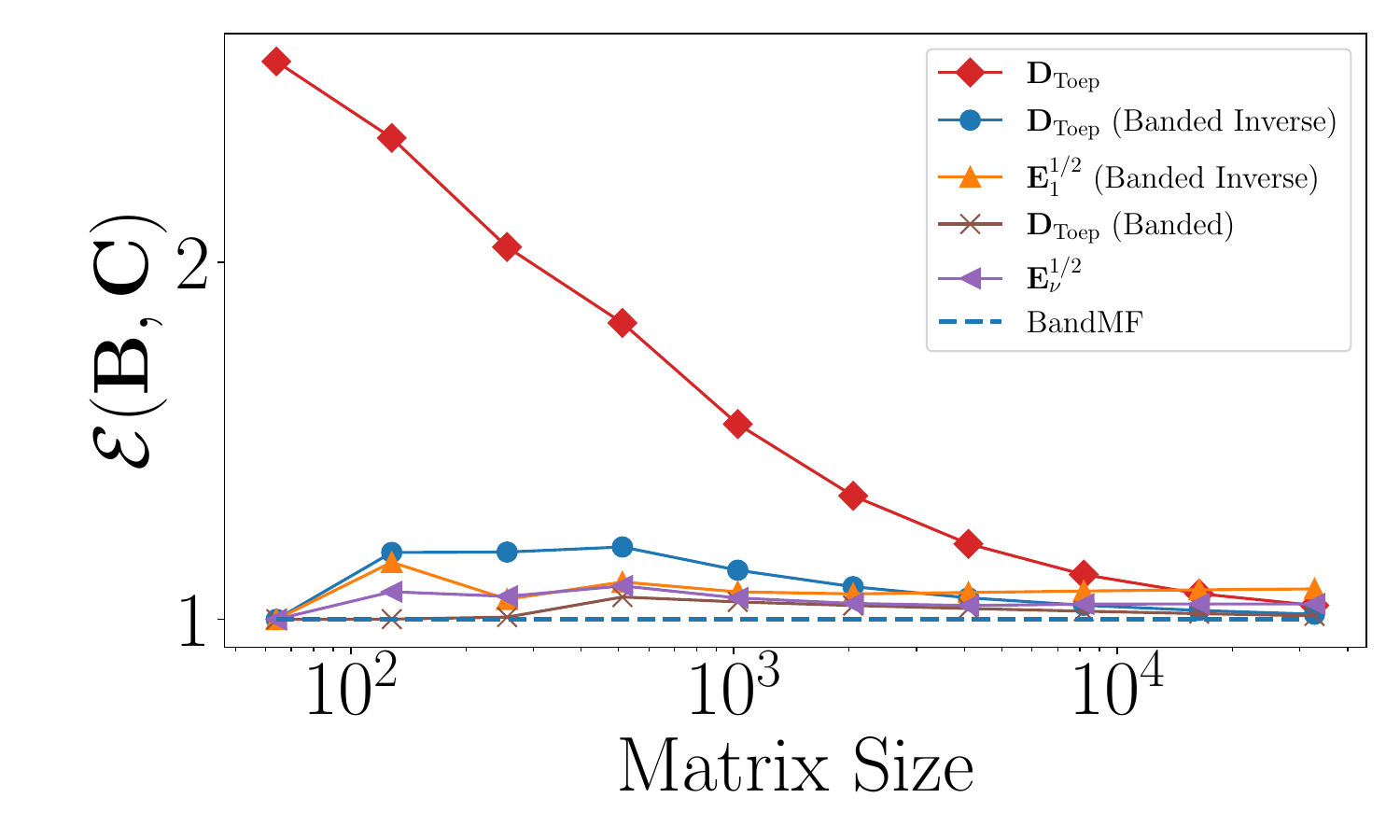}
    \caption{$k=64$}
    \label{fig:multi_k64}
\end{subfigure}
    \caption{RMSE at step $n$ of our proposed factorization and the prefix sum based factorization divided by the RMSE of the numerically optimized BandMF factorization, in the multi-participation setting. The bandwidth for $\mathbf{E}_1^{1/2}$ is set to $p=\lceil\log_2 b\rceil$ and for $\mathbf{D}_{\mathrm{Toep}}$ it is set to $p= b$. The banded version of $\mathbf{D}_{\mathrm{Toep}}$ shows a slight benefit over the proposed banded inverse, with the difference diminishing as the matrix size grows.
}
    \label{fig:rmse_multi}
\end{figure}

\section{Experiments}

Figure~\ref{fig:rmse_multi} plots the error ratios of different factorizations relative to the numerically optimized BandMF. We compare the non-banded, banded, and banded-inverse versions of our proposed factorization with the best-performing banded inverse prefix-sum–based factorization $\mathbf{E}_{1}^{1/2}$ and $\nu$-DP-FTRL ($\mathbf{E}_{\nu}^{1/2}$). The results show that, for sufficiently large matrix size $n$, all variants of $\mathbf{D}_{\mathrm{Toep}}$ achieve the lowest error among the considered factorizations and asymptotically dominate the prefix-sum–based factorization $\mathbf{E}_{1}^{1/2}$ and $\nu$-DP-FTRL $\mathbf{E}_{\nu}^{1/2}$, despite the latter being numerically optimized over the choice of $\nu$.

Table~\ref{tab:factorizations_multi} in the appendix reports numerical results for the expected error $\mathcal{E}_n(\mathbf{B},\mathbf{C})$ across different factorizations in the non-banded, banded, and banded-inverse settings. We evaluate the performance under varying participation numbers $k \in \{4, 16, 64\}$. Our proposed factorization with $\mathbf{C}= \mathbf{D}_{\mathrm{Toep}}$ performs substantially better than the prefix-sum–based factorization $\mathbf{C}=\mathbf{E}_1^{1/2}$. Moreover, applying banded and banded-inverse modifications further improves performance in the multi-participation setting. While the banded variant offers a slight empirical advantage over the banded-inverse variant for our mean-aware factorization, it provides weaker guarantees for the prefix-sum factorization. For $\nu$-DP-FTRL, we did not observe any additional benefit from optimizing over the choice of bandwidth. We argue that optimization over $\nu$ can effectively mimic the same process by automatically downweighting lower subdiagonals.

We further numerically validate our proposed mechanism by comparing it with \citep{george2024continual}; see Figure~\ref{fig:bin_mech_vs_mat_fact_mech}. Their approach is based on pure DP and the Binary Tree mechanism.
To achieve improved bounds for continual mean estimation, they introduce two techniques: \emph{exponential withholding} and \emph{concentration}. The former aggregates multiple samples from users before releasing them, while the latter exploits concentration of Bernoulli random variables to reduce the sensitivity of the Laplace mechanism. While their contribution is truly original, the procedure is inherently tailored to pure DP, Binary Tree mechanisms, and the Bernoulli setting, and it involves large hidden constants in the asymptotics, which limits its practical applicability. Empirically, we find that our method is more practical and significantly outperforms \citep{george2024continual} in numerical experiments, even on Bernoulli data. In Section~\ref{sec:distr_assumptions}, we generalize their bounds to the more general setting of approximate DP, allowing multidimensional sub-Gaussian distributions with unbounded support and arbitrary matrix mechanisms, yielding improved asymptotic rates.

In Figure~\ref{fig:bin_mech_vs_mat_fact_mech}, we plot both the prefix-sum factorization $\mathbf{E}_1^{1/2}$ and the mean-specific factorization $\mathbf{D}_{\text{Toep}}$. As theory predicts, prefix-sum factorization achieves a better asymptotic convergence rate. At the same time, $\mathbf{D}_{\text{Toep}}$ has lower error in the initial steps, and this early advantage dominates, leading to a smaller overall RMSE.
Choosing the optimal method is nontrivial, since we care about the entire error trajectory rather than just the final error. Moreover, no error curve uniformly dominates the others: at any point $t$, one could, in principle, spend the entire privacy budget to compute almost exact average at that step while returning pure noise to all other steps. In practice, however, the choice depends on the application, and one of these factorizations may prove more appealing than the other.

\paragraph{\textbf{Experiment with unknown value $b$.}}
We further evaluate our method on more realistic data; see Figure~\ref{fig:real_dataset}. For this, we use the Credit Card Transactions dataset \citep{kagglecreditcard2024}, which contains spending records from about 1,000 users over six months. This dataset enables us to study continual mean estimation under user-level differential privacy. We set the separation parameter to $b=500$. Ideally, we would use $b=1000$, but insufficient contributions from some users would halt the estimation process. Reducing $b$ increases update frequency but also amplifies noise, as more frequent updates require stronger privacy protection in the worst case\footnote{In general, the choice of the $b$-minimum-separation parameter depends on the experimental setting and on prior knowledge of the number of active users. The associated risks can, however, be mitigated by choosing a smaller value of $b$.}.
We use a generous clipping value of $\xi=1000$, assuming that transactions exceeding this amount are rare. With privacy parameters $\epsilon=10$ and $\delta=5\times 10^{-6}$, we obtain qualitatively accurate estimates of the running mean under both proposed factorizations.

\begin{figure}[t!]
    \centering

    \begin{minipage}[t]{0.48\textwidth}
        \centering
        \includegraphics[width=\linewidth]{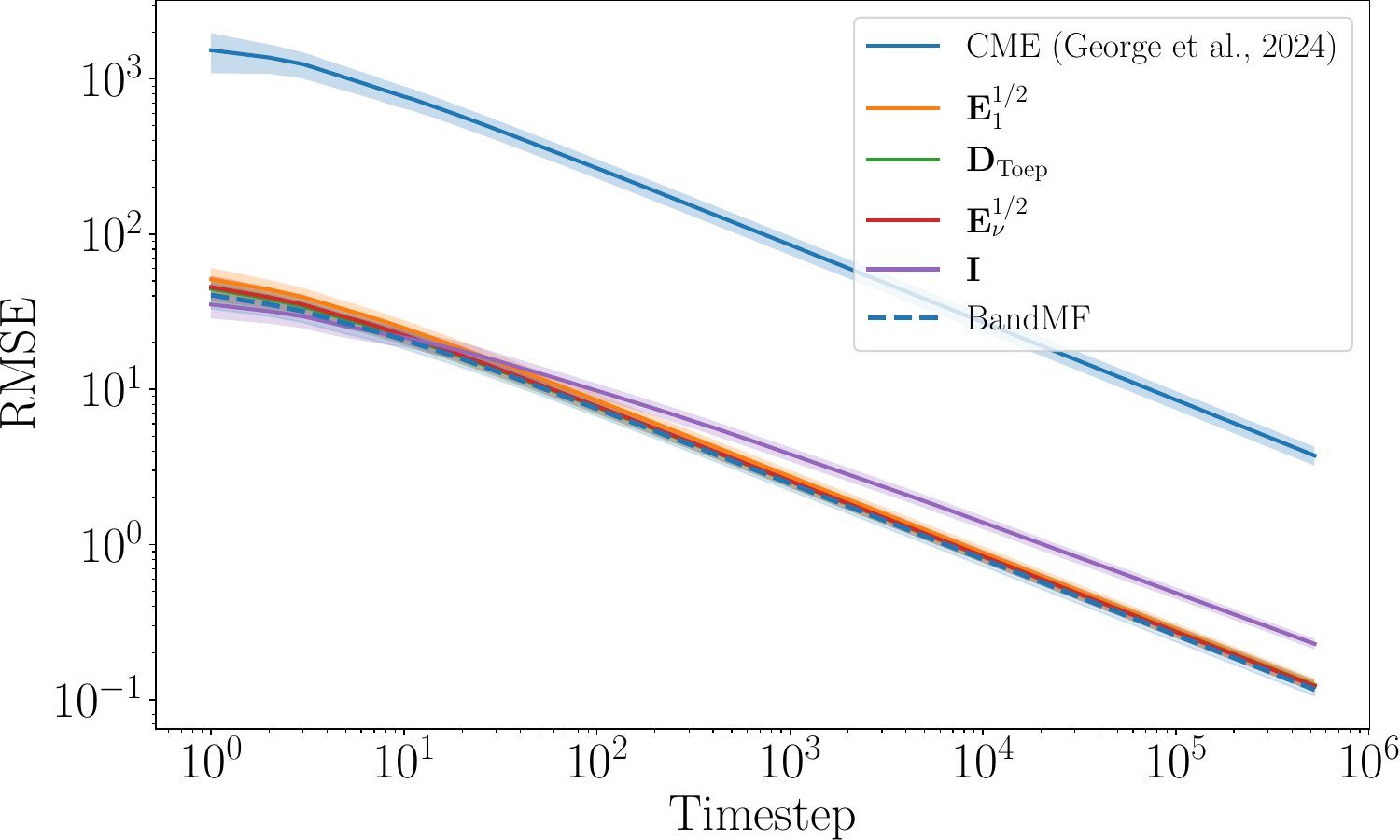}
        \caption{The error of mean estimation over time for $\varepsilon=1$ and $\delta=10^{-6}$ with $n=2^{19}$ and $k = 128$ participations measured in root averaged mean squared error. We used the banded inverse versions of $\mathbf{E}_1^{1/2}$ and $\mathbf{D}_{\mathrm{Toep}}$ with $p=16$ and $\mathbf{E}^{1/2}_{\nu}$ with $\nu=0.5$. 
        }
        \label{fig:bin_mech_vs_mat_fact_mech}
    \end{minipage}
    \hfill
    \begin{minipage}[t]{0.48\textwidth}
        \centering
        \includegraphics[width=\linewidth]{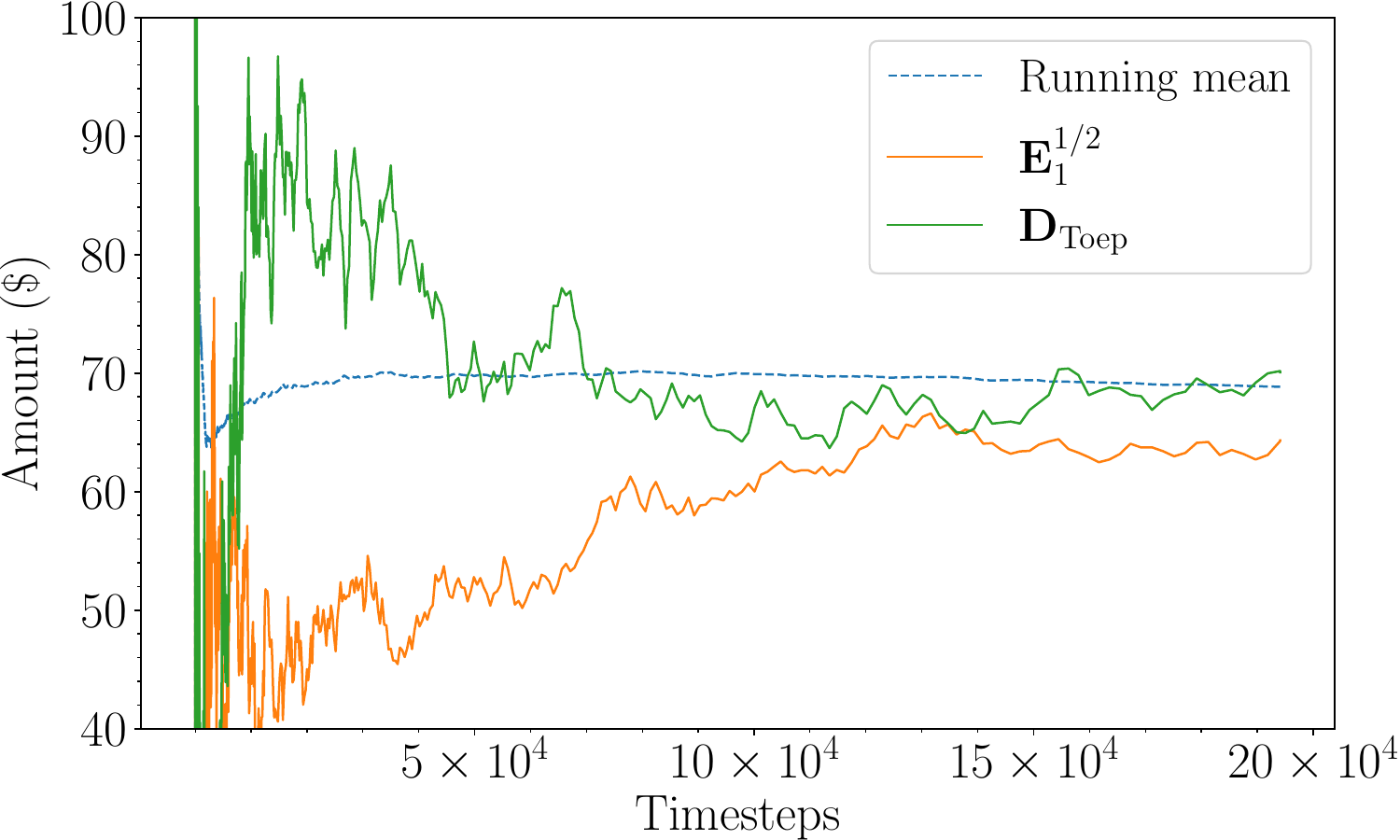}
        \caption{The error of mean estimation for the credit card transactions dataset for $\delta=5\times 10^{-6}$ and $\varepsilon=10$. We set $b=500$ and clip the stream of data by $\xi = 1000$ for the private mean.}
        \label{fig:real_dataset}
    \end{minipage}

\end{figure}
\section{Distributional Assumptions}
\label{sec:distr_assumptions}

So far, our theoretical analysis does not impose a statistical model on the data stream $\mathbf{X}$. We now introduce the following distributional assumption.

\begin{assumption}
    \label{ass:iidbounded}
    Let $\mathbf{X} := (\mathbf{x}_1,...,\mathbf{x}_n)^{\top}\in \mathbb{R}^{n \times d}$ have \iid{} rows with $\mathbb{E}[\mathbf{x}_i] = \boldsymbol{\mu} \in \mathbb{R}^d$ and $\|\mathbf{x}_i\|_\infty \leq \zeta$.
    
\end{assumption}

This assumption lets us compare not only the privacy overhead but also the statistical performance of the estimators introduced above. Beyond the finite-sample RMSE bounds stated below, we further discuss concentration-based refinements in Appendix \ref{sec:concentration}.

\begin{restatable}{lemma}{AlgOneRootSquaredError}
    \label{lm:averagemse}    
    Let $\mathbf{X}$ fulfill Assumption \ref{ass:iidbounded} and $\mathbf{A} = \mathbf{B} \mathbf{C}$ be any factorization. With probability at least $1-2\beta$, 
    \begin{align*}
        \frac{1}{\tau} \sum_{t = 1}^\tau \|\widehat{\boldsymbol{\mu}}_t - \boldsymbol{\mu}\|_2^2
        &\leq \tilde O\left(\frac{d\zeta^2}{\tau} + d^2\zeta^2 \cdot \sigma^2_{\varepsilon,\delta} \cdot \mathcal E_\tau^2(\mathbf{B}, \mathbf{C})\right).  
    \end{align*}
\end{restatable}

Applying this lemma with $\mathbf{B} = \mathbf{A}\mathbf{D}^{-1}_{\mathrm{Toep}}$ and $\mathbf{C} = \mathbf{D}_{\mathrm{Toep}}$ yields the following corollary as an immediate consequence. 

\begin{corollary}
    \label{cor:averagemse1}
    Under the same assumptions as in Lemma \ref{lm:averagemse} and $\mathbf{C}=\mathbf{D}_{\mathrm{Toep}}$, with probability at least $1-2\beta$, 
    \begin{align*}
        \frac{1}{\tau} \sum_{t = 1}^\tau \|\widehat{\boldsymbol{\mu}}_t - \boldsymbol{\mu}\|_2^2
        &\leq \tilde O\left(\frac{d\zeta^2}{\tau} + \frac{d^2\zeta^2k}{\tau\varepsilon^2} \right).  
    \end{align*}
\end{corollary}

The distributional assumption enables us also to exploit concentration of $\mathbf{x}_1,\ldots,\mathbf{x}_n$ to reduce the privacy error for bounded observations, and it also extends to unbounded (e.g., sub-Gaussian) observations. To this end, we combine the Matrix Factorization mechanism with the withhold--release scheme of \citet{george2024continual}. The scheme runs $L+1$ private mechanisms with $L = \lfloor \log_2(k) \rfloor$ and withholds observations in an exponential schedule so that, for each $l \in \{0,...,L\}$, it aggregates $2^{l-1}$ observations before privatizing. These averages concentrate in $\ell_2$-balls of radius $\tilde O\!\big(\sqrt{d\zeta^2/2^{l-1}}\big)$ around $\boldsymbol{\mu}$, reducing the noise needed after projecting to that required for such a ball. Since $\boldsymbol{\mu}$ is unknown, \citet{george2024continual} use a two-stage procedure \citep{smith2011privacy, levy2021uldp} that first computes a crude mean estimate with $\ell_2$-error $\tilde O\!\big(\sqrt{d\zeta^2/2^{l-1}}\big)$ and then uses it to center the projection balls. Their crude estimator relies on an $(\varepsilon,0)$-DP Exponential mechanism; because we allow approximate DP, we replace it with an improved $(\varepsilon,\delta)$-DP estimator based on private histograms \citep{karwa2018hist, kamath2020meanest, agarwal2025private, avellamedina2025meanest}, which in turn supports unbounded observations with sub-Gaussian entries.

We present the resulting estimator relying on $L+1$ Matrix Factorization mechanisms with prefix-sum matrices $\mathbf{E}_1 = \mathbf{B} \mathbf{C} \in \mathbb{R}^{b \times b}$ in the appendix (Algorithm \ref{alg:continualestimator}). There, we also extensively discuss Assumption \ref{ass:diversity} -- a "diversity condition" as in \citep{george2024continual} that the estimator's theoretical analysis puts on the stream of observations $\mathbf{X}$. In the item-level setting, the crude estimators above require a minimum number $K_c$ of observations to achieve $\ell_2$-error $O(\sqrt{d\zeta^2})$ with constant probability. As in \citet{george2024continual}, we apply them to averages of $2^{l-1}$ observations. The diversity condition ensures that one can form at least $K_c$ such averages without using observations from the same user in multiple ones, so that changing one user's data affects at most one average, thereby preserving user-level privacy.

To compare Algorithm \ref{alg:continualestimator} and Algorithm \ref{alg:continual-mean-estimation}, we specialize to an arrival pattern that satisfies both the $b$-min-separation and diversity conditions. Concretely, we assume a round-robin pattern in which $b$ users each contribute $k$ observations, i.e., user $u\in[b]$ provides $\mathbf{x}_{b+u},\mathbf{x}_{2b+u},\ldots,\mathbf{x}_{(k-1)b+u}$. 

\begin{restatable}{lemma}{AlgTwoRootSquaredErrorSimplified}
    \label{lm:averagemse2}
    Let $\mathbf{X}$ fulfill Assumption \ref{ass:iidbounded}, $n = kb$ and users contribute in a round-robin pattern. Let $K_c$ be as in Assumption \ref{ass:sufficientdiversity} and 
    $b \geq 2K_c$. Let $\mathbf{E}_1 \in \mathbb{R}^{b \times b}$ be a prefix-sum matrix and $\mathbf{B} = \mathbf{C} = \mathbf{E}_1^{1/2}$. With probability at least $1-5\beta$,
    \begin{align*}
        \frac{1}{\tau} \sum_{t = 1}^\tau \|\tilde{\boldsymbol{\mu}}_t - \boldsymbol{\mu}\|_2^2
        &\leq \tilde O\left(\frac{d\zeta^2}{\tau} + \frac{d^2\zeta^2}{\tau\varepsilon^2}\right). 
    \end{align*}
    
\end{restatable}

Note that Appendix \ref{sec:concentration} provides qualitatively similar bounds for unbounded observations and permits arbitrary factorizations $\mathbf{E}_1=\mathbf{B}\mathbf{C}$. Up to constants and logarithmic factors, the privacy cost in Lemma \ref{lm:averagemse2} is strictly smaller than in Corollary \ref{cor:averagemse1} whenever $k>1$: its privacy term is smaller by a factor $k$. Consequently, in the large-$k$ multi-participation regime, the exponential withhold--release scheme can outperform Algorithm \ref{alg:continual-mean-estimation}. Yet, for small $k$ constants and logarithmic terms become consequential, as reflected in Figure \ref{fig:bin_mech_vs_mat_fact_mech} and in the refined bounds in the appendix (e.g., Theorem \ref{thm:rmse}). Thus, while exponential withhold--release is asymptotically preferable for large $k$, Algorithm \ref{alg:continual-mean-estimation} remains an attractive default due to its simplicity and better performance once constants and logarithmic terms are taken into account.

\section{Conclusion and Future Directions}
We studied the problem of continual mean estimation under user-level differential privacy. Our approach builds on the Matrix Factorization mechanism, for which we propose a new factorization tailored to mean estimation. This design is both efficient and accurate compared to prior work. In particular, our factorization achieves lower RMSE than alternative explicit factorizations, both asymptotically and in numerical experiments. By combining matrix factorization with exponential withholding and sub-Gaussian concentration, we obtain improved asymptotic bounds, generalizing the continual mean estimation results of \citet{george2024continual} from pure DP to approximate DP, from the Binary Tree mechanism to arbitrary Matrix Factorization mechanisms, and from scalar Bernoulli observations to multidimensional sub-Gaussian samples with unbounded support.

One interesting direction for future research is data heterogeneity. Although our algorithm does not impose restrictions on the samples, our concentration bounds rely on the assumption that they are i.i.d.\ sub-Gaussian. Extending the analysis to more general heterogeneous distributions, following recent work on heterogeneous mean estimation~\citep{cummings2022mean}, would be a valuable next step.

\begin{ack}
We thank Jalaj Upadhyay and Joel Andersson for their valuable feedback on the early version of the paper.

Nikita Kalinin: This work is
supported in part by the Austrian Science Fund (FWF) [10.55776/COE12]. A part of this work was done while visiting  University of Copenhagen. 
\end{ack}

\bibliographystyle{plainnat}
\bibliography{bibliography}


\appendix
\onecolumn
\appendix

\section{Additional Experiments and the Full Algorithm}

\begin{algorithm}[h]
  \caption{User-Level DP Mean Estimation with Mean Aware Matrix Factorization}
  \label{alg:continual-mean-estimation}
  \begin{algorithmic}[1]
    \Require  Stream of data $\mathbf{X}=(\mathbf{x}_1,\dots,\mathbf{x}_n)^\top \in \mathbb{R}^{n\times d}$,
    privacy parameters $(\varepsilon,\delta)$, clipping norm $\xi$, bandwidth $p$, separation parameter $b$.

    \State $g_0, c^{p}_0 \gets 1, 1$
      \State $g_i \gets \displaystyle\sum_{j = 0}^{i - 1} g_{j}\,\frac{-1}{i + 1 - j} \qquad \forall i \in [1, p-1]$  \Comment{ Compute $i$-th Gregory coefficient.}

      \State $c_i^{p} \gets \displaystyle\!\!\!\! \sum_{j = 1}^{\min(p - 1, i)} \!\!\!\!\!\!
      (-g_{j})c_{i - j}^{p} \qquad \forall i \in [1, n - 1]$ \Comment{Compute $i$-th diagonal coefficient of $\mathbf{C}_{\text{Toep}}^{p}$.}
    \State $S \gets \displaystyle \sqrt{\sum_{i = 0}^{n - 1}\left(\sum_{j = 0}^{\left\lceil \frac{i + 1}{b} \right\rceil - 1} 
    \!\!\!\!\!\!
    c^{p}_{i - jb}\right)^2}$ \Comment{Compute $b$-min-separation sensitivity.}

    \State $\sigma \gets \sigma_{\varepsilon,\delta}\cdot \xi \cdot S$

    \For{$t = 1$ to $n$}
      \State Observe $\mathbf{x}_t \in \mathbb{R}^d$ with $\|\mathbf{x}_t\|_2 \le \xi$ and sample $\mathbf{z}_t \sim \mathcal{N}(\mathbf{0},\,\sigma^2\mathbf{I}_{d})$
      \State $\displaystyle\mathbf{u}_t \gets \mathbf{x}_t + \!\!\!\!\!\!\sum_{j=0}^{\min(p - 1, t - 1)}\!\!\!\!\!\! g_{j}\,\mathbf{z}_{t - j}$ \Comment{Compute the $t$-th row of $\mathbf{X}+\mathbf{C}^{-1}\mathbf{Z}$.}
      
      \State $\widehat{\boldsymbol{\mu}}_t \gets \frac{1}{t}\displaystyle\sum_{i=1}^{t} \mathbf{u}_i$
      \State \textbf{Release:} $\widehat{\boldsymbol{\mu}}_t$
    \EndFor
  \end{algorithmic}
\end{algorithm}

\begin{table}[h!]
  \centering
  \caption{$\mathcal{E}_n(\mathbf{B},\mathbf{C})$ for various factorizations at \(n=8192\). For the Banded and Banded Inverse factorizations, we set $p = b$ for 
$\mathbf{D}_{\mathrm{Toep}}$ and $\mathbf{E}_{\nu}^{1/2}$, and 
$p = \lceil \log_2 b \rceil$ for $\mathbf{E}_1^{1/2}$. We find that our choice of bandwidth is close to numerically optimal. Using the numerically optimized BandMF with optimal bandwidth $p=b$ as a lower bound, the proposed factorization achieves a near-optimal factorization error without incurring the computational burden.
}
  \label{tab:factorizations_multi}
  \begin{tabular}{c l c c c}
    \toprule
    \textbf{} & $\mathbf{C}$ & \({k=4}\) & \({k=16}\) & \({k=64}\) \\
    \midrule
    \multirow{4}{*}{Not Banded}
      & $\mathbf{I}$
        & 0.068 & 0.137 & 0.274 \\
      & $\mathbf{E}_1^{1/2}$
        & 0.072 & 0.221 & 0.813 \\
      & $\mathbf{D}_{\mathrm{Toep}}$
        &\textbf{0.042} & 0.086 & 0.186 \\
      & $\mathbf{E}_{\nu}^{1/2}$
        & 0.043 & 0.086 & 0.172 \\
    \midrule
    \multirow{3}{*}{Banded}
      & $\mathbf{E}_{1}^{1/2}$ 
        & 0.047 & 0.094 & 0.196 \\
      &  $\mathbf{D}_{\mathrm{Toep}}$ 
        & \textbf{0.042} & \textbf{0.084} & \textbf{0.169} \\
      & $\mathbf{E}_{\nu}^{1/2}$ 
        & 0.043 & 0.086 & 0.172 \\
      & BandMF & \textcolor{red}{0.041} & \textcolor{red}{0.081} & \textcolor{red}{0.164}\\
     
    \midrule
    \multirow{3}{*}{Banded Inverse}
      &  $\mathbf{E}_{1}^{1/2}$ 
        & 0.045 & 0.089 & 0.179 \\
      &  $\mathbf{D}_{\mathrm{Toep}}$ 
        & \textbf{0.042} & 0.085 & 0.172 \\
      & $\mathbf{E}_{\nu}^{1/2}$ 
        & 0.043 & 0.086 & 0.172 \\
    \bottomrule
  \end{tabular}
\end{table}
\clearpage

\section{Concentration Bounds for the Mean}

\label{sec:concentration}

\begingroup
\everymath{\displaystyle}
\begin{table*}[t]
  \centering
  \caption{Concentration bounds of $\|\widehat{\boldsymbol{\mu}}_t - \boldsymbol{\mu}\|_2$ for selected factorizations and $\mathrm{Bernoulli}(\boldsymbol{\mu})$ data. All bounds hold with probability $1-2\beta$.}
  \label{tab:concentration_multi}
\makebox[\textwidth][c]{%
  \begin{tabular}{c l c}
    \toprule
    \textbf{} & $\mathbf{C}$ & \(\|\widehat{\boldsymbol{\mu}}_t - \boldsymbol{\mu}\|_2\) \\
    \midrule
    \multirow{2}{*}{\rotatebox[origin=c]{90}{Not Banded}}
      & $\mathbf{D}_{\mathrm{Toep}}$
        & $O\left(\sqrt{\frac{d\zeta^2}{t}\ln\frac{d}{\beta}} +\frac{\sqrt{d\ln \frac{1}{\delta}}}{\varepsilon \sqrt{t}\ln t}\left(\sqrt{k} +  \frac{k}{\sqrt{n}}\sqrt{\ln k \ln n}\right) \left(\sqrt{d} + \sqrt{\ln\tfrac{1}{\beta}}\right)\right)$ \\
      & $\mathbf{E}_1^{1/2}$
        & $O\left(\sqrt{\frac{d\zeta^2}{t}\ln\frac{d}{\beta}} +\frac{\sqrt{d\ln t \ln \tfrac{1}{\delta}}\left(\sqrt{k\ln n}+k\right) }{\varepsilon t}\left(\sqrt{d} + \sqrt{\ln\tfrac{1}{\beta}}\right)\right)$ \\
      
    \midrule
    \multirow{2}{*}[\dimexpr2.5ex]{\rotatebox[origin=c]{90}{Banded Inverse}}
      &  $\mathbf{D}_{\mathrm{Toep}}$ 
        & $O\left(\sqrt{\frac{d\zeta^2}{t}\ln\frac{d}{\beta}} +\frac{\sqrt{dk\ln \frac{1}{\delta}}}{\varepsilon \sqrt{t \ln (n/k)}} \left(1 + \sqrt{\frac{n}{kt}}\right)\left(\sqrt{d} + \sqrt{\ln\tfrac{1}{\beta}}\right)\right)$ \\
        & $\mathbf{E}_{1}^{1/2}$ 
        & $\!\!\!O\left(\sqrt{\frac{d\zeta^2}{t}\ln\frac{d}{\beta}} +\frac{\sqrt{dk}}{\varepsilon\sqrt{t}}\sqrt{\ln \frac{1}{\delta}}\left(\sqrt{\ln\ln(\tfrac{n}{k})} + \sqrt{\frac{k\ln (\tfrac{n}{k})}{n}} \right) \left(\frac{\sqrt{k}}{\sqrt{\ln (\tfrac{n}{k})}}+\frac{\sqrt{\ln\ln (\tfrac{n}{k})}}{\sqrt{t}}\right)\left(\sqrt{d} + \sqrt{\ln\tfrac{1}{\beta}}\right)\right)$ \\
    \bottomrule
  \end{tabular}
  }
\end{table*}
\endgroup

The RMSE analysis in the main body of the paper quantified the average performance of our estimator. However, RMSE alone does not capture how the estimation error fluctuates around its expected value. To address this, we now establish high-probability concentration bounds for the private running mean, providing a finer characterization of its accuracy over time. Specifically, we derive high-probability concentration bounds on the error of the private running mean $\widehat{\boldsymbol{\mu}}_t$, relative to the true mean $\boldsymbol{\mu} = \mathbb{E}[\mathbf{x}_k]$ and the non-private empirical mean $\boldsymbol{\mu}_t = \tfrac{1}{t}\sum_{k=1}^t \mathbf{x}_k$ where $\mathbf{x}_1,\dots,\mathbf{x}_t$ are the private data. We first state Theorem~\ref{thm:StatErrorBound} which gives a bound on the value of $\|\widehat{\boldsymbol{\mu}}_t - \boldsymbol{\mu}_t\|_2$ for a general factorization $\mathbf{A} = \mathbf{BC}$.

\begin{restatable}{theorem}{StatErrorBound}\label{thm:StatErrorBound}

    Let $\mathbf{X} := (\mathbf{x}_1,...,\mathbf{x}_n)^{\top}\in \mathbb{R}^{n \times d}$ have $\|\mathbf{x}_i\|_\infty \leq \zeta$ and let $\mathbf{A}=\mathbf{BC}$ be any factorization  with $\mathbf{A}, \mathbf{B}, \mathbf{C} \in \mathbb{R}^{n\times n}$. Further, let $\mathbf{b}_t$ be the $t$-th row of matrix $\mathbf{B}$ and $\mathrm{sens}_{k,b}(\mathbf{C})$ be a b-min-separation sensitivity in  multi participation. Then, with probability at least $1-\beta$
    \begin{align}
        \|\widehat{\boldsymbol{\mu}}_t - \boldsymbol{\mu}_t\|_2 = O\left(\sqrt{d\zeta^2} \ \sigma_{\varepsilon,\delta} \ \mathrm{sens}_{k,b}(\mathbf{C}) \|\mathbf{b}_t\|_2\left(\sqrt{d}+ \sqrt{\ln \tfrac{1}{\beta}}\right)\right). 
    \end{align}

\end{restatable}

By plugging in the value of $\sigma_{\varepsilon,\delta}\,\mathrm{sens}(\mathbf{C})\|\mathbf{b}_t\|$ for a given matrix factorization, one can directly obtain the bounds for $\mathbf{C}=\mathbf{E}_1^{1/2}$ and $\mathbf{C}=\mathbf{D}_{\mathrm{Toep}}$. The corresponding results are stated in Corollaries~\ref{cor:DAsqrt} and~\ref{cor:ADtpInv}, respectively.

To derive a bound for the underlying mean, we again rely on the statistical model used in Section \ref{sec:distr_assumptions}, which is formalized in Assumption \ref{ass:iidbounded}. An important instantiation of this model are \iid{} observations from the multi-dimensional Bernoulli distribution (See \citet{cummings2022mean, george2024continual}) where $\zeta = 1$. Theorem \ref{thm:StatErrorBoundTrueMean} below provides a finite-sample upper bound on the root squared error for general factorizations. It is obtained by combining Theorem \ref{thm:StatErrorBound} with a bound on the statistical error $\|\boldsymbol{\mu}_t - \boldsymbol{\mu}\|_2$ through a triangle inequality. 

\begin{restatable}{theorem}{StatErrorBoundTrueMean}\label{thm:StatErrorBoundTrueMean}

    Let $\mathbf{X} \in \mathbb{R}^{n \times d}$ fulfill Assumption \ref{ass:iidbounded}. With the same assumptions of Theorem~\ref{thm:StatErrorBound}, with probability at least $1-2\beta$,
    \begin{align}
        \|\widehat{\boldsymbol{\mu}}_t - \boldsymbol{\mu}\|_2 &= O\left(\sqrt{\frac{d\zeta^2}{t}\ln\frac{d}{\beta}} + \sqrt{d\zeta^2} \ \sigma_{\varepsilon,\delta} \ \mathrm{sens}_{k,b}(\mathbf{C}) \|\mathbf{b}_t\|_2\left(\sqrt{d}+ \sqrt{\ln \tfrac{1}{\beta}}\right) \right).
    \end{align}
\end{restatable}

Based on this bound, we now summarize the final error bounds for the prefix-sum–based factorization $\mathbf{C}=\mathbf{E}_1^{1/2}$ and the mean-estimation–specific factorization $\mathbf{C}=\mathbf{D}_{\mathrm{Toep}}$, both in the non-banded and banded-inverse settings, for the multi-dimensional Bernoulli estimation problem. We plug in the values of $\sigma_{\varepsilon, \delta}$ from Lemma~\ref{lem:GaussianMechanism}. The bounds are presented in Table~\ref{tab:concentration_multi} and are formally stated in the following theorem, with the proof deferred to the proof section of the appendix.

\begin{restatable}{theorem}{Concentration_multi}\label{thm:Concentration_multi}
    Under the $b$-min-separation condition, the concentration bound for the error of mean estimation holds as shown in Table~\ref{tab:concentration_multi}. 
    For the banded inverse $\mathbf{E}_1^{1/2}$, the bandwidth is $p=\lceil\ln b\rceil$, 
    while for the other factorizations it is $p=b$.
\end{restatable}

Table~\ref{tab:concentration_multi} highlights a fundamental trade-off between the prefix-sum–based factorization and the mean-estimation–specific factorization. For the prefix-sum approach, the tail of the error decreases at the rate $\widetilde{O}(k/t)$, whereas for the mean-specific factorization the rate is slower, at $\widetilde{O}(\sqrt{k/t})$. This difference illustrates a tension between two desirable properties: lowering the overall expected error (as measured by $\mathrm{RMSE}$) and accelerating the rate at which the error decreases with $t$. The mean-specific factorization improves performance at the initial steps, but this comes at the price of a slower dependence on $t$.  

Consequently, neither method uniformly dominates the other. In fact, such uniform dominance is impossible: one could always spend the entire privacy budget on a single timestep $t$, achieving optimal accuracy there at the cost of injecting pure noise everywhere else.
Hence, the choice of factorization naturally depends on the application and which part of the error profile is most critical. We therefore consider the prefix-sum factorization to be a valuable option for implementing the Matrix Factorization mechanism when accuracy at later time steps is more important than accuracy at earlier ones.

Moreover, the prefix-sum factorization is useful when combined with additional tricks like the withhold-release scheme due to \citet{george2024continual}. Since Algorithm \ref{alg:continualestimator} implements a projection onto a private interval, our theory allows for unbounded observations that fulfill the condition below that replaces Assumption \ref{ass:iidbounded}. 

\begin{assumption}
    \label{ass:iid}
    Let $\mathbf{X} :=(\mathbf{x}_1,...,\mathbf{x}_n)^{\top}\in \mathbb{R}^{n \times d}$ have \iid{} rows with $\zeta^2$-sub-Gaussian entries and $\mathbb{E}[\mathbf{x}_i] = \boldsymbol{\mu} \in \mathbb{R}^d$.
    
\end{assumption}

Similar to Theorem \ref{thm:StatErrorBoundTrueMean} for Algorithm \ref{alg:continual-mean-estimation}, the result below states a utility guarantee for Algorithm \ref{alg:continualestimator} in terms of the pointwise root squared error $\|\widehat{\boldsymbol{\mu}}_t - \boldsymbol{\mu}\|_2$. Since the withhold-release estimator maintains $L+1$ privacy mechanisms that each only privatize $b$ sums, the Matrix Factorization mechanisms involved directly use the prefix-sum matrix $\mathbf{E}_1$. Here, we state the bound for general factorizations $\mathbf{E}_1 = \mathbf{B} \mathbf{C}$. 

\begin{restatable}{lemma}{AlgTwoSquaredError}
    \label{lm:squarederror}
    Let $\mathbf{X} \in \mathbb{R}^{n \times d}$ fulfill Assumption \ref{ass:iid} and let $\mathbf{E}_1 \in \mathbb{R}^{b \times b}$ be a prefix-sum matrix and $\mathbf{E}_1 = \mathbf{B} \mathbf{C}$ with $\mathbf{C}, \mathbf{B} := (\mathbf{b}_1,...,\mathbf{b}_n)^\top \in \mathbb{R}^{n \times n}$. For all $t \in [n]$ such that Assumption \ref{ass:sufficientdiversity} holds, with probability $1-5d\gamma$, 
    \begin{align}
        \|\tilde{\boldsymbol{\mu}}_t - \boldsymbol{\mu}\|_2
        &\leq \tilde O\left(\sqrt{\frac{d\zeta^2}{t}} + \frac{d\zeta}{t} \cdot \sqrt{k_t} \cdot \sigma_{\varepsilon^\prime, \delta^\prime} \operatorname{sens}(\mathbf{C}) \cdot \sum_{\ell \in \mathbb A} \|\mathbf{b}_{|\mathbb{M}_{t, \ell}|}\|_2\right). 
    \end{align}
\end{restatable}

This upper bound on the root squared error immediately implies a bound on the RMSE via a union bound over the timepoints $t \in [\tau]$. We provide this result in Theorem \ref{thm:rmse}. Stating the theorem for unbounded observations with $\zeta^2$-sub-Gaussian entries, requires a "warm-up" period during which no sensible utility guarantee can be given, since the crude mean cannot yet be obtained with sufficient accuracy. We display this general result to complement Lemma \ref{lm:averagemse2} in the main body that uses prefix-sum factorizations with $\mathbf{B} = \mathbf{C} = \mathbf{E}_1$ and avoids the warm-up through bounded observations. For such, the first two projection intervals are given trivially. 

\begin{restatable}{theorem}{AlgTwoRootSquaredError}
    \label{thm:rmse}
    Let $\mathbf{X} \in \mathbb{R}^{n \times d}$ fulfill Assumption \ref{ass:iid} and let $\mathbf{E}_1 \in \mathbb{R}^{b \times b}$ be a prefix-sum matrix with $\mathbf{E}_1 = \mathbf{B} \mathbf{C}$ for $\mathbf{C}, \mathbf{B} \in \mathbb{R}^{b \times b}$. Suppose $b \geq 2K_c$ and let $2K_c \leq \tau_l \leq \tau_u \leq n$. When Assumption \ref{ass:sufficientdiversity} holds for all $t \in \{\tau_l,...,\tau_u\}$, with probability $1-5d\gamma$,
    \begin{align}
         &\frac{1}{\tau_u -  \tau_l + 1}  \sum_{t = \tau_l}^{\tau_u} \|\tilde{\boldsymbol{\mu}}_t - \boldsymbol{\mu}\|_2^2 \\
       & \qquad \leq \tilde O\left(\frac{d\zeta^2}{\tau_u - \tau_l + 1} +  \frac{d^2\zeta^2}{\tau_u - \tau_l + 1} \cdot \sigma^2_{\varepsilon^\prime, \delta^\prime} \operatorname{sens}^2(\mathbf{C}) \cdot \sum_{t = \tau_l}^{\tau_u} \frac{k_t}{t^2} \cdot \left(\sum_{\ell \in \mathbb A} \|\mathbf{b}_{|\mathbb{M}_{t, \ell}|}\|_2\right)^2\right). 
    \end{align}
    
\end{restatable}

The theorem shows that we could explicitly optimize over factorizations to potentially improve the rate obtained using the prefix-sum factorization. The issue of the lack of uniform dominance discussed above thereby applies again. 

\section{Key Relationships}
Throughout the proofs, we use the propositions below.

\begin{restatable}{proposition}{harmonicNumber}\label{pro:harmonicNumber}
Let $H_n$ be the $n$-th harmonic number then
\begin{align}
    H_n = \ln n + \gamma + \frac{1}{2n} - \varepsilon_n, \label{eq:harmonic_number}
\end{align}
where $\gamma\approx0.5772$ is the  Euler–Mascheroni constant, and $0\le\varepsilon_n \le \frac{1}{8n^2}$.
\end{restatable}

\begin{restatable}{proposition}{harmonicNumberRiemannZeta}\label{pro:harmonicNumberRiemannZeta}
Let $\zeta$ be the Riemann zeta function, then from \citet{borwein1995intrigintegral} it holds
\begin{align}
        2 \sum_{t=1}^{\infty} \frac{H_t}{t^m}
        = (m+2)\,\zeta(m+1)
        - \sum_{t=1}^{m-2} \zeta(m-t)\,\zeta(t+1), \label{eq:harmonic_number_riemann_zeta}
    \end{align}
    which for $m=2$ gives
    \begin{align}
        \sum_{t=1}^{\infty} \frac{H_t}{t^2} = 2\zeta(3).
    \end{align}
Additionally, $\zeta(3)$ is called Apéry's constant and is approximately $1.202$.
\end{restatable}

\begin{restatable}{proposition}{dilogarithm}\label{pro:dilogarithm}
The dilogarithm \citep{lewin1981polylogarithms} denoted by $\mathrm{Li}_2(z)$ is defined as
\begin{align}
    \operatorname{Li}_{2}(z) = - \int_{0}^{z} \frac{\ln(1-t)}{t}\, dt 
= - \int_{0}^{1} \frac{\ln(1-zt)}{t}\, dt,
\end{align}
where in case of real $z\ge 1$ it is written as
\begin{align}
    \operatorname{Li}_{2}(z) = \frac{\pi^{2}}{6} - \int_{1}^{z} \frac{\ln(t-1)}{t}\, dt - i\pi \ln z.
\end{align}
For $z=1$ it directly follows that $\operatorname{Li}_{2}(1)=\frac{\pi^2}{6}$.
\end{restatable}

\section{Proofs}

\DToepInverse*
\begin{proof}

    To invert the Toeplitz matrix $\mathbf{D}_{\mathrm{Toep}}$, we use its generating function
    \begin{align}
        D_{\mathrm{Toep}}(x) = \sum_{j=0}^{\infty} \frac{x^j}{j+1} = \frac{\ln (1-x)}{-x}, \qquad \ |x| < 1.
    \end{align}
    Hence its formal inverse series is
    \begin{align}
        D^{-1}_{\mathrm{Toep}}(x) = \frac{-x}{\ln(1-x)} = 1 - \sum_{j=1}^{\infty} |G_j|x^j, \label{eq:g_k_sign}
    \end{align}
    where $G_j$ is the $j$-th Gregory coefficient \citep{Gregory1841_CollinsLetter}. The coefficients of \(D_{\mathrm{Toep}}^{-1}(x)\) give the entries of \(\mathbf{D}_{\mathrm{Toep}}^{-1}\), which is again Lower‑Triangular Toeplitz (LTT):
    \begin{align}
        \mathbf{D}_{\mathrm{Toep}}^{-1}= \begin{pmatrix}
    1        & 0        & \cdots & 0 \\
    g_{1}    & 1        & \cdots & 0 \\
    \vdots   & \vdots   & \ddots & \vdots \\
    g_{n-1}  & g_{n-2}  & \cdots & 1
    \end{pmatrix},
    \end{align}
    where $g_j=-|G_j|$ for $j\ge 1$.
\end{proof}

\subsection{Single Participation}

\begin{restatable}{lemma}{rmseDAoneSqrt}\label{lem:rmse_DA1sqrt_A1sqrt}
For the prefix sum based factorization $\mathbf{B}=\mathbf{DE}_1^{1/2}$ and $\mathbf{C}=\mathbf{E}_1^{1/2}$ we have, for all $n\ge2$,
\begin{align}
\sqrt{\frac{\pi}{6}\cdot\frac{\ln n}{n}} (1+o(1))
< \mathcal{E}_n\!\left(\mathbf{D}\mathbf{E}_1^{1/2},\mathbf{E}_1^{1/2}\right)
&<
\sqrt{\frac{5}{2\pi}\cdot\frac{\ln n}{n}}.
\end{align}
so for $n$ sufficiently large, prefix-sum based factorization  behaves better asymptotically than trivial  $\mathbf{B} = \mathbf{A}$ and $\mathbf{C}=\mathbf{I}$.
\end{restatable}

\begin{proof}
    Using the results from \citet{fichtenberger2023constant}, we know that the square root of the prefix‑sum matrix can be computed as
    \begin{align}
    \label{eq:A1sqrt}
\mathbf{E}_{1}^{1/2} =
\begin{pmatrix}
1        & 0        & \cdots & 0 \\
r_{1}    & 1        & \cdots & 0 \\
\vdots   & \vdots   & \ddots & \vdots \\
r_{n-1}  & r_{n-2}  & \cdots & 1
\end{pmatrix},
    \end{align}
where $r_{j} = \left|\binom{-1/2}{j}\right|$. The bounds for the coefficients follow from Wallis’s inequality~\citep{chen2005best} for $j \ge 1$:
\begin{equation}
\label{eq:bound_on_rk}
    \frac{1}{\sqrt{\pi (j + 1)}}
    \le \frac{1}{\sqrt{\pi \left(j + \tfrac{4}{\pi} - 1\right)}}
    \le r_j
    \le \frac{1}{\sqrt{\pi \left(j + \tfrac{1}{4}\right)}}
    \le \frac{1}{\sqrt{\pi j}}.
\end{equation}

    Furthermore, in the recent work by \citet{henzinger2025normalized}, 
    tight bounds were derived for the sum of squares of $r_j$:
    \begin{equation}
        \frac{\ln n}{\pi} + \frac{\gamma + \ln(16)}{\pi} - \frac{1}{5n}
        \le \sum_{j = 0}^{n - 1} r_j^2
        \le \frac{\ln n}{\pi} + \frac{\gamma + \ln(16)}{\pi},
    \end{equation}
    where $\gamma$ denotes the Euler--Mascheroni constant. Using this result, we obtain
    \begin{align}
        \|\mathbf{E}_1^{1/2}\|_{1\to 2} = \sqrt{\sum_{j=0}^{n-1} r_j^2}  \sim \sqrt{\frac{\ln n}{\pi}}.
    \end{align}

    Also for $\|\mathbf{D}\mathbf{E}_1^{1/2}\|_{F}$ we have:
    \begin{align}
        \|\mathbf{D}\mathbf{E}_1^{1/2}\|_{F} = \sqrt{\sum_{t=1}^n \frac{1}{t^2} \sum_{j=0}^{t-1} r_j^2} 
        &\le \sqrt{\sum_{t=1}^n \frac{1}{t^2} \left( 1 + \frac{1}{\pi} + \sum_{j=2}^{t-1} \frac{1}{\pi j} \right)} \label{eq:r_k_bound} \\ 
        &\le\sqrt{\sum_{t=1}^n \frac{1}{t^2}\left( 1 + \frac{1}{\pi} + \frac{\ln t}{\pi}\right)} \label{eq:r_k_bound2} \\
        &\le \sqrt{\frac{\pi^2}{6}\left( 1 + \frac{1}{\pi}\right) + \frac{1}{\pi}\sum_{t=1}^n \frac{\ln t}{t^2}} \\
        &\le \sqrt{\frac{\pi^2}{6}+\frac{\pi}{6}+\frac{1}{\pi}}
        < \sqrt{\frac{5}{2}}, \label{eq:DA1sqrt_upper}
    \end{align}
where we used that $\sum_{t=1}^\infty \frac{\ln t}{t^2} < 1$. Combining \eqref{eq:A1sqrt} and \eqref{eq:DA1sqrt_upper}, we obtain
\begin{align}
\mathcal{E}_n(\mathbf{D}\mathbf{E}_1^{1/2},\mathbf{E}_1^{1/2}) < \sqrt{\frac{5}{2\pi}\cdot\frac{\ln n}{n}}.
\end{align}
This proves that the factorization $\mathbf{B}=\mathbf{DE}_1^{1/2} , \mathbf{C}=\mathbf{E}_1^{1/2}$ gives a smaller error than the trivial factorization $\mathbf{B}=\mathbf{A}, \mathbf{C}=\mathbf{I}$.

We can similarly bound $\|\mathbf{D}\mathbf{E}_1^{1/2}\|_{F}$ from below,
\begin{align}
    \|\mathbf{D}\mathbf{E}_1^{1/2}\|_{F} = \sqrt{\sum_{t=1}^n \frac{1}{t^2} \sum_{j=0}^{t-1} r_j^2} 
        &\ge \sqrt{\sum_{t=1}^n \frac{1}{t^2} \left( 1 + \sum_{j=1}^{t-1} \frac{1}{\pi (j+1)} \right)} \\ 
        &\ge \sqrt{\sum_{t=1}^n \frac{1}{t^2}\left( 1 - \frac{1}{\pi} + \frac{\ln (t+1)}{\pi}\right)} \\
        &= \sqrt{\frac{\pi^2}{6}-\frac{\pi}{6}+C+o(1)}
        > \frac{\pi}{\sqrt{6}} + o(1), \label{eq:DA1sqrt_lower}
\end{align}
where $C = \sum_{t=1}^\infty \frac{\ln (t+1)}{t^2} \approx 1.80$ is a constant larger that $\frac{\pi}{6}$. This yields the bound:
\begin{align}
    \mathcal{E}_n(\mathbf{D}\mathbf{E}_1^{1/2},\mathbf{E}_1^{1/2}) > \sqrt{\frac{\pi}{6}\cdot\frac{\ln n}{n}} (1+o(1)).  
\end{align}
\end{proof}

\begin{restatable}{lemma}{rmseAI}\label{lem:rmse_A_I}
The root‑mean‑square error ({RMSE}) of trivial factorization $\mathbf{B} = \mathbf{A}$ and $\mathbf{C} = \mathbf{I}$ satisfies
\begin{align}
\mathcal{E}_n(\mathbf{A},\mathbf{I})
= \sqrt{\frac{H_n}{n}} \sim \sqrt{\frac{\ln n}{n}},
\end{align}
where $H_n=\sum_{k=1}^n\frac1k$ is the $n$-th harmonic number.
\end{restatable}

\begin{proof}
This theorem is just a special case of Lemma~\ref{lem:AIMulti} for $k=1$. See the proof of Lemma~\ref{lem:AIMulti} for a more general result.
\end{proof}

\begin{restatable}{lemma}{InvLogSqrBound}\label{lem:InvLogSqrBound}
The sum $\sum_{j=2}^t \frac{1}{\ln^2 j}$
is upper bounded by $\frac{3t}{\ln^2 t}$, for any value of $t\ge 2$.
\end{restatable}
\begin{proof}
    We prove this by induction. For $t=2,\dots,21$ the result can be checked to be true numerically. Now we prove the induction step for $t\ge 22$. Let $f(x) = \frac{x}{\ln^2 x}$. Then we prove that
    \begin{align}
        \frac{1}{\ln^2 t} \le 3\left(f(t) - f(t-1) \right), \qquad \text{for} \quad t\ge 22. \label{eq:ft-ft-1}
    \end{align}
    We compute the derivative $f'(x)$ as
    \begin{align}
        f'(x) = \frac{1}{\ln^2 x} - \frac{2}{\ln^3 x} = \frac{1}{\ln^2 x} \left( 1-\frac{2}{\ln x} \right).
    \end{align}
    By the mean value theorem, there exists a value $\eta\in[t-1,t]$ such that
    \begin{align}
        f(t) - f(t-1) = f'(\eta) = \frac{1}{\ln^2 \eta}\left(1-\frac{2}{\ln \eta}\right).
    \end{align}
    Since $1-\frac{2}{\ln \eta} \ge 1-\frac{2}{\ln (t-1)}$, we get
    \begin{align}
        f(t) - f(t-1) \ge \frac{1}{\ln^2 \eta}\left(1-\frac{2}{\ln (t-1)}\right) \ge \frac{1}{\ln^2 t} \left(1-\frac{2}{\ln (t-1)}\right).
    \end{align}
    So to show \eqref{eq:ft-ft-1}, it suffices to show that $1 - \frac{2}{\ln (t-1)} \ge \frac{1}{3}$, which follows directly since $\ln(t-1) \ge \ln (21) \ge 3$.
    Using this result we can prove the induction step. Assume that the bound $\sum_{j=2}^{t-1} \frac{1}{\ln^2 j} \le 3f(t-1)$ holds for some $t\ge 22$. Then, using \eqref{eq:ft-ft-1} we get
    \begin{align}
        \sum_{j=2}^t \frac{1}{\ln^2 j} = \sum_{j=2}^{t-1} \frac{1}{\ln^2 j} + \frac{1}{\ln^2 t} &\le 3f(t-1) + \frac{1}{\ln^2 t} \\
        &\le 3f(t-1) + 3\left(f(t) - f(t-1) \right) = 3f(t),
    \end{align}
    which completes the proof.
\end{proof}

\begin{restatable}{lemma}{ADtpInv}\label{lem:ADtpInv}
    The RMSE of mean estimation specific factorization  $\mathbf{B}=\mathbf{AD}_{\mathrm{Toep}}^{-1}$ and $\mathbf{C}=\mathbf{D}_{\mathrm{Toep}}$ is bounded by 
    \begin{align}
        \mathcal{E}_n(\mathbf{A}\mathbf{D}_{\mathrm{Toep}}^{-1}, \mathbf{D}_{\mathrm{Toep}}) = \mathrm{\Theta} \left(\frac{1}{\sqrt{n}}\right).
    \end{align}
    In particular, this improves on the standard prefix‑sum based factorization (see Lemma~\ref{lem:rmse_DA1sqrt_A1sqrt}) choice by eliminating the $\sqrt{\ln n}$ factor in the denominator.
\end{restatable}

\begin{proof}
 We begin by bounding the sensitivity. We can see that,
\begin{align}
    \|\mathbf{D}_{\mathrm{Toep}}\|_{1\to 2} = \sqrt{\sum_{t=1}^{n}\frac{1}{t^2}} \sim \frac{\pi}{\sqrt{6}} + o(1), \label{eq:sensC}
\end{align}
and for the Frobenius norm part recall \(\mathbf{A}= \mathbf{D}\,\mathbf{E}_1\) with \(\mathbf{D}=\operatorname{diag}(1,\tfrac{1}{2},\dots,\tfrac{1}{n})\).  Write \(\mathbf{B}=\mathbf{D}\,\mathbf{E}_1\,\mathbf{D}_{\mathrm{Toep}}^{-1}\) and denote matrix $\mathbf{E}_1\,\mathbf{D}_{\mathrm{Toep}}^{-1}$ as $\mathrm{LTT}(a_0,\dots,a_{n-1})$ where
\begin{align}
    a_0=1,\quad \text{and}\quad a_t = 1 + \sum_{j=1}^{t} g_j, \quad t>0. \label{eq:a_k_def}
\end{align}
Then, we compute the Frobenius norm as
\begin{align}
    \|\mathbf{A}\mathbf{D}_{\mathrm{Toep}}^{-1}\|_F = \sqrt{\sum_{t=1}^n \frac{1}{t^2}\sum_{j=0}^{t-1} a_j^2}.
\end{align}
Using equations \eqref{eq:greg_first_prop} we know that
\begin{align}
    \sum_{t=1}^\infty g_t = -1,
\end{align}
and for $t\ge 5$
\begin{align}
    \frac{1}{t\ln^2 t} - \frac{2}{t\ln^3 t} \le |G_t| \le \frac{1}{t\ln^2 t} - \frac{2\gamma}{t\ln^3 t},
\end{align}
so using the fact that all $g_t$'s have the same sign (see equation~\eqref{eq:g_k_sign}), we get
\begin{align}
|a_{t}| = \left|1 + \sum_{j=1}^{t} g_j\right| = \sum_{j=t+1} ^{\infty}|g_j| \le \sum_{j=t+1}^{\infty} \frac{1}{j\ln^2 j} \le \frac{1}{\ln t},\qquad t\ge 4, \label{eq:a_k_bound}
\end{align}
where we used integral bound in the last inequality. Also for $t<4$ it can be easily shown that $|a_{t}|\le 1$.
Now using \eqref{eq:a_k_bound} we can bound the Frobenius norm:
\begin{align}
    \|\mathbf{A}\mathbf{D}_{\mathrm{Toep}}^{-1}\|_F^2&\le \overbrace{\sum_{t=1}^n \frac{1}{t^2}\sum_{j=0}^3 a_j^2}^{\le \ \mathrm{constant\ }C} + \sum_{t=5}^n \frac{1}{t^2} \sum_{j=4}^{t-1} a_j^2 \\
    &\le C + \sum_{t=5}^n \frac{1}{t^2}\sum_{j=4}^{t-1} \frac{1}{\ln^2 j} \le C + \sum_{t=5}^{n}\frac{3}{t\ln^2 t} + o(1) = O(1), \label{eq:FrobeniusNormADtpInv}
\end{align}
where in the last inequality we used Lemma \ref{lem:InvLogSqrBound} to bound the sum of $\frac{1}{\ln^2 j}$. So overall:
\begin{align}
    \mathcal{E}_n(\mathbf{A}\mathbf{D}_{\mathrm{Toep}}^{-1}, \mathbf{D}_{\mathrm{Toep}}) = O\left(\tfrac{1}{\sqrt{n}}\right)
\end{align}
On the other hand, since in both $\mathbf{A}\mathbf{D}_{\mathrm{Toep}}^{-1}$ and $\mathbf{D}_{\mathrm{Toep}}$ there exists an entry with value at least 1, both norm values are at least 1. This means that the error is lower bounded by $\Omega\left(\tfrac{1}{\sqrt{n}}\right)$,  Hence
\begin{align}
    \mathcal{E}_n(\mathbf{A}\mathbf{D}_{\mathrm{Toep}}^{-1}, \mathbf{D}_{\mathrm{Toep}}) = \mathrm{\Theta} \left(\tfrac{1}{\sqrt{n}}\right).
\end{align}
\end{proof}

\StatErrorBound*

\begin{proof}

At timestep $t$ when we release the mean, the estimation error will be the norm of $t$-th row of matrix $\mathbf{BZ}$. Denote the $t$-th row of $\mathbf{B}$ as $\mathbf{b}_t^{\top}$, so we get that $\|\widehat{\boldsymbol{\mu}}_t - \boldsymbol{\mu}_t\|_2 = \|\mathbf{b}_t^{\top} \mathbf{Z}\|_2$. Since the entries of $\mathbf{Z}$ are all i.i.d. and come from the distribution $\mathcal{N}(0,\sigma^2_{\varepsilon,\delta}\cdot \mathrm{sens^2(\mathbf{C})})$,
then if we let $\mathbf{e}_{t} = \mathbf{b_t^\top \mathbf{Z}}$, we have
\begin{align}
\mathbf{e}_{t} \sim \mathcal{N}(0,\sigma^2 \mathbf{I}_{d\times d})    
\end{align}
where $\sigma^2 =\sigma^2_{\varepsilon,\delta}\cdot \mathrm{sens^2(\mathbf{C})}\cdot \|\mathbf{b}_t\|_2^2$. Now, one can see that $\frac{\|\mathbf{e}_t\|_2^2}{\sigma^2}\sim \mathcal{X}^2_d$, so using the Laurent-Massart bound (see \citet{laurent2000adaptive}) for a chi squared random variable with $d$ degrees of freedom we have,
\begin{align}
    \mathbb{P}\left[ \frac{\|\mathbf{e}_t\|_2^2}{\sigma^2} \le d + \sqrt{2d\ln \frac{1}{\beta}} + 2\ln \frac{1}{\beta} \right] \ge 1 - \beta.
\end{align}
So we can write the following:
\begin{align}
    \mathbb{P}\left[ \|\mathbf{e}_t\|_2 \le \sigma\sqrt{d + \sqrt{2d\ln \frac{1}{\beta}} + 2\ln \frac{1}{\beta}} \right] \ge 1 - \beta,
\end{align}
where one can substitute $\sqrt{d + \sqrt{2d\ln \frac{1}{\beta}} + 2\ln \frac{1}{\beta}} = O(\sqrt{d}+\sqrt{\ln\frac{1}{\beta}})$ to get the final result.
\end{proof}

\StatErrorBoundTrueMean*

\begin{proof}

    As $\boldsymbol{\mu}_t$ is an average computed from $\mathbf{X}_{:t}$ which as a submatrix of $\mathbf{X}$ also fulfills Assumption \ref{ass:iidbounded}, by Corollary \ref{cor:concsubgaussmean} 
    \begin{align}
        \mathbb{P}\left[\|\boldsymbol{\mu}_t - \boldsymbol{\mu}\|_2 \leq \sqrt{\frac{2d\zeta^2}{t}\ln\frac{2d}{\beta}}\right] 
        \geq 1 - \beta. 
    \end{align}

    Combining this bound with the one in Theorem~\ref{thm:StatErrorBound} using union bound we get the final result.
\end{proof}

\begin{restatable}{corollary}{corollaryDAsqrt} \label{cor:DAsqrt}
For  prefix-sum based factorization with $\mathbf{B}=\mathbf{DE}_1^{1/2}$ and $\mathbf{C}=\mathbf{E}_1^{1/2}$, the error in Theorem~\ref{thm:StatErrorBound} has the form
\begin{align}
        \|\widehat{\boldsymbol{\mu}}_t - \boldsymbol{\mu}_t\|_2 = O\left(\frac{ \sqrt{d\ln \frac{1}{\delta}\ln n\ln t}}{\varepsilon t}\left(\sqrt{d}+ \sqrt{\ln \tfrac{1}{\beta}}\right)\right),
\end{align}
with probability at least $1-\beta$.
\end{restatable}

\begin{proof}

One can plug in the value of $\sigma_{\varepsilon,\delta} = \frac{1}{\varepsilon}\sqrt{2\ln \frac{1.25}{\delta}}$ into the bound. Then, using equation~\eqref{eq:A1sqrt} we know that $\|\mathbf{E}_1^{1/2}\|_{1\to 2} \sim \sqrt{\frac{\ln n}{\pi}}$. Also using the same approach as in \eqref{eq:r_k_bound} and \eqref{eq:r_k_bound2} we can bound the $t$-th row norm of $\mathbf{DE}_1^{1/2}$ as $\frac{1}{t}\sqrt{1+\frac{1}{\pi}+\frac{\ln t}{\pi}} = O\left(\frac{\sqrt{\ln t}}{t}\right)$. Substituting these into the bound gives the result directly.
\end{proof}

\begin{restatable}{corollary}{corollaryDAsqrtTrueMean} \label{cor:DAsqrtTrueMean}
For prefix-sum based factorization  with $\mathbf{B}=\mathbf{DE}_1^{1/2}$ and $\mathbf{C}=\mathbf{E}_1^{1/2}$, the error in Theorem~\ref{thm:StatErrorBoundTrueMean} has the form
\begin{align}
        O\left(\sqrt{\frac{d\zeta^2}{t}\ln\frac{d}{\beta}} + \frac{\sqrt{d\ln \frac{1}{\delta}\ln n\ln t}}{\varepsilon t}\left(\sqrt{d}+ \sqrt{\ln \tfrac{1}{\beta}}\right)\right),
\end{align}
with probability at least $1-2\beta$.
\end{restatable}

\begin{restatable}{corollary}{corollaryADtpInv} \label{cor:ADtpInv}
For mean estimation specific factorization with $\mathbf{B}=\mathbf{AD}_{\mathrm{Toep}}^{-1}$ and $\mathbf{C}=\mathbf{D}_{\mathrm{Toep}}$, the error in Theorem~\ref{thm:StatErrorBound} has the form
\begin{align}
       \|\widehat{\boldsymbol{\mu}}_t - \boldsymbol{\mu}_t\|_2= O\left(\frac{\sqrt{d\ln \frac{1}{\delta}}}{\varepsilon\sqrt{t}\ln t}\left(\sqrt{d}+ \sqrt{\ln \tfrac{1}{\beta}}\right)\right),
\end{align}
with probability at least $1-\beta$.
\end{restatable}

\begin{proof}
    We use the same value of $\sigma_{\varepsilon,\delta}$ as in the proof of Corollary~\ref{cor:DAsqrt}.
     For sensitivity, using equation~\eqref{eq:sensC} we find that $\|\mathbf{C}\|_{1\to 2}= O(1)$ and for $\|\mathbf{b}_t\|_2$, using the definition of $a_j$ in \eqref{eq:a_k_def}, and the bound in \eqref{eq:a_k_bound}, we write
\begin{align}
    \|\mathbf{b}_t\|_2 = \sqrt{\frac{1}{t^2}\sum_{j=0}^{t-1}a_j^2} \le \frac{1}{t}\sqrt{\sum_{j=4}^{t-1}\frac{1}{\ln^2 j}+O(1)} \le \frac{1}{t}\sqrt{\frac{3t}{\ln^2 t} + O\left(1\right)} = O\left( \frac{1}{\sqrt{t}\ln t}\right).
\end{align}
where we used Lemma~\ref{lem:InvLogSqrBound} for the last inequality.
Thus, with probability at least $1-\beta$ we have:
\begin{align}
   \|\widehat{\boldsymbol{\mu}}_t - \boldsymbol{\mu}_t\|_2 = O\left(\frac{\xi\sqrt{\ln \frac{1}{\delta}}}{\varepsilon\sqrt{t}\ln t}\left(\sqrt{d}+ \sqrt{\ln \tfrac{1}{\beta}}\right)\right),
\end{align}
which proves our claim.
\end{proof}

\begin{restatable}{corollary}{corollaryADtpInvTrueMean} \label{cor:ADtpInvTrueMean}
For mean estimation specific factorization $\mathbf{B}=\mathbf{AD}_{\mathrm{Toep}}^{-1}$ and $\mathbf{C}=\mathbf{D}_{\mathrm{Toep}}$, the error in Theorem~\ref{thm:StatErrorBoundTrueMean} has the form
\begin{align}
        O\left(\sqrt{\frac{d\zeta^2}{t}\ln\frac{d}{\beta}} + \frac{\sqrt{d\ln \frac{1}{\delta}}}{\varepsilon\sqrt{t}\ln t}\left(\sqrt{d}+ \sqrt{\ln \tfrac{1}{\beta}}\right)\right),
\end{align}
with probability at least $1-2\beta$.
\end{restatable}

\subsection{Multi-Participation}

\OptimalityMultiPart*
\begin{proof}

In the proof of Theorem~1 in \citet{kalinin2025back}, the following general lower bound was established:
\begin{equation}
    \mathcal{E}_n(\mathbf{B}, \mathbf{C}) \;\ge\; \frac{1}{\sqrt{n}} \, \|\mathbf{A}\pi_1\|_2,
\end{equation}
where $\pi_1$ is a Boolean vector with ones at positions $1 + jb$ for $j \in [0, k - 1]$.  
It remains to compute the norm, which can be done explicitly:
\begin{equation}
    \|\mathbf{A}\pi_1\|_2^2 
    = \sum_{j=1}^{n} \left(\Big\lceil \tfrac{j}{b} \Big\rceil \tfrac{1}{j}\right)^2
    = \sum_{j=1}^{b} \frac{1}{j^2} 
    + \sum_{j=b+1}^{n} \left(\Big\lceil \tfrac{j}{b} \Big\rceil \tfrac{1}{j}\right)^2 
    \;\ge\; 1 + \sum_{j=b+1}^{n} \left(\tfrac{j}{b}\cdot \tfrac{1}{j}\right)^2
    = 1 + \frac{n-b}{b^2}.
\end{equation}
For $k = \lceil n/b \rceil$, this implies
\begin{equation}
    \|\mathbf{A}\pi_1\|_2^2 = \Omega\!\left(1 + \frac{k^2}{n}\right).
\end{equation}
Taking the square root and multiplying by $\tfrac{1}{\sqrt{n}}$ yields the desired result.

For the second part of the theorem, consider the case where the matrix $\mathbf{C}$ is lower triangular Toeplitz. Let the main diagonal element be $c_0$. By dividing the matrix $\mathbf{C}$ by $c_0$ and multiplying the matrix $\mathbf{B}$ by $c_0$, the RMSE error remains unchanged. Therefore, without loss of generality, we may assume that $\mathbf{C}$ has value $1$ on the main diagonal. Then
\begin{equation}
    \mathrm{sens}_{k, b} (\mathbf{C}) \ge \|\mathbf{C}\pi_1\|_2 \ge \|\mathbf{I}\pi_1\|_2 = \sqrt{k},
\end{equation}
where the first inequality holds by the definition of sensitivity, and the second follows from either the positivity of the coefficients of $\mathbf{C}$ or its $b$-bandedness. If the matrix $\mathbf{C}$ is column-normalized and $b$-banded, then each column contributes $1$ to the sum of scalar products, and all intercolumn products are $0$ because of bandedness. Thus, the sensitivity is exactly $\mathrm{sens}_{k, b}(\mathbf{C}) = \sqrt{k}$.

On the other hand, the first element in the lower triangular matrix $\mathbf{B}$ must be equal to $1$, and hence $\|\mathbf{B}\|_F \ge 1$. Normalizing by $\sqrt{n}$ gives the desired lower bound.
\end{proof}

\RMSEmulti*
\begin{proof}
    For proofs of the corresponding bounds, see Lemma~\ref{lem:AIMulti} for the trivial factorization; Lemmas~\ref{lem:DAoneSqrtMulti} and~\ref{lem:BISRAoneSqrt} for the prefix-sum–based factorization; and Lemmas~\ref{lem:ADtpInvMulti} and~\ref{lem:MultiDtpBISR} for the mean-aware factorization.
\end{proof}

\begin{restatable}{lemma}{AIMulti}\label{lem:AIMulti}
    For trivial factorization $\mathbf{B}=\mathbf{A}$ and $\mathbf{C}=\mathbf{I}$ the expected approximation error is 
    \begin{align}
        \mathcal{E}_t(\mathbf{A}, \mathbf{I}) \sim \sqrt{\frac{k\ln t}{t}}.
    \end{align}
\end{restatable}

\begin{proof}
    Let us first compute the Frobenius norm
    \begin{align}
        \frac{1}{\sqrt{t}}\|\mathbf{A}_{:t}\|_F = \sqrt{\frac{1}{t}\sum_{r=1}^t \sum_{j=1}^r \frac{1}{r^2}} 
        = \sqrt{\frac{1}{t}\sum_{r=1}^t \frac{1}{r}} 
        = \sqrt{\frac{H_t}{t}} \sim \sqrt{\frac{\ln t }{t}}, \label{eq:FrobeniusAI}
    \end{align}
    where the last term comes from Proposition~\ref{pro:harmonicNumber}.
    We now proceed to calculating the sensitivity $\mathrm{sens}_{k,b}(\mathbf{I})$. Using the formula for sensitivity in Theorem~\ref{thm:sensTheorem}, for a matrix $\mathbf{C} = \mathbf{I}$, we immediately  have 

    \begin{align}
        \mathrm{sens}^2_{k,b}(\mathbf{I}) = \sum_{i=0}^{k-1} \bigl\langle \mathbf{I}_{[:,\,i b]}\,,\;\mathbf{I}_{[:,\,i b]}\bigr\rangle = k.
    \end{align}
    Thus, the error in this case is
    \begin{align}
        \mathcal{E}_t(\mathbf{A},\mathbf{I}) = \sqrt{\frac{kH_t}{t}} \sim \sqrt{\frac{k\ln t}{t}}
    \end{align}
\end{proof}

\begin{restatable}{lemma}{DAoneSqrtMulti}\label{lem:DAoneSqrtMulti}
    For prefix sum based factorization  $\mathbf{B}=\mathbf{DE}_1^{1/2}$ and $\mathbf{C}=\mathbf{E}_1^{1/2}$ the expected approximation error is lower and upper bounded by
    \begin{align}
    \sqrt{\frac{k\ln n}{8t}} + O\left(\frac{k}{\sqrt{t}}\right) \le \mathcal{E}_t(\mathbf{DE}_1^{1/2},\mathbf{E}_1^{1/2}) &\le \sqrt{\frac{5k\ln n}{2t}} + \frac{\sqrt{15}k}{\sqrt{t}} + O\left(\sqrt{\frac{k}{t}}\right).
    \end{align}
    
\end{restatable}

\begin{proof}
    For the Frobenius norm according to equation~\eqref{eq:DA1sqrt_upper}, we get $1+o(1)\le\|(\mathbf{DE}_1^{1/2})_{:t}\|_F \le \sqrt{\frac{5}{2}}$. So we move on to calculating the sensitivity.
    From \citet{kalinin2024banded} we know that for an LLT matrix $\mathbf{A}_{\alpha,\beta}$ such that
    \begin{align}
        \mathbf{A}_{\alpha,\beta} = \begin{pmatrix}
            1 & 0 & \cdots & 0 \\
            \alpha & 1 & \cdots & 0 \\
            \vdots & \vdots & \ddots & \vdots \\
            \alpha^{n-1} & \alpha^{n-2} & \cdots & 1 \\
        \end{pmatrix} \times
        \begin{pmatrix}
            1 & 0 & \cdots & 0 \\
            \beta & 1 & \cdots & 0 \\
            \vdots & \vdots & \ddots & \vdots \\
            \beta^{n-1} & \beta^{n-2} & \cdots & 1 \\
        \end{pmatrix},
        \label{eq:Aalphabeta}
    \end{align}
    if we consider $\mathbf{C}_{\alpha,\beta}$ to be its square root, then for $\alpha = 1$,
    \begin{align}
    \operatorname{sens}_{k,b}^{2}\!\left(\mathbf{C}_{1,\beta}\right)
    &\le \frac{k}{(1-\beta)^{2}}(\ln n + 1) + \frac{6}{(1-\beta)^{2}}\,k^{2},\\
    \operatorname{sens}_{k,b}^{2}\!\left(\mathbf{C}_{1,\beta}\right)
    &\ge \frac{k}{4}(\ln n - 1) - \frac{8}{3\sqrt{b}}\,k^{3/2} + \frac{2}{5}\,k^{2}.
\end{align}
Now if we let $\beta=0$, then $\mathbf{A}_{1,0}=\mathbf{E}_1$ and therefore $\mathbf{C}_{1,0} = \mathbf{E}_1^{1/2}$, so we can write
\begin{align}
    \frac{k\ln n}{4} + \frac{2k^2}{5} - \frac{8k^{3/2}}{3\sqrt{b}} -\frac{k}{4}\le \mathrm{sens}_{k,b}^2(\mathbf{E}_1^{1/2}) \le k\ln n + 6k^2 + k.
\end{align}
So the error is bounded by
\begin{align}
    \sqrt{\frac{ k\ln n}{4t} + O\left(\frac{k^2}{t}\right)} \le \mathcal{E}_t(\mathbf{DE}_1^{1/2},\mathbf{E}_1^{1/2}) \le \sqrt{\frac{5k\ln n}{2t} + \frac{15k^2}{t} + \frac{5k}{2t}},
\end{align}
which, using $\frac{1}{\sqrt{2}}(\sqrt{x}+\sqrt{y})\le \sqrt{x+y} \le \sqrt{x}+\sqrt{y}$ for positive $x$ and $y$, is written as
\begin{align}
    \sqrt{\frac{k\ln n}{8t}} + O\left(\frac{k}{\sqrt{t}}\right) \le \mathcal{E}_t(\mathbf{DE}_1^{1/2},\mathbf{E}_1^{1/2}) \le \sqrt{\frac{5k\ln n}{2t}} + \frac{\sqrt{15}k}{\sqrt{t}} + O\left(\sqrt{\frac{k}{t}}\right).
\end{align}

\end{proof}

\begin{restatable}{lemma}{BISRAoneSqrt}\label{lem:BISRAoneSqrt}
    For banded inverse factorization $\mathbf{A} = \mathbf{B}^p\mathbf{C}^p$ where $\mathbf{C}=\mathbf{E}_1^{1/2}$, the expected approximation error is bounded by
    \begin{align}
        \mathcal{E}_t(\mathbf{B}^p, \mathbf{C}^p) 
        =O\sqrt{\frac{k}{t}\left(\ln p + \frac{pk}{n}\right)\left(1 + \frac{1}{p}\ln \frac{t}{p}\right)}.
    \end{align}
    For $p=\lceil \ln b \rceil$ this becomes
    \begin{align}
        \mathcal{E}_t(\mathbf{B}^p, \mathbf{C}^p) = O\left(\sqrt{\frac{k \ln \ln n}{t}} + \sqrt{\frac{k^2\ln n}{nt}}\right).
    \end{align}
\end{restatable}

\begin{proof}
    We first compute the Frobenius norm $\|\mathbf{B}^p\|_F$. The inverse of $\mathbf{E}_1^{1/2} = \mathrm{LTT}(1,r_1,\cdots,r_{n-1})$ has been computed in \citet{kalinin2025back} as $\mathbf{E}_1^{-1/2} = \mathrm{LTT}(1,\tilde{r}_1,\cdots,\tilde{r}_{n-1})$ where
    \begin{align}
        \tilde{r}_j = \begin{cases}
            1, & j=0, \\
            \frac{-r_j}{2j-1}, & j\ge 1.
        \end{cases}
    \end{align}
    Using the bound $\tfrac{1}{\sqrt{\pi(j+1)}}\le r_j \le \tfrac{1}{\sqrt{\pi j}}$ from equation~\eqref{eq:bound_on_rk}, we get the bound below on $|\tilde{r}_j|$ for $j\ge 1$,
    \begin{align}
        \frac{1}{(2j-1)\sqrt{\pi(j+1)}} \le |\tilde{r}_j| \le \frac{1}{(2j-1)\sqrt{\pi j}}.
    \end{align}
    Using this bound we can compute the Frobenius norm of $\mathbf{B}^p = \mathbf{A}\times \mathrm{LTT}(1,\tilde{r}_1,\dots,\tilde{r}_{p-1},0,\dots,0)$ for the first $t$ rows as
    \begin{align}
         \|\mathbf{B}^p_{:t}\|_F^2 = \sum_{m=1}^{t} \frac{1}{m^2}\sum_{j=0}^{m-1} b_j^2,
    \end{align}
    where $b_j = \sum_{i=0}^{\min\{j,p-1\}} \tilde{r}_i$. Since $\sum_{i=0}^\infty \tilde{r}_i = 0$, we get $\sum_{i=0}^j \tilde{r}_i = \sum_{i=j+1}^{\infty}|\tilde{r}_i|$. Then, we can apply integral bound on $b_j$ for $1\le j \le p-1$ (since for $j> p-1$ we have $b_j=b_{p-1}$) to get
    \begin{align}
        b_j = \sum_{i=0}^j \tilde{r}_i = \sum_{i=j+1}^{\infty}|\tilde{r}_i| &\le \sum_{i=j+1}^{\infty}\frac{1}{(2i-1)\sqrt{\pi i}} \\
        &\le \frac{1}{\sqrt{\pi}}\sum_{i=j+1}^{\infty}\frac{1}{i^{3/2}} 
        \le \frac{1}{\sqrt{\pi}}\int_{j}^\infty x^{-3/2}dx 
        = \frac{2}{\sqrt{\pi j}}. \label{eq:c_tilde_bound}
    \end{align}
    So the Frobenius form can be bounded by
    \begin{align}
        \|\mathbf{B}^p_{:t}\|_F^2 
        = \sum_{m=1}^{t} \frac{1}{m^2}\sum_{j=0}^{m-1} b_j^2 
        &= \sum_{m=1}^{t} \frac{1}{m^2} + \sum_{m=2}^{n} \frac{1}{m^2}\sum_{j=1}^{m-1} b_j^2 \\
        &= \sum_{m=1}^{n} \frac{1}{m^2} + \sum_{m=2}^{p} \frac{1}{m^2}\sum_{j=1}^{m-1} b_j^2 + \sum_{m=p+1}^{t} \frac{1}{m^2}\sum_{j=1}^{m-1} b_j^2.
    \end{align}
    For the second sum we have
    \begin{align}
        \sum_{m=2}^{p} \frac{1}{m^2}\sum_{j=1}^{m-1} b_j^2 \le \frac{4}{\pi}\sum_{m=2}^{p} \frac{1}{m^2}\sum_{j=1}^{m-1} \frac{1}{j} = \frac{4}{\pi}\sum_{m=2}^{p} \frac{H_{m-1}}{m^2} \le \frac{6}{\pi}
    \end{align}
    where for the last inequality, we used the bound
    \begin{align}
    \sum_{m=2}^{\infty}\frac{H_{m-1}}{m^2} = \sum_{m=1}^{\infty}\frac{H_{m}}{m^2} - \sum_{m=1}^{\infty}\frac{1}{m^3} = 2\zeta(3) - \zeta(3) = \zeta(3) < \frac{3}{2},\label{eq:harmonic_over_sqr_bound}
    \end{align}
    which can be computed by plugging $m=2$ into equation~\eqref{eq:harmonic_number_riemann_zeta} from Proposition~\ref{pro:harmonicNumberRiemannZeta}.
    
    So the Frobenius norm bound becomes
    \begin{align}
        \|\mathbf{B}^p_{:t}\|_F^2 
        &\le \sum_{m=1}^{t} \frac{1}{m^2} + \sum_{m=2}^{p} \frac{1}{m^2}\sum_{j=1}^{m-1} b_j^2 + \sum_{m=p+1}^{t} \frac{1}{m^2}\sum_{j=1}^{m-1} b_j^2 \\ &\le \frac{\pi^2}{6} + \frac{6}{\pi} + \sum_{m=p+1}^{t} \frac{1}{m^2}\left(\sum_{j=1}^{p-1} b_j^2 + \sum_{j=p}^{m-1} b_j^2\right).
    \end{align}
    We then bound the last sum as
    \begin{align}
        \sum_{m=p+1}^{t} \frac{1}{m^2}\left(\sum_{j=1}^{p-1} b_j^2 + \sum_{j=p}^{m-1} b_j^2\right) 
        &\le 
        \sum_{m=p+1}^{t} \frac{1}{m^2}\left(\sum_{j=1}^{p-1} \frac{4}{\pi j} + \sum_{j=p}^{m-1} \frac{4}{\pi(p-1)}\right) \\
        & \ = \frac{4}{\pi}\sum_{m=p+1}^{t} \frac{1}{m^2} \left( H_{p-1} + \frac{m-p}{p-1}\right).
    \end{align}
    We can bound $\frac{4}{\pi}\sum_{m=p+1}^{t} \frac{1}{m^2} \left( H_{p-1} - \frac{p}{p-1}\right) \le \frac{6}{\pi}$ using the bound in equation~\eqref{eq:harmonic_over_sqr_bound}. For the other term we have
    \begin{align}
        \frac{4}{\pi}\sum_{m=p+1}^{t} \frac{1}{m(p-1)} \le \frac{4}{\pi(p-1)} \int_{p}^t \frac{dx}{x} = \frac{4}{\pi(p-1)}\ln \frac{t}{p}.
    \end{align}
    So the overall bound on the Frobenius norm is
    \begin{align}
        \|\mathbf{B}^p_{:t}\|_F^2 \le \frac{\pi^2}{6}+\frac{12}{\pi} + \frac{4}{\pi(p-1)}\ln \frac{t}{p}.
    \end{align}
    For sensitivity, we know from Theorem~2 of \citet{kalinin2025back} that for a matrix $\mathbf{A}_{\alpha,\beta}$ defined in equation~\eqref{eq:Aalphabeta}, the sensitivity of $\mathbf{C}_{1,\beta}$ has the bound below:
    \begin{align}
        \operatorname{sens}_{k,b}^{2}\!\left(\mathbf{C}^p_{1,\beta}\right) \le \frac{k(1+\beta)^{2}}{(1-\beta)^{6}}
\left( 2 + \ln p + 54 \frac{p}{b} \right).
    \end{align}
    Plugging in $\beta = 0$, and $k=\lceil\tfrac{n}{b}\rceil$, we get
    \begin{align}
        \operatorname{sens}_{k,b}^{2}\!\left(\mathbf{C}^p\right) \le k \left(2 + \ln p + 54 \frac{pk}{n}\right).
    \end{align}
    So the overall expected error is bounded by
    \begin{align}
        \mathcal{E}_t(\mathbf{B}^p, \mathbf{C}^p)^2 &\le \frac{k}{t} \left(2 + \ln p + 54 \frac{pk}{n}\right) \left( \frac{\pi^2}{6}+\frac{12}{\pi} + \frac{4}{\pi(p-1)}\ln \frac{t}{p}\right) \\
        &= O\left( \frac{k}{t}\left(\ln p + \frac{pk}{n}\right)\left(1 + \frac{1}{p}\ln \frac{t}{p}\right)\right).
    \end{align}
    Plugging in $p =\lceil\ln b \rceil$ gives $\ln p + \frac{pk}{n} = O\left(\ln \ln \frac{n}{k} + \frac{k}{n}\ln \frac{n}{k}\right)$ and $1 + \frac{1}{p}\ln \frac{n}{p} = O(\frac{\ln n}{\ln (n/k)})$, leading to the bound 
    \begin{align}
        \mathcal{E}_t(\mathbf{B}^p, \mathbf{C}^p)^2 = O\left(\frac{k \ln(n) \ln \ln (n/k)}{t\ln (n/k)} + \frac{k^2\ln n}{nt}\right).
    \end{align}

\end{proof}

\begin{restatable}{lemma}{PairwiseInverseSum}\label{lem:pairwise-inverse-sum}
The series $\sum_{0\le i < j \le k-1} \frac{1}{j-i}$
is bounded by 
\begin{align}
    \sum_{0\le i < j \le k-1} \frac{1}{j-i} \le k \ln k + (\gamma - 1)k + \frac{1}{2} + O\left(\frac{1}{k}\right).
\end{align}
\end{restatable}

\begin{proof}
We rewrite the series using harmonic number. By denoting $t=j-i$ we can write,
    \begin{align}
        \sum_{0\le i < j \le k-1} \frac{1}{j-i} = \sum_{t = 1}^{k-1} \sum_{i=0}^{k-1-t} \frac{1}{t} = \sum_{t=1}^{k-1} \frac{k-t}{t} = kH_{k-1} -(k-1) = kH_k - k. \label{eq:pairwise-inverse-sum}
    \end{align}
Then using the bound 
$\ln k + \gamma + \frac{1}{2k} - O(\frac{1}{k^2})
\le
H_{k}
\le
\ln k + \gamma + \frac{1}{2k}$ we can plug it into \eqref{eq:pairwise-inverse-sum} to get the result immediately.
\end{proof}

\begin{restatable}{lemma}{SumLnOverDiff}\label{lem:sum‐ln_over_diff}
The series $\sum_{0\le i<j\le k-1}\frac{\ln(j-i)}{j-i}$ satisfies the bound below
\begin{align}
    \sum_{0\le i < j \le k-1}\frac{\ln (j-i)}{j-i} = \frac{k\ln^2 k}{2} - k \ln k + O(k).
\end{align}
\end{restatable}

\begin{proof}
By denoting $t=j-i$, we write the sum of series as below
\begin{align}
    \sum_{0\le i < j \le k-1} \frac{\ln (j-i)}{j-i} &= \sum_{t = 2}^{k-1} \sum_{i=0}^{k-1-t} \frac{\ln t}{t} \\
    &= \sum_{t=2}^{k-1} \frac{k-t}{t}\ln t 
    = k\sum_{t=2}^{k-1} \frac{\ln t}{t} - \sum_{t=2}^{k-1} \ln t = k\sum_{t=2}^{k-1} \frac{\ln t}{t} - \ln ((k-1)!).
\end{align}
Since $\tfrac{\ln x}{x}$ is decreasing over $(e,\infty)$, using integral bound one can get
\begin{align}
     \frac{\ln^2 (k-1)}{2} - \ln^2 2
    \ \le\ 
    \sum_{t=2}^{k-1} \frac{\ln t}{t} 
    \ \le\ 
    \frac{\ln 2}{2} +\frac{\ln 3}{3} + \frac{\ln^2 k}{2} - \frac{\ln^2 3}{2} \label{eq:log-over-t-sum-bounds}
\end{align}
and for $\sum_{t=2}^{k} \ln t$ we use Stirling's approximation to get a bound
\begin{align}
    \ln (k!) = k\ln k - k + \frac{1}{2}\ln( 2\pi k) + o(1)
\end{align}
So overall one can see that
\begin{align}
\sum_{0\le i < j \le k-1} & \frac{\ln (j-i)}{j-i} \\
& \le \frac{k\ln^2 k}{2} - k \ln k + k \left(1 + \frac{\ln 2}{2} + \frac{\ln 3}{3} - \frac{\ln^2 3}{2} \right) + \frac{\ln k}{2} - \frac{\ln (2\pi)}{2} + o(1)    
\end{align}
and 
\begin{align}
     \sum_{0\le i < j \le k-1}\frac{\ln (j-i)}{j-i} \ge \frac{k\ln^2 (k-1)}{2} - k \ln k + \left(1-\frac{\ln^2 2}{2}\right)k + \frac{\ln k}{2} - \frac{\ln (2\pi)}{2} + o(1).
\end{align}
\end{proof}

\begin{restatable}{lemma}{LnRatioSum}\label{lem:ln-ratio-sum}
The series $\sum_{0\le i<j\le k-1}\frac{1}{j-i}\,\ln\frac{k-j}{k-i}$ satisfies the bounds below
\begin{align}
    -\frac{\pi^2k}{6} \le \sum_{0\le i<j\le k-1}\frac{1}{j-i}\,\ln\frac{k-j}{k-i} \le -\frac{\pi^2k}{6} + \left(\ln k + 1\right)\left(\ln k + 2 + \frac{1}{k-1}\right).
\end{align}
\end{restatable}

\begin{proof}

By denoting $t=j-i$ we write
\begin{align}
    \sum_{0\le i<j\le k-1}\frac{1}{j-i}\,\ln\frac{k-j}{k-i} 
    &= \sum_{t=1}^{k-1}\sum_{i=0}^{k-1-t}\frac{1}{t}\,\ln\frac{k-t-i}{k-i} \\
    &= \sum_{t=1}^{k-1}\frac{1}{t}\sum_{i=0}^{k-1-t}\,\ln\frac{k-t-i}{k-i} = -\sum_{t=1}^{k-1}\frac{1}{t}\ln\binom{k}{t}.
\end{align}
Using the bound for binomial coefficient (see \citet{cover2012elements}) 
\begin{align}
    \frac{1}{k+1}2^{k H(\frac{t}{k})} \le \binom{k}{t} \le 2^{k H(\frac{t}{k})},
\end{align}
where $H(m) = -m \log_2 m - (1-m)\log_2(1-m)$, we get
\begin{align}
    \frac{k\ln2}{t} &
    \left(-\frac{t}{k}\log_2 \left(\frac{t}{k}\right)-\frac{k-t}{k} 
    \log_2\left(\frac{k-t}{k}\right)\right) - \frac{\ln(k+1)}{t} \\
    &\qquad \qquad \qquad \le \frac{1}{t}\ln \binom{k}{t} \le \frac{k\ln2}{t}\left(-\frac{t}{k}\log_2 \left(\frac{t}{k}\right)-\frac{k-t}{k}\log_2\left(\frac{k-t}{k}\right)\right),
\end{align}
or equivalently
\begin{align}
    -\ln \left(\frac{t}{k}\right)+\ln\left(\frac{k-t}{k}\right)
    &-\frac{k}{t}\ln\left(\frac{k-t}{k}\right)  - \frac{\ln(k+1)}{t} \\
    &\le \frac{1}{t}\ln \binom{k}{t} 
    \le -\ln \left(\frac{t}{k}\right)+\ln\left(\frac{k-t}{k}\right)-\frac{k}{t}\ln\left(\frac{k-t}{k}\right).
\end{align}
Since $\sum_{t=1}^{k-1}\left(-\ln \left(\frac{t}{k}\right)+\ln\left(\frac{k-t}{k}\right)\right) = 0$, we get the bound
\begin{align}
    \sum_{t=1}^{k-1}\frac{k}{t}\ln\left(\frac{k-t}{k}\right)\le -\sum_{t=1}^{k-1}\frac{1}{t}\ln\binom{k}{t} \le \sum_{t=1}^{k-1}\frac{k}{t}\ln\left(\frac{k-t}{k}\right) +\sum_{t=1}^{k-1}\frac{\ln(k+1)}{t}.
\end{align}
We write $\sum_{t=1}^{k-1}\frac{k}{t}\ln\left(\frac{k-t}{k}\right)$ as
\begin{align}
    \sum_{t=1}^{k-1}\frac{k}{t}\ln\left(\frac{k-t}{k}\right) = \sum_{t=1}^{k-1}\frac{k}{t}\ln\left(1-\frac{t}{k}\right) = \sum_{t=1}^{k-1}f\left(\frac{t}{k}\right),
\end{align}
where $f(x) = \frac{\ln(1-x)}{x}$ is a negative and decreasing function. So integral bound gives us
\begin{align}
    \int_{\frac{1}{k}}^1 f(x) dx \le \frac{1}{k}\sum_{t=1}^{k-1}f\left(\frac{t}{k}\right) \le \int_{0}^{1-\frac{1}{k}} f(x) dx,
\end{align}
which gives
\begin{align}
    \int_{0}^1 f(x) dx \le \frac{1}{k}\sum_{t=1}^{k-1}f\left(\frac{t}{k}\right) \le \int_{0}^{1} f(x) dx - \int_{1-\frac{1}{k}}^{1} f(x) dx.
\end{align}
We use the fact that $ \int_{0}^1 \frac{\ln(1-x)}{x}dx = -\mathrm{Li}_2(1) = -\frac{\pi^2}{6}$ (see Proposition~\ref{pro:dilogarithm}), to write the bounds as
\begin{align}
    -\frac{\pi^2}{6} \le \frac{1}{k}\sum_{t=1}^{k-1}f\left(\frac{t}{k}\right) \le -\frac{\pi^2}{6} - \int_{1-\frac{1}{k}}^{1} f(x) dx.
\end{align}
For the upper bound we write
\begin{align}
    -\frac{\pi^2}{6} - \int_{1-\frac{1}{k}}^{1} f(x) dx 
    &\le -\frac{\pi^2}{6} + \frac{k}{k-1}\int_{1-\frac{1}{k}}^{1} -\ln(1-x) dx \\
    & \ = -\frac{\pi^2}{6} + \frac{k}{k-1}\left(\frac{\ln k +1}{k}\right) = -\frac{\pi^2}{6} + \frac{\ln k +1}{k-1}.
\end{align}
So, the bound for $\sum_{t=1}^{k-1}\frac{k}{t}\ln\left(\frac{k-t}{k}\right)$ is written as
\begin{align}
     -\frac{\pi^2k}{6}\le \sum_{t=1}^{k-1}\frac{k}{t}\ln\left(\frac{k-t}{k}\right) \le  -\frac{\pi^2k}{6} + \frac{k\ln k +k}{k-1}.
\end{align}
So $-\sum_{t=1}^{k-1}\frac{1}{t}\ln\binom{k}{t}$ is lower bounded by
\begin{align}
    -\sum_{t=1}^{k-1}\frac{1}{t}\ln\binom{k}{t} \ge -\frac{\pi^2k}{6}.
\end{align}
For the upper bound we write
\begin{align}
    -\sum_{t=1}^{k-1}\frac{1}{t}\ln\binom{k}{t} & \le -\frac{\pi^2k}{6} + \frac{k\ln k +k}{k-1} + \sum_{t=1}^{k-1}\frac{\ln(k+1)}{t} \\
    & \ = -\frac{\pi^2k}{6} + \frac{k\ln k +k}{k-1} + H_{k-1}\ln(k+1).
\end{align}
Using $H_{k-1} \le \ln k + 1$ we write
\begin{align}
    -\sum_{t=1}^{k-1}\frac{1}{t}\ln\binom{k}{t} & \le -\frac{\pi^2k}{6} + \left(\ln k + 1\right)\left(1 + \frac{1}{k-1}\right) + \left(\ln k + 1 \right)^2 \\
    & \ = -\frac{\pi^2k}{6} + \left(\ln k + 1\right)\left(\ln k + 2 + \frac{1}{k-1}\right),
\end{align}
which completes the proof.
\end{proof}

\begin{restatable}{lemma}{ADtpInvMulti}\label{lem:ADtpInvMulti}
    For mean estimation aware factorization $\mathbf{B}=\mathbf{AD}_{\mathrm{Toep}}^{-1}$ and $\mathbf{C}=\mathbf{D}_{\mathrm{Toep}}$ the expected approximation error at time $t$ is 
    \begin{align}
        \mathcal{E}_t(\mathbf{AD}_{\mathrm{Toep}}^{-1}, \mathbf{D}_{\mathrm{Toep}}) = O\left(\sqrt{\frac{k}{t}}\right)+O\left(\frac{k}{\sqrt{nt}}\sqrt{\ln k \ln n}\right).
    \end{align}
\end{restatable}

\begin{proof}
    To compute the sensitivity of a matrix $\mathbf{C} = \mathrm{LTT}(c_0,\dots,c_{n-1})$ such that $c_0\ge c_1\ge\cdots \ge c_{n-1}\ge 0$, we can write it according to Theorem~\ref{thm:sensTheorem} as
\begin{align}
    \mathrm{sens}_{k,b}^2(\mathbf{C}) = \sum_{i=0}^{k-1}\sum_{j=0}^{k-1}
  \bigl\langle \mathbf{C}_{[:,\,i b]}\,,\;\mathbf{C}_{[:,\,j b]}\bigr\rangle,
\end{align}
where $\mathbf{C}_{[:,t]}$ denotes the $t$-th column of matrix $\mathbf{C}$. For the case of $j>i$ one can write the scalar product as
\begin{align}
    \bigl\langle \mathbf{C}_{[:,\,i ]}\,,\;\mathbf{C}_{[:,\,j ]}\bigr\rangle = \sum_{m=0}^{\,n-1-j}c_{m}\,c_{\,j-i+m}.
\end{align}
Now if we let $\mathbf{C} = \mathbf{D}_{\mathrm{Toep}}$, then $c_m = \frac{1}{m+1}, \forall m\ge 0$, so for $j>i$,
\begin{align}
    \bigl\langle \mathbf{C}_{[:,\,i b]}\,,\;\mathbf{C}_{[:,\,j b]}\bigr\rangle 
    &= \sum_{m=0}^{n-1-jb} c_m c_{jb-ib+m} \\
    &= \sum_{m=1}^{n-jb} \frac{1}{m} \frac{1}{jb-ib+m} = \frac{1}{jb-ib}\sum_{m=1}^{n-jb} \left( \frac{1}{m} - \frac{1}{jb-ib+m} \right).
\end{align}
So by writing $\sum_{m=1}^{n-jb} \frac{1}{jb-ib+m} = H_{n-ib} - H_{jb-ib}$ we get
\begin{align}
    \bigl\langle \mathbf{C}_{[:,\,i b]}\,,\;\mathbf{C}_{[:,\,j b]}\bigr\rangle = \frac{1}{b(j-i)} \left( H_{n-jb} + H_{jb-ib} - H_{n-ib} \right).
\end{align}
Moreover, for the case $j=i$,
\begin{align}
    \bigl\langle \mathbf{C}_{[:,\,i b]}\,,\;\mathbf{C}_{[:,\,i b]}\bigr\rangle = \sum_{m=0}^{n-1-ib} c_m^2 = \sum_{m=1}^{n-ib} \frac{1}{m^2}.
\end{align}
The sensitivity can be written as below
\begin{align}
\mathrm{sens}_{k,b}^2(\mathbf{D}_{\mathrm{Toep}}) = \sum_{i=0}^{k-1} \sum_{m=1}^{n-ib} \frac{1}{m^2} + 2\sum_{0\le i < j \le k-1} \frac{1}{b(j-i)} \left( H_{n-jb} + H_{jb-ib} - H_{n-ib} \right).
\end{align}
For the second sum, using the harmonic number bound, we obtain (note that $n>(k−1)b$, so all the terms inside the logarithm are strictly positive).
\begin{align}
    H_{n-jb} + H_{jb-ib} - H_{n-ib} = \ln\left( \frac{(n-jb)(jb-ib)}{(n-ib)}\right) + O(1).
\end{align}
One can substitute $kb$ instead of $n$ in the equation above since $\ln\left( \frac{(n-jb)(jb-ib)}{(n-ib)}\right)$ is increasing in $n$ and $kb\ge n$, so
\begin{align}
    H_{n-jb} + H_{jb-ib} - H_{n-ib} &\le \ln\left( \frac{(k-j)(jb-ib)}{(k-i)}\right) + O(1) \\
    &= \ln \left( \frac{k-j}{k-i}\right) + \ln (j-i) + \ln b + O(1).
\end{align}
So rewriting the sum gives
\begin{align}
    &2\sum_{0\le i < j \le k-1} \frac{1}{b(j-i)} \left( H_{n-jb} + H_{jb-ib} - H_{n-ib} \right)
    \\
   & \qquad \qquad  \le 2\sum_{0\le i < j \le k-1} \frac{1}{b(j-i)} \left(\ln \left( \frac{k-j}{k-i}\right) + \ln (j-i) + \ln b + O(1) \right).
\end{align}
Now, by applying Lemmas~\ref{lem:pairwise-inverse-sum},~\ref{lem:sum‐ln_over_diff}, and~\ref{lem:ln-ratio-sum} to the sums above, we obtain
\begin{align}
    2 \! \! \! \sum_{0\le i < j \le k-1} \frac{1}{b(j-i)} \left( H_{n-jb} + H_{jb-ib} - H_{n-ib} \right) \le 2\frac{k\ln k \ln b}{b} + \frac{k\ln^2 k}{b} + O\left(\frac{k \ln b}{b}\right).
\end{align}
For the first sum we have $\sum_{i=0}^{k-1} \sum_{m=1}^{n-ib} \frac{1}{m^2} \le \tfrac{\pi^2}{6}k$, and therefore the bound on sensitivity is obtained as
\begin{align}
\mathrm{sens}_{k,b}^2(\mathbf{D}_{\mathrm{Toep}}) \le \frac{\pi^2k}{6} + \frac{2k\ln k \ln b}{b} + \frac{k\ln^2 k}{b} + O\left(\frac{k\ln b}{b}\right). \label{eq:sensDtpMulti}
\end{align}
For the Frobenius part, as computed in the proof of Lemma~\ref{lem:ADtpInv}, we know that $\|\mathbf{AD}_{\mathrm{Toep}}^{-1}\|_F^2 = O(1)$ so the expected error of the factorization after substituting $b \sim \frac{n}{k}$ would be
\begin{align}
    \mathcal{E}_t(\mathbf{AD}_{\mathrm{Toep}}^{-1}, \mathbf{D}_{\mathrm{Toep}}) = O\left(\sqrt{\frac{k}{t}}\right)+O\left(\frac{k}{\sqrt{nt}}\sqrt{\ln k \ln n}\right).
\end{align}
\end{proof}

\begin{restatable}{lemma}{GregorySumBound}\label{lem:gregory-sum-bound}
For $p\ge 5$, the sum of the series $\sum_{j=1}^{p-1} |G_j|$ is bounded by
\begin{align}
    1 - \frac{1}{\ln (p-1)} + \frac{\gamma}{\ln^2 (p-1)} \le  \sum_{j=1}^{p-1} |G_j| \le 1 - \frac{1}{\ln p} + \frac{1}{\ln^2 p},
\end{align}
and the series $\sum_{j=1}^{p-1} j|G_j|$ is bounded by
\begin{align}
\sum_{j=1}^{p-1} j|G_j| \le \frac{3p}{\ln^2 p}   ,
\end{align}
where $G_j$ is the $j$-th Gregory coefficient.
\end{restatable}

\begin{proof}
    Since $\sum_{j=1}^{\infty} |G_j| = 1$, we first give a bound for $\sum_{j=p}^{\infty} |G_j|$ and then subtract it from 1. Note that according to equation~\eqref{eq:greg_first_prop}, for $j\ge 5$ we have the bound,
    \begin{align}
        \frac{1}{j\ln^2 j} - \frac{2}{j\ln^3 j} \le |G_j| \le \frac{1}{j\ln^2 j} - \frac{2\gamma}{j\ln^3 j},
    \end{align}
    so we bound the series as
    \begin{align}
        \sum_{j=p}^{\infty} \left(\frac{1}{j\ln^2 j} - \frac{2}{j\ln^3 j} \right)\le \sum_{j=p}^{\infty}|G_j| \le \sum_{j=p}^{\infty} \left(\frac{1}{j\ln^2 j} - \frac{2\gamma}{j\ln^3 j}\right).
    \end{align}
    By using integral bound for this series we get
    \begin{align}
        \frac{1}{\ln p} - \frac{1}{\ln^2 p}&= \int_p^{\infty} \left(\frac{1}{x\ln^2 x} - \frac{2}{x\ln^3 x}\right) dx \\&\le \sum_{j=p}^{\infty}|G_j| \le \int_{p-1}^{\infty} \left(\frac{1}{x\ln^2 x} - \frac{2\gamma}{x\ln^3 x}\right) dx = \frac{1}{\ln(p-1)} - \frac{\gamma}{\ln^2(p-1)}.
    \end{align}
    Now by using $\sum_{j=1}^{p-1} |G_j| = 1 - \sum_{j=p}^{\infty} |G_j|$, we immediately get the desired result for the first series. 

    For the second series, we use the fact that $ \sum_{j=1}^{4} j|G_j| \le \sum_{j=2}^4 \frac{1}{\ln^2 j}$ and the bound $|G_j| \le \frac{1}{j\ln^2 j}$ from equation~\eqref{eq:greg_first_prop} to get
    \begin{align}
        \sum_{j=1}^{p-1} j|G_j| = \sum_{j=1}^{4} j|G_j| + \sum_{j=5}^{p-1}j|G_j| \le \sum_{j=2}^4 \frac{1}{\ln^2 j} + \sum_{j=5}^{p-1}\frac{1}{\ln^2 j},
    \end{align}
    which using Lemma~\ref{lem:InvLogSqrBound} is upper bounded by $\frac{3p}{\ln^2 p}$ and proves our claim.
\end{proof}

\begin{restatable}{lemma}{MonotoneFunc}\label{lem:monotone-func}
The function
\begin{align}
    f(x) = \left(\frac{p}{x+1}\right)^{\frac{1}{x-p}},
\end{align}
is monotonically increasing for $x\ge p$.
\end{restatable}
\begin{proof}
    Equivalently, we show that $g(x) = \ln f(x) = \frac{1}{x-p}(\ln p - \ln (x+1))$ is monotonically increasing. $g'(x)$ can be computed as
    \begin{align}
        g'(x) = \frac{\ln\!\left( \frac{x+1}{p} \right) - \frac{x-p}{x+1}}{(x-p)^2}.
    \end{align}
    Using the bound $\ln x \ge 1 - \frac{1}{x}$ we get $\ln \left( \frac{x+1}{p} \right) \ge 1 - \frac{p}{x+1} = \frac{x+1-p}{x+1} > \frac{x-p}{x+1}$. So $g'(x) > 0$, and $f(x)$ is increasing.
\end{proof}

\begin{restatable}{lemma}{BandedInvSeriesBound}\label{lem:banded-inv-series-bound}
If we make $\mathbf{D}_{\mathrm{Toep}}$ banded inverse, then for $\mathbf{C}^p = \mathrm{LTT}(c_0,\cdots,c_{n-1})$, each entry $c_l$, decays exponentially for $l\ge p$, i.e.
\begin{align}
    c_l = \sum_{j=1}^{p-1} |G_j| c_{l-j}\le c_{p-1} \alpha^{l-p}, \qquad \text{for\quad} \alpha = 1 - \frac{1}{4p}.
\end{align}
\end{restatable}

\begin{proof}
    We prove our claim by induction, we first prove the induction step for $l\ge 2p-1$. Suppose that the bound holds for $l-1,l-2,\dots,l-p+1$, we then write,
    \begin{align}
        c_l = \sum_{j=1}^{p-1} |G_j| c_{l-j} \le \sum_{j=1}^{p-1} |G_j| c_{p-1} \alpha^{l-j-p} = c_{p-1} \alpha^{l-p} \sum_{j=1}^{p-1} |G_j| \alpha^{-j}.
    \end{align}
    So it suffices to show that $\sum_{j=1}^{p-1} |G_j| \alpha^{-j}\le 1$. To show this, first note that using the bound $e^{-x} \ge 1-x$, we get
    \begin{align}
        \alpha^{-j} = \left(1 - \frac{1}{4p}\right)^{-j} = \left(1 + \frac{1}{4p - 1}\right)^{j}\le e^{\tfrac{j}{4p - 1}} 
        \le e^{\tfrac{4j}{15p}}\le 1 + \frac{5j}{16p},
    \end{align}
    where in the second inequality we used $p \ge 4$, which is true as we assumed $p =b$ to be large, in the last inequality we used the facts that $\tfrac{j}{p}\le \frac{p-1}{p} \le 1$ and that for $0\le x\le 1$ the inequality $e^{4x/15} \le 1 + \tfrac{5x}{16}$ holds. Applying this bound gives
    \begin{align}
        \sum_{j=1}^{p-1} |G_j| \alpha^{-j} \le \sum_{j=1}^{p-1} |G_j|\left(1 + \frac{5j}{16p}\right) = \sum_{j=1}^{p-1} |G_j| + \frac{5}{16p}\sum_{j=1}^{p-1} j|G_j|.
    \end{align}
    Now, by using Lemma~\ref{lem:gregory-sum-bound}, we obtain the bound
    \begin{align}
        \sum_{j=1}^{p-1} |G_j| + \frac{5}{16p}\sum_{j=1}^{p-1} j|G_j| 
        &\le 1 - \frac{1}{\ln p} + \frac{1}{\ln ^2 p} + \frac{5}{16p} \frac{3p}{\ln^2 p} \\
        & \ = 1 - \frac{1}{\ln p} + \frac{1}{\ln^2 p}\left(1+\frac{15}{16}\right).
    \end{align}
    So it suffices to show that
    \begin{align}
        \frac{1}{\ln p} \ge \frac{1}{\ln^2 p}\left(1+\frac{15}{16}\right),
    \end{align}
    or equivalently
    \begin{align}
        \ln p \ge 1 + \frac{15}{16},
    \end{align}
    which since $p=b$ and we are assuming $b$ to be sufficiently large, the inequality clearly holds. So we get $\sum_{j=1}^{p-1} |G_j| \alpha^{-j} \le 1$, and this proves the induction step for $l\ge 2p-1$.

    Next, we prove the induction base for $p\le l \le 2p-2$. Since all the coefficients in the recursion $c_l = \sum_{j=1}^{p-1} |G_j| c_{l-j}$ are positive and each $c_l$ is positive, we get
    \begin{align}
        c_l = \sum_{j=1}^{p-1} |G_j| c_{l-j} < \sum_{j=1}^{l} |G_j| c_{l-j}, \qquad \text{for\quad} l\ge p.
    \end{align}
    The term $\sum_{j=1}^{l} |G_j| c_{l-j}$ corresponds to the recursion defining the inverse of 
$\mathrm{LTT}(1,g_1,g_2,\dots,g_{n-1})$ with $g_j = -|G_j|$. Since the banded version of this recursion has fewer terms and these terms are decreasing, it follows that the coefficients $c_l$ in the banded case are smaller than in the non-banded case. 
Hence, we get $c_l \le \tfrac{1}{l+1}$. 
 So we only need to show that
    \begin{align}
        \alpha^{l-p} \ge \frac{p}{l+1},
    \end{align}
\begin{align}
    \alpha \ge \sup_{p\le l\le 2p-2} \left( \frac{p}{l+1}\right)^{\frac{1}{l-p}}.
\end{align}
    By Lemma~\ref{lem:monotone-func}, we know that the function on the right hand side is monotonically increasing so it suffices for $\alpha$ to be larger than $\left( \frac{p}{l+1}\right)^{\frac{1}{l-p}}$ for some $l\ge 2p-2$. We show that the condition is satisfied for $l=2p$, i.e.
    \begin{align}
         \left(1 - \tfrac{1}{4p} \right)^p = \alpha^p \ge \left( \frac{p}{2p+1} \right).
    \end{align}
    Using Bernoulli's Inequality $(1+x)^r \ge 1+rx$ for $r\ge 1 $ and $ x\ge -1$, we obtain,
    \begin{align}
        \left(1 - \tfrac{1}{4p} \right)^p \ge 1 - \frac{1}{4} > \frac{p}{2p+1},
    \end{align}
    which proves the induction base, and the Lemma's claim.
\end{proof}

\begin{corollary}[Monotonicity of the banded inverse coefficients]\label{cor:d_toep_banded_decreasing}
Let $\mathbf{C}^p=\mathrm{LTT}(c_0,\dots,c_{n-1})$ be the banded inverse of $\mathbf{D}_{\text{Toep}}$. Then $(c_\ell)_{\ell\ge 0}$ is strictly decreasing.
\end{corollary}

\begin{proof}
For $\ell\ge1$ the recursion gives
\[
c_\ell=\sum_{j=1}^{\min\{p-1,\ell\}} |G_j|\,c_{\ell-j}.
\]
Since $c_{\ell-j}\le c_{\ell-1}$ for all $j\ge1$, we have
\[
c_\ell \le c_{\ell-1}\sum_{j=1}^{\min\{p-1,\ell\}} |G_j|
< c_{\ell-1}\sum_{j=1}^{\infty}|G_j|=c_{\ell-1},
\]
hence $c_\ell<c_{\ell-1}$ for all $\ell\ge1$.
\end{proof}

\begin{restatable}{lemma}{MultiDtpBISR}\label{lem:MultiDtpBISR} Let $\mathbf{A}=\mathbf{B}^p\mathbf{C}^p$ be the banded inverse factorization for $\mathbf{C}=\mathbf{D}_{\mathrm{Toep}}$. Then the expected approximation error at time $t$ is
    \begin{align}
        \mathcal{E}_t(\mathbf{B}^p,\mathbf{C}^p) = \Theta\left(\sqrt{\frac{k}{t}}\right) + O\left(\sqrt{\frac{k \ln k}{t \ln^2 (n/k)}}\right).
    \end{align}
\end{restatable}

\begin{proof}
    To bound the Frobenius norm we use an approach similar to equation~\eqref{eq:FrobeniusNormADtpInv} in the proof of Lemma~\ref{lem:ADtpInv}. The Frobenius norm is computed as
    \begin{align}
        \|\mathbf{B}^p_{:t}\|_F^2 = \sum_{m=1}^t \frac{1}{m^2}\sum_{j=0}^{m-1} a_j^2  &\le \sum_{m=1}^t \frac{1}{m^2}\sum_{j=0}^3 a_j^2 + \sum_{m=5}^t \frac{1}{m^2} \sum_{j=4}^{m-1} a_j^2 \\
        &\le O(1) +  \sum_{m=5}^t \frac{1}{m^2} \sum_{j=4}^{m-1} a_j^2,
    \end{align}
    where $a_j$ is defined as 
    \begin{align}
        a_j = 1 + \sum_{m=1}^{\min\{p-1,j\}} g_m, \qquad \text{for}\quad g_m = -|G_m|.
    \end{align}
    Using the bound for $a_j$ in equation~\eqref{eq:a_k_bound}, and the fact that $a_j = a_{p-1}$ for $j\ge p$, we get
    \begin{align}
        \|\mathbf{B}^p_{:t}\|_F^2 &\le O(1) +  \sum_{m=5}^p \frac{1}{m^2} \sum_{j=4}^{m-1} a_j^2 + \sum_{m=p+1}^t \frac{1}{m^2} \sum_{j=4}^{m-1} a_j^2 \\
        & \ = O(1) +  \sum_{m=5}^p \frac{1}{m^2} \sum_{j=4}^{m-1} a_j^2 + \sum_{m=p+1}^t \frac{1}{m^2} \left( \sum_{j=4}^{p-1} a_j^2 + \sum_{j=p}^{m-1} a_j^2 \right).
    \end{align}
    Using equation~\eqref{eq:FrobeniusNormADtpInv}, we know that 
    \begin{align}
        \sum_{m=5}^p \frac{1}{m^2} \sum_{j=4}^{m-1} a_j^2 + \sum_{m=p+1}^t \frac{1}{m^2}  \sum_{j=4}^{p-1} a_j^2 = O(1).
    \end{align}
    For the sum $\sum_{m=p+1}^n \frac{1}{m^2}\sum_{j=p}^{m-1} a_j^2 $, we write
    \begin{align}
        \sum_{m=p+1}^t \frac{1}{m^2}\sum_{j=p}^{m-1} a_j^2  &\le \sum_{m=p+1}^t \frac{1}{m^2}\frac{m-p}{\ln^2 p}  \\
        &\ = \sum_{m=p+1}^t \frac{1}{m\ln^2 p} -\sum_{m=p+1}^t \frac{p}{m^2\ln^2 p} \le \sum_{m=p+1}^t \frac{1}{m\ln^2 p} \le \frac{\ln \frac{t}{p}}{\ln^2 p}.
    \end{align}
    Since $p=b$ and $\ln \frac{t}{p} =O( \ln k)$, so $\frac{\ln \frac{t}{p}}{\ln^2 p} = O\big(1 + \frac{\ln k}{\ln^2(n/k)}\big)$, and the Frobenius norm is $O\big(1 + \frac{\ln k}{\ln^2(n/k)}\big)$.

    Next, we bound the sensitivity by considering two cases $j\neq i$ and $j=i$,
    \begin{align}
        \mathrm{sens}_{k,b}^2(\mathbf{C}^p) &= \sum_{i=0}^{k-1}\sum_{j=0}^{k-1}
  \bigl\langle \mathbf{C}^p_{[:,\,i b]}\,,\;\mathbf{C}^p_{[:,\,j b]}\bigr\rangle \\
  &= \sum_{i=0}^{k-1}
  \bigl\langle \mathbf{C}^p_{[:,\,i p]}\,,\;\mathbf{C}^p_{[:,\,i p]}\bigr\rangle + 2 \sum_{0\le i<j\le k-1} \bigl\langle \mathbf{C}^p_{[:,\,i p]}\,,\;\mathbf{C}^p_{[:,\,j p]}\bigr\rangle.
    \end{align}
    For the case $j=i$ we write,
    \begin{align}
        \sum_{i=0}^{k-1}
  \bigl\langle \mathbf{C}^p_{[:,\,i p]}\,,\;\mathbf{C}^p_{[:,\,i p]}\bigr\rangle = \sum_{i=0}^{k-1} \sum_{m=0}^{\,n-1-i}c_{m}^2 \le \sum_{i=0}^{k-1}  \sum_{m=0}^{\,n-1-i} \frac{1}{(m+1)^2} = O(k). \label{eq:sens-ieqj-bisr}
    \end{align}
    For the case $j\neq i$ we write, \begin{align}
        \bigl\langle \mathbf{C}^p_{[:,\,ip ]}\,,\;\mathbf{C}^p_{[:,\,jp ]}\bigr\rangle 
        &= \sum_{m=0}^{\,n-1-jp}c_{m}\,c_{\,jp-ip+m} \\
        &= \sum_{m=0}^{p-1} c_m c_{jp-ip+m} + \sum_{m=p}^{n-1-jp}c_m c_{jp-ip+m} \\ 
        &  
        \le \sum_{m=0}^{p-1} c_m c_{p-1} \alpha^{jp-ip+m-p} + \sum_{m=p}^{n-1-jp}c_{p-1}^2 \alpha^{jp-ip+2m-2p}.
    \end{align}
    Where the second sum appears only when $n-1-jp\ge p$. Now since $p=b$ we get,
    \begin{align}
        \bigl\langle \mathbf{C}^p_{[:,\,ip ]}\,,\;\mathbf{C}^p_{[:,\,jp ]}\bigr\rangle 
        &\le \alpha^{jp-ip}c_{p-1}\left(\sum_{m=0}^{p-1} c_m \alpha^{m-p} + \sum_{m=p}^{n-1-jp}c_{p-1} \alpha^{2m-2p} \right),\\
        &\le \alpha^{jp-ip}c_{p-1}\left(\alpha^{-p}\sum_{m=0}^{p-1} \frac{\alpha^{m}}{m+1} + c_{p-1}\sum_{m=0}^{n-1-jp-p} \alpha^{2m} \right).
    \end{align}
    Using the series $\sum_{m=1}^{\infty} \tfrac{\alpha^m}{m} = -\ln(1-\alpha)$, we get the bound on the first sum,
    \begin{align}
        \bigl\langle \mathbf{C}^p_{[:,\,ip ]}\,,\;\mathbf{C}^p_{[:,\,jp ]}\bigr\rangle 
        &\le \alpha^{jp-ip}c_{p-1}\left(\alpha^{-p-1}\sum_{m=1}^{\infty} \frac{\alpha^{m}}{m} + c_{p-1}\sum_{m=0}^{n-1-jp-p} \alpha^{2m} \right), \\
        &\le \alpha^{jp-ip} c_{p-1} \left( -\alpha^{-p-1} \ln(1-\alpha) + c_{p-1} \frac{1-\alpha^{2(n-jp-p)}}{1-\alpha^2} \right), \\
        &\le \alpha^{jp-ip} c_{p-1} \left(  \frac{-\ln(1-\alpha)}{\alpha^{p+1}} +\frac{ c_{p-1}}{1-\alpha^2} \right).
    \end{align}
    So we obtain
    \begin{align}
        \sum_{0\le i<j\le k-1} \bigl\langle \mathbf{C}^p_{[:,\,i p]}\,,\;\mathbf{C}^p_{[:,\,j p]}\bigr\rangle =  c_{p-1}\left(  \frac{-\ln(1-\alpha)}{\alpha^{p+1}} +\frac{ c_{p-1}}{1-\alpha^2} \right)\sum_{0\le i<j\le k-1} \alpha^{jp-ip}.
    \end{align}
    Next, we compute $\sum_{0\le i<j\le k-1} \alpha^{jp-ip}$ as 
    \begin{align}
        \sum_{0\le i<j\le k-1} \alpha^{jp-ip} = \sum_{m = 1}^{k-1} \sum_{i=0}^{k-1-m} \alpha^{mp} = \sum_{m = 1}^{k-1} (k-m) \alpha^{mp} \le k\sum_{m = 1}^{k-1}  \alpha^{mp} \le \frac{k\alpha^p}{1-\alpha^p}.
    \end{align}
    Now, we write the bound as
    \begin{align}
        \sum_{0\le i<j\le k-1} \bigl\langle \mathbf{C}^p_{[:,\,i p]}\,,\;\mathbf{C}^p_{[:,\,j p]}\bigr\rangle &= c_{p-1}\left(  \frac{-\ln(1-\alpha)}{\alpha^{p+1}} +\frac{ c_{p-1}}{1-\alpha^2} \right) \frac{k\alpha^p}{1-\alpha^p} \\
        &= \frac{1}{1-\alpha^p}\left(\frac{-k\ln(1-\alpha)}{p\alpha} + \frac{k\alpha^p}{p^2(1-\alpha^2)}\right).
    \end{align}
    Using the bound $e^{-x}\ge 1-x$ we obtain $\alpha^p = (1-\frac{1}{4p})^p \le e^{-1/4}$, so $\frac{1}{1-\alpha^p}\le \frac{1}{1-e^{-1/4}}$. Also, since $-\ln(1-x) \le \frac{x}{1-x}$ we can bound the first term as,
    \begin{align}
        \frac{-k\ln(1-\alpha)}{p\alpha} \le \frac{k}{p\alpha}\cdot \frac{\alpha}{1-\alpha} = 4k.
    \end{align}
    For the second term,
    \begin{align}
        \frac{k\alpha^p}{p^2(1-\alpha^2)} = \frac{k\alpha^p}{p^2(\tfrac{1}{4p})(1+\alpha)} = \frac{4k}{p(\alpha^{-p} + \alpha^{-p + 1})}  \le \frac{2k}{p}.
    \end{align}
    So combining these two bounds gives us
    \begin{align}
        \sum_{0\le i<j\le k-1} \bigl\langle \mathbf{C}^p_{[:,\,i p]}\,,\;\mathbf{C}^p_{[:,\,j p]}\bigr\rangle \le \frac{1}{(1-e^{-1/4})}\left(4k+\frac{2k}{p}\right) = O(k).
    \end{align}
    With this and equation~\eqref{eq:sens-ieqj-bisr} we can bound the squared sensitivity by $O(k)$. This gives the approximation error bound as below:
    \begin{align}
        \mathcal{E}_t(\mathbf{B}^p, \mathbf{C}^p) = O\left(\sqrt{\frac{k}{t}} + \sqrt{\frac{k \ln k}{t \ln^2 (n/k)}}\right).
    \end{align}
    Overall with the fact that both matrices have 1 on their main diagonal, the sensitivity is at least $\Omega(k)$, and thus
    \begin{align}
        \mathcal{E}_t(\mathbf{B}^p, \mathbf{C}^p) = \Theta\left(\sqrt{\frac{k}{t}}\right) + O\left(\sqrt{\frac{k \ln k}{t \ln^2 (n/k)}}\right).
    \end{align}
\end{proof}

\begin{restatable}{corollary}{StatsErrorBoundDAoneSqrt}\label{cor:StatsErrorBoundDAoneSqrt}
    For the prefix sum based factorization $\mathbf{B}=\mathbf{DE}_1^{1/2}$ and $\mathbf{C}=\mathbf{E}_1^{1/2}$, in the multi-participation setting, the error $\|\widehat{\boldsymbol{\mu}}_t - \boldsymbol{\mu}_t \|_2$ in Theorem~\ref{thm:StatErrorBound} takes the form
    \begin{align}
     O\left(\frac{\sqrt{d\ln t \ln \tfrac{1}{\delta}}\left(\sqrt{k\ln n}+k\right) }{\varepsilon t}\left(\sqrt{d} + \sqrt{\ln\tfrac{1}{\beta}}\right)\right),
    \end{align}
    with probability at least $1-\beta$.
\end{restatable}

\begin{proof}
    We will use the result from Theorem~\ref{thm:StatErrorBound} to compute the probability bound. From the proof of Lemma~\ref{lem:DAoneSqrtMulti}, we know that $\mathrm{sens}_{k,b}(\mathbf{E}_1^{1/2}) \le \sqrt{k\ln n +  6k^2 + k}$. So, we just need to compute the the norm of the $t$-th row of the matrix $\mathbf{DE}_1^{1/2}$, which we will denote as $\mathbf{b}_t^\top = \frac{1}{t}(
        r_{t-1}, r_{t-2}, \cdots, r_0, 0, \cdots, 0)$. So its norm can be bounded as
    \begin{align}
        \|\mathbf{b}_t\|_2 = \sqrt{\frac{1}{t^2}\sum_{j=0}^{t-1}r_j^2} \le \frac{1}{t}\sqrt{1+\sum_{j=1}^{t-1}\frac{1}{\pi j}} \le \frac{1}{t}\sqrt{1+\frac{1}{\pi}+\frac{\ln t}{\pi}}.
    \end{align}
    So by substituting the norm and sensitivity along with $\sigma_{\varepsilon,\delta}=\frac{1}{\varepsilon}\sqrt{2\ln \frac{1.25}{\delta}}$, we get the overall error bound as
    \begin{align}
        \|\widehat{\boldsymbol{\mu}}_t - \boldsymbol{\mu}_t \|_2 = O\left(\frac{\sqrt{d\ln t \ln \tfrac{1}{\delta}}\left(\sqrt{k\ln n}+k\right) }{\varepsilon t}\left(\sqrt{d} + \sqrt{\ln\tfrac{1}{\beta}}\right)\right),
    \end{align}
    with probability at least $1-\beta$.
\end{proof}

\begin{restatable}{corollary}{StatsErrorTrueDAsqrt}\label{lem:StatsErrorTrueDAsqrt}
For prefix sum based factorization $\mathbf{B}=\mathbf{DA}_{1}^{1/2}$ and $\mathbf{C}=\mathbf{E}_{1}^{1/2}$, the error in Theorem~\ref{thm:StatErrorBoundTrueMean}, in the multi-participation setting has the form
\begin{align}
        O\left(\sqrt{\frac{d\zeta^2}{t}\ln\frac{d}{\beta}} +\frac{\sqrt{d\ln t \ln \tfrac{1}{\delta}}\left(\sqrt{k\ln n}+k\right) }{\varepsilon t}\left(\sqrt{d} + \sqrt{\ln\tfrac{1}{\beta}}\right)\right),
\end{align}
with probability at least $1-2\beta$.
\end{restatable}

\begin{proof}
    Follows directly by plugging in the result of Corollary~\ref{cor:StatsErrorBoundDAoneSqrt} into Theorem~\ref{thm:StatErrorBoundTrueMean}.
\end{proof}

\begin{restatable}{corollary}{StatsErrorBoundAoneSqrtBISR}\label{cor:StatsErrorBoundAoneSqrtBISR}
    For the banded inverse factorization of matrix $\mathbf{C}=\mathbf{E}_1^{1/2}$, the error bound for  $\|\widehat{\boldsymbol{\mu}}_t - \boldsymbol{\mu}_t \|_2$ is
    \begin{align}
     O\Bigg(\frac{1}{\varepsilon}\sqrt{dk \ln \frac{1}{\delta}}&\left(\sqrt{\ln\ln (n/k)} + \sqrt{\frac{k\ln (n/k)}{n}} \right) \\
     &\qquad  \qquad  \times \left(\frac{\sqrt{k}}{\sqrt{t\ln (n/k)}}+\frac{\sqrt{\ln\ln (n/k)}}{t}\right)\left(\sqrt{d} + \sqrt{\ln\tfrac{1}{\beta}}\right)\Bigg),
    \end{align}
    with probability at least $1-\beta$.
\end{restatable}

\begin{proof}
    From the proof of Lemma~\ref{lem:BISRAoneSqrt} we know that $\mathrm{sens}_{k,b}(\mathbf{C}^p)= O\left(\sqrt{k\ln \ln (n/k)} + \sqrt{\frac{k^2\ln (n/k)}{n}}\right)$. For the norm of $\mathbf{b}_t^\top=\frac{1}{t}(
        b_{t-1}, b_{t-2}, \cdots, b_0, 0, \cdots, 0)$ we use equation~\eqref{eq:c_tilde_bound} from the proof of Lemma~\ref{lem:BISRAoneSqrt} which shows that
    $\tilde{c}_{j} \le \tfrac{2}{\sqrt{\pi j}} $. Also note that for $j\ge p$, $b_j=b_{p-1}$. So we obtain,
    \begin{align}
        \|\mathbf{b}_t^\top\|_2 = \frac{1}{t}\sqrt{\sum_{j=0}^{t-1}b_j^2} \le \frac{1}{t}\sqrt{ (t-p+1)b_{p-1}^2 + \sum_{j=0}^{p-1}b_j^2} 
        &\le \frac{1}{t}\sqrt{ \frac{4(t-p+1)}{\pi (p-1)} + \frac{4H_{p-1}}{\pi}} \\
        & \ = O\left(\sqrt{\frac{1}{tp}}\right) + O\left(\sqrt{\frac{\ln p}{t^2}}\right).
    \end{align}
    Plugging these into Theorem~\ref{thm:StatErrorBound} gives
    \begin{align}
        \|\widehat{\boldsymbol{\mu}}_t  - \boldsymbol{\mu}_t \|_2 = O\Bigg(\frac{1}{\varepsilon}\sqrt{dk \ln \frac{1}{\delta}}& \left(\sqrt{\ln\ln (n/k)} + \sqrt{\frac{k\ln (n/k)}{n}} \right) \\ &\quad\times\left(\frac{\sqrt{k}}{\sqrt{t\ln (n/k)}}+\frac{\sqrt{\ln\ln (n/k)}}{t}\right)\left(\sqrt{d} + \sqrt{\ln\tfrac{1}{\beta}}\right)\Bigg),
    \end{align}
    with probability at least $1-\beta$. 
\end{proof}

\begin{restatable}{corollary}{StatsErrorTrueDAsqrtBISR}\label{lem:StatsErrorTrueDAsqrtBISR}
For the banded inverse factorization of matrix $\mathbf{C}=\mathbf{E}_1^{1/2}$, the error in Theorem~\ref{thm:StatErrorBoundTrueMean}, in the multi-participation setting has the form
\begin{align}
        O\Bigg(\sqrt{\frac{d\zeta^2}{t}\ln\frac{d}{\beta}} + & \frac{1}{\varepsilon} 
        \sqrt{dk \ln \frac{1}{\delta}}\left(\sqrt{\ln\ln (n/k)} + 
        \sqrt{\frac{k\ln (n/k)}{n}} \right) \\ 
        & \qquad \qquad \times \left(\frac{\sqrt{k}}{\sqrt{t\ln (n/k)}}+\frac{\sqrt{\ln\ln (n/k)}}{t}\right) 
        \left(\sqrt{d} + \sqrt{\ln\tfrac{1}{\beta}}\right)\Bigg),
\end{align}
with probability at least $1-2\beta$.
\end{restatable}

\begin{restatable}{corollary}{StatsErrorBoundADtpMulti}\label{cor:StatsErrorBoundADtpMulti}
    For mean estimation specific factorization $\mathbf{B}=\mathbf{AD}_{\mathrm{Toep}}^{-1}$ and $\mathbf{C}=\mathbf{D}_{\mathrm{Toep}}$, the error in Theorem~\ref{thm:StatErrorBound} for multi participation setting has the form
    \begin{align}
    \|\widehat{\boldsymbol{\mu}}_t - \boldsymbol{\mu}_t \|_2 = O\left( \frac{\sqrt{d\ln \frac{1}{\delta}}}{\varepsilon \sqrt{t}\ln t}\left(\sqrt{k} +  \frac{k}{\sqrt{n}}\sqrt{\ln k \ln n}\right) \left(\sqrt{d} + \sqrt{\ln\tfrac{1}{\beta}}\right) \right),
    \end{align}
    with probability at least $1-\beta$.
\end{restatable}

\begin{proof}
    Using equation~\eqref{eq:sensDtpMulti} we have the bound below for sensitivity,
    \begin{align}
        \mathrm{sens}_{k,b}(\mathbf{D}_{\mathrm{Toep}}) = O\left(\sqrt{k} +  \frac{k}{\sqrt{n}}\sqrt{\ln k \ln n}\right),
    \end{align}
    Next, we bound $\|\mathbf{b}_t\|_2$ as
    \begin{align}
        \|\mathbf{b}_t\|_2 = \frac{1}{t}\sqrt{\sum_{j=0}^{t-1} a_j^2},
    \end{align}
    where $a_j$ is defined in equation~\eqref{eq:a_k_def} and bounded in equation~\eqref{eq:a_k_bound} for $j\ge 4$ as $|a_j| \le \frac{1}{\ln j}$. Using this we obtain
    \begin{align}
        \|\mathbf{b}_t\|_2 \le \frac{1}{t}\sqrt{\sum_{j=4}^{t-1}\frac{1}{\ln^2 j}+O(1)},
    \end{align}
    which according to Lemma~\ref{lem:InvLogSqrBound} is bounded by
    \begin{align}
        \frac{1}{t}\sqrt{\sum_{j=4}^{t-1}\frac{1}{\ln^2 j}+O(1)} =O\left(\frac{1}{t}\sqrt{\frac{t}{\ln^2 t}}\right) = O\left(\frac{1}{\sqrt{t}\ln t}\right).
    \end{align}
    So, overall we get the bound
    \begin{align}
        \|\widehat{\boldsymbol{\mu}}_t - \boldsymbol{\mu}_t \|_2 = O\left( \frac{\sqrt{d\ln \frac{1}{\delta}}}{\varepsilon \sqrt{t}\ln t} \left(\sqrt{k} +  \frac{k}{\sqrt{n}}\sqrt{\ln k \ln n}\right)\left(\sqrt{d} + \sqrt{\ln\tfrac{1}{\beta}}\right) \right),
    \end{align}
    with probability at least $1-\beta$.
\end{proof}

\begin{restatable}{corollary}{StatsErrorTrueADtpMulti}\label{lem:StatsErrorTrueADtpMulti}
For mean estimation specific factorization $\mathbf{B}=\mathbf{AD}_{\mathrm{Toep}}^{-1}$ and $\mathbf{C}=\mathbf{D}_{\mathrm{Toep}}$, the error in Theorem~\ref{thm:StatErrorBoundTrueMean}, in the multi-participation setting has the form
\begin{align}
        O\left(\sqrt{\frac{d\zeta^2}{t}\ln\frac{d}{\beta}} +\frac{\sqrt{dk\ln \frac{1}{\delta}}}{\varepsilon \sqrt{t}\ln t} \left(\sqrt{d} + \sqrt{\ln\tfrac{1}{\beta}}\right)\right),
\end{align}
with probability at least $1-2\beta$.
\end{restatable}

\begin{restatable}{corollary}{StatsErrorBoundADtpMultiBISR}\label{cor:StatsErrorBoundADtpMultiBISR}
    For the banded inverse factorization of matrix $\mathbf{D}_{\mathrm{Toep}}$, the error in Theorem~\ref{thm:StatErrorBound} for multi participation setting has the form
    \begin{align}
    \|\widehat{\boldsymbol{\mu}}_t - \boldsymbol{\mu}_t \|_2 = O\left( \frac{\sqrt{dk\ln \frac{1}{\delta}}}{\varepsilon \sqrt{t \ln (n/k)}} \left(1 + \sqrt{\frac{n}{kt}}\right)\left(\sqrt{d} + \sqrt{\ln\tfrac{1}{\beta}}\right) \right)
    \end{align}
    with probability at least $1-\beta$.
\end{restatable}

\begin{proof}
    We showed in the proof of Lemma~\ref{lem:MultiDtpBISR} we showed that the sensitivity is bounded by $O(\sqrt{k})$. To bound the norm of $\mathbf{b}_t$, we use the same approach as in the proof of Lemma~\ref{lem:MultiDtpBISR},
    \begin{align}
        \|\mathbf{b}_t\|_2 
        = \frac{1}{t}\sqrt{\sum_{j=0}^{t-1} a_j^2} 
        &= \frac{1}{t}\sqrt{O(1) + \sum_{j=5}^{p-1} a_j^2 + \sum_{j=p}^{t-1} a_j^2} \\
        &\le \frac{1}{t}\sqrt{O(1) + \sum_{j=4}^{p-1} \frac{1}{\ln^2 j} + \sum_{j=p}^{t-1} a_j^2} \\
        &\le \frac{1}{t}\sqrt{O(1) + \frac{3p}{\ln^2 p} + \sum_{j=p}^{t-1} a_j^2} \\
        &\le \frac{1}{t}\sqrt{O(1) + \frac{3p}{\ln^2 p} + \frac{t-p}{\ln^2 p}} \\
        & = O\left(\frac{1}{t}\sqrt{\frac{t + p}{\ln^2 p}}\right) = O\left(\frac{1}{\sqrt{t}\ln p} + \frac{\sqrt{p}}{t\ln p}\right)
    \end{align}
    where in the first and second inequality we used equation~\eqref{eq:a_k_bound} and Lemma~\ref{lem:InvLogSqrBound}, respectively.
    So with probability at least $1-\beta$ we have
    \begin{align}
        \|\widehat{\boldsymbol{\mu}}_t - \boldsymbol{\mu}_t \|_2 = O\left( \frac{\sqrt{dk\ln \frac{1}{\delta}}}{\varepsilon \sqrt{t \ln (n/k)}} \left(1 + \sqrt{\frac{n}{kt}}\right)\left(\sqrt{d} + \sqrt{\ln\tfrac{1}{\beta}}\right) \right).
    \end{align}
\end{proof}

\begin{restatable}{corollary}{StatsErrorTrueADtpMultiBISR}\label{lem:StatsErrorTrueADtpMultiBISR}
For the banded inverse factorization of matrix $\mathbf{D}_{\mathrm{Toep}}$, the error in Theorem~\ref{thm:StatErrorBoundTrueMean}, in the multi-participation setting has the form
\begin{align}
        O\left(\sqrt{\frac{d\zeta^2}{t}\ln\frac{d}{\beta}} +\frac{\sqrt{dk\ln \frac{1}{\delta}}}{\varepsilon \sqrt{t \ln (n/k)}} \left(1 + \sqrt{\frac{n}{kt}}\right)\left(\sqrt{d} + \sqrt{\ln\tfrac{1}{\beta}}\right)\right),
\end{align}
with probability at least $1-2\beta$.
\end{restatable}

\AlgOneRootSquaredError*

\begin{proof}

    In the following, let $\mathbf{Z} \sim \mathcal N(0, \varsigma^2)$ with $\varsigma^2 = d\zeta^2 \cdot \sigma^2_{\varepsilon,\delta} \cdot \mathrm{sens}_{k,b}^2(\mathbf{C})$. By a triangle inequality we may decompose
    \begin{align}
        \|\widehat{\mathbf{Y}}_{:\tau} \! -\! \mathbf{1}_\tau \boldsymbol{\mu}^\top \|_F
        \leq \|\widehat{\mathbf{Y}}_{:\tau} \!- \! \mathbf{Y}_{:\tau} \|_F \!+\! \| \mathbf{Y}_{:\tau} \!-\! \mathbf{1}_\tau \boldsymbol{\mu}^\top \|_F 
        \le \|\mathbf{B}_{:\tau} \mathbf{Z}_{:\tau} \|_F + \| \mathbf{Y}_{:\tau} \! - \! \mathbf{1}_\tau \boldsymbol{\mu}^\top \|_F. 
    \end{align}

    The first term is $\|\mathbf{B}_{:t}\|_{op}$-Lipschitz with respect to the Frobenius and thus also the $\ell_2$-norm as shown below: 
    \begin{align}
        | \|\mathbf{B}_{:\tau} \mathbf{Z}\|_F \! - \! \|\mathbf{B}_{:\tau} \mathbf{Z}^\prime\|_F |
        \leq \|\mathbf{B}_{:\tau} \mathbf{Z} - \mathbf{B}_{:\tau} \mathbf{Z}^\prime \|_F
        \leq \|\mathbf{B}_{:\tau} ( \mathbf{Z} - \mathbf{Z}^\prime) \|_F
        \leq \|\mathbf{B}_{:\tau}\|_{op} \|\mathbf{Z} - \mathbf{Z}^\prime \|_F. 
    \end{align}

    We can bound the expectation as $\mathbb{E}[\|\mathbf{B}_{:\tau} \mathbf{Z}\|_F] \leq \sqrt{\mathbb{E}[\|\mathbf{B}_{:\tau} \mathbf{Z}\|_F^2]} = \varsigma \sqrt{d} \| \mathbf{B}_{:\tau} \|_F$ and it holds that $\|\mathbf{B}_{:\tau} \|_{op} \leq \|\mathbf{B}_{:\tau} \|_F$. Hence, by Gaussian Lipschitz concentration applied to the function $\|\mathbf{B}_{:\tau} \cdot \|_F$,
    \begin{align}
        \mathbb{P} & \left[\|\mathbf{B}_{:\tau} \mathbf{Z}\|_F \geq \varsigma \|\mathbf{B}_{:\tau}\|_F \left(\sqrt{d}  +  \sqrt{2\ln(1/\beta)}\right) \right] \\ & \qquad \qquad \qquad 
        \leq
        \mathbb{P} \left[\|\mathbf{B}_{:\tau} \mathbf{Z}\|_F \geq \mathbb{E}[\|\mathbf{B}_{:\tau} \mathbf{Z}\|_F] + \varsigma \|\mathbf{B}_{:\tau}\|_{op} \sqrt{2\ln(1/\beta)}  \right]
        \leq \beta. 
    \end{align}

    Let $\bar{\mathbf{X}}_t = \frac 1 t \sum_{j=1}^t \mathbf{x}_j$. Since Assumption \ref{ass:iidbounded} implies Assumption \ref{ass:iid} with constant $\zeta^2$, by Corollary \ref{cor:concsubgaussmean},
    \begin{align}
        \mathbb{P}\left[\|\bar{\mathbf{X}}_t - \boldsymbol{\mu}\|_2 \leq \sqrt{\frac{2d\zeta^2 \ln(2dt/\beta)}{t}}\right] 
        \geq 1 - \frac{\beta}{t}. 
    \end{align}

    Via a union bound this then yields that with probability at least $1-\beta$ we have
    \begin{align}
        \| \mathbf{Y}_{:\tau} - \mathbf{1}_\tau \boldsymbol{\mu}^\top \|_F
        &= 
        \sqrt{\sum_{t=1}^\tau \|\bar{\mathbf{X}}_t - \boldsymbol{\mu}\|_2^2} \\
        &\leq 
        \sqrt{2d\zeta^2 \ln(2dt/\beta) \sum_{t=1}^\tau \frac{1}{t}} \\
        &\leq O\left(\sqrt{d\zeta^2 \ln(2d\tau/\beta)\ln(\tau)}\right). 
    \end{align}

    By another union bound over the two events established above, with probability $1-2\beta$, 
    \begin{align}
        \| &\widehat{\mathbf{Y}}_{:\tau} - \mathbf{1}_\tau \boldsymbol{\mu}^\top \|_F\\
        &   \leq O\! \left(\sqrt{d\zeta^2 \ln(2d\tau/\beta)\ln(\tau)} + \varsigma \|\mathbf{B}_{:\tau}\|_F \left(\sqrt{d}  +  \sqrt{2\ln(1/\beta)}\right)\right) \\
        &  \leq O\!\left(\sqrt{d\zeta^2 \ln(2d\tau/\beta)\ln(\tau)} + d\zeta^2 \cdot \sigma^2_{\varepsilon,\delta} \cdot \mathrm{sens}_{k,b}^2(\mathbf{C}) \cdot  \|\mathbf{B}_{:\tau}\|_F \left(\sqrt{d}  +  \sqrt{2\ln(1/\beta)}\right)\right). 
    \end{align}

    We use that $(a+b)^2 \leq 2a^2 + 2b^2$ for all $a,b \in \mathbb{R}$ to get 
    \begin{align}
        \sum_{t = 1}^\tau \|\widehat{\boldsymbol{\mu}}_t - \boldsymbol{\mu}\|_2^2
        & = \|\widehat{\mathbf{Y}}_{:\tau} - \mathbf{1}_\tau \boldsymbol{\mu}^\top \|_F^2 \\
        &\leq O\left(d\zeta^2 \ln(2d\tau/\beta)\ln(\tau) + d^2\zeta^2 \cdot \sigma^2_{\varepsilon,\delta} \cdot \mathrm{sens}_{k,b}^2(\mathbf{C}) \cdot  \|\mathbf{B}_{:\tau}\|_F^2 \cdot \ln(1/\beta)\right).  
    \end{align}

    We obtain the result by dividing by $\tau$ and substituting the definition of $\mathcal E_\tau(\mathbf{B}, \mathbf{C})$.   
\end{proof}

\section{Exponential Withhold-Release}

\begin{algorithm}[H]
    \caption{WithholdReleaseEstimator($\mathbf{X}, \mathbf{U}, \mathbf{B}, \mathbf{C}, \varepsilon, \delta$)}
    \label{alg:continualestimator}
    \begin{algorithmic}[1]
        \Require Stream $\mathbf{X} \in \mathbb{R}^{n \times d}$, arrival pattern $\mathbf{U} \in [b]^n$, factorization $\mathbf{E_1} = \mathbf{B} \mathbf{C} \in \mathbb{R}^{b \times b}$, privacy parameters $\varepsilon, \delta \in (0,1)$
        \State $L \gets \lfloor \log_2(k) \rfloor$
        \State $\mathbb A \gets \emptyset \subseteq \{0,...,L\}$
        \State $(\varepsilon^\prime, \delta^\prime) \gets (\varepsilon, \delta)/(2L+2)$
        \State $(\varepsilon^{\prime\prime}, \delta^{\prime\prime}) \gets (\varepsilon^\prime/\sqrt{8d\ln(4d/\delta^\prime)}, \delta^\prime/(2d))$
        \State $(\mathbb{M}_l, \mathbb B_l) \gets (\emptyset, \emptyset)$ \hfill $\forall l \in \{0,...,L\}$
        \State $k_{0,u} \gets 0$ \hfill $\forall u \in [b]$
        \State 
        \ForAll{$t \in [n]$}
            \State $k_{t,u_t} \gets k_{t-1,u_t} + 1$
            \State $k_t \gets \max_{u \in [b]} k_{t,u}$
            \State
            \If{$k_{t, u_t} \in \{2^0, 2^1, 2^2, 2^3,..., 2^L\}$}
            \State $\ell \gets \log_2(k_{t, u_t})$
            \State 
            \State $\mathbb I_\ell \gets \begin{cases}
                \{1\} & \text{if $\ell = 0$} \\
                \{2^{\ell-1} + 1,..., 2^\ell\} & \text{if $\ell > 0$}
            \end{cases}$
            \State $\mathbf{z}_{\ell, u_t} \gets \frac{1}{|\mathbb I_\ell|}\sum_{i \in \mathbb I_\ell} ([\mathbf{X}]_{u_t})_i$
            \State
            \If{$\ell \in \mathbb A$}
                \State $\mathbb{M}_\ell \gets \mathbb{M}_\ell \cup \{u_t\}$
                \State $\mathbf{P}_\ell \gets (\Pi_{\hat\Pi_\ell}(\mathbf{z}_{\ell, 1}),..., \Pi_{\hat\Pi_\ell}(\mathbf{z}_{\ell, |\mathbb{M}_\ell|}), 0,...,0)^\top$ \hfill $\mathbf{P}_\ell \in \mathbb{R}^{b \times d}$
            \Else
                \State $\mathbb B_\ell \gets \mathbb B_\ell \cup \{u_t\}$
            \EndIf
            \EndIf
            \ForAll{$\ell \in \mathbb A^c$}
                \If{$\sum_{u \in [b]} k_{t,u} \wedge |\mathbb I_\ell| \geq |\mathbb I_\ell| \cdot 2K_c(d, \varepsilon^{\prime\prime}, \delta^{\prime\prime})$}
                    \State $\mathbb A \gets \mathbb A \cup \{\ell\}$
                    \State
                    \State $S_1,...,S_{K_c} \gets \operatorname{GreedyBinCovering}(k_{t,1},...,k_{t,b}, |\mathbb I_\ell|)$
                    \State $\mathbf{y}_{\ell,k} \gets \sum_{u \in S_k}\sum_{i=1}^{k_{t,u}} ([\mathbf{X}]_u)_i / \sum_{u \in S_k} k_{t,u}$ \hfill $\forall k \in [K_c]$
                    \State $\mathbf{Y}_\ell \gets (\mathbf{y}_{\ell,1},...,\mathbf{y}_{\ell, K_c})^\top$
                    \State $\tau^\prime_\ell \gets \sqrt{{2\zeta^2\ln(2LK_cd/\gamma)}/|\mathbb I_\ell|}$
                    \State $\tau_\ell \gets \sqrt{{2\zeta^2 \ln(2Ldb/\gamma)}/|\mathbb I_\ell|}$
                    \State $\hat \Pi_\ell \gets \operatorname{ProjectionInterval}(\mathbf{Y}_\ell, \tau^\prime_\ell, \tau_\ell, \varepsilon^{\prime}, \delta^{\prime})$
                    \State 
                    \State $\mathbf{G}_\ell \gets \mathbf{G} \sim \mathcal N(0,1)^{b\times d}$
                    \State $\sigma_\ell \gets \sigma_{\varepsilon^\prime, \delta^\prime} \operatorname{sens}(\mathbf{C}) \cdot \operatorname{diam}_{\ell_2}(\hat\Pi_\ell)/2$
                    \State $\mathbf{\Xi}_\ell \gets \sigma_\ell \cdot \mathbf{G}_\ell$
                \EndIf
            \EndFor
            \State $\mathbb{M}_{t, l} \gets \mathbb{M}_l$ \hfill $\forall l \in \{0,...,L\}$
            \State \textbf{Return:} $
                \tilde{\boldsymbol{\mu}}_t :=\begin{cases}
                    0 & \text{if $|\mathbb A| = 0$} \\
                    \frac{1}{\sum_{l \in \mathbb A} |\mathbb I_l| |\mathbb{M}_{t, l}|} \sum_{\ell \in \mathbb A} |\mathbb I_\ell| \left(\mathbf{P}_\ell^\top {\mathbf{e}_1}_{|\mathbb{M}_{t, \ell}|} + \mathbf{\Xi}_\ell^\top \mathbf{b}_{|\mathbb{M}_{t, \ell}|}\right) & \text{otherwise}
            \end{cases}$
        \EndFor
    \end{algorithmic}
\end{algorithm}

\newpage

Algorithm \ref{alg:continualestimator} implements the exponential withhold-release scheme introduced by \citet{george2024continual}. It generalizes their estimator in the sense that it allows for (potentially) unbounded observations in $\mathbb{R}^d$ with sub-Gaussian entries and uses Gaussian Matrix Factorization mechanisms instead of Binary Tree Mechanisms relying on Laplace noise. Its guarantees can therefore be stated for arbitrary factorizations of the prefix-sum matrix $\mathbf{E}_1 = \mathbf{B} \mathbf{C} \in \mathbb{R}^d$ (See Lemma \ref{lm:squarederror} and Theorem \ref{thm:rmse}). This would potentially allow to optimize the resulting statistical rate over factorizations. We leave this for future work and use the prefix-sum based factorization with $\mathbf{B} = \mathbf{E}_1^{1/2}$ and $C = \mathbf{E}_1^{1/2}$ to obtain the concrete rates in Lemma \ref{lm:averagemse2}.

To state Algorithm \ref{alg:continualestimator} and its guarantees we assume that $b$ users each provide exactly $k$ observations such that the resulting stream is $\mathbf{X} \in \mathbb{R}^{n \times d}$ with $n = kb$. For convenience, we let the observations a user $u \in [b]$ contributes to $\mathbf{X}$ be $[\mathbf{X}]_u \in \mathbb{R}^{k \times d}$. As in \citet{george2024continual}, the algorithm maintains $L + 1 = \lfloor\log_2(k)\rfloor +1$ mechanisms, which here are instances of the Matrix Factorization mechanism. The $l$-th mechanism uses the observations with indices $\mathbb{I}_l$ of each user $u \in [b]$. If these are not yet accessible, the observations covered by $\mathbb{I}_\ell$ provided so far are withheld. If observations of user $u \in [b]$ are released but mechanism $l$ is not active yet, $u$ is buffered in $\mathbb{B}_\ell$. Otherwise, $u$ is added to $\mathbb{M}_l$. Mechanism $l \in [L]$ is activated and added to $\mathbb{A}$ once the diversity condition holds. Then, enough observations are provided to estimate the crude mean up to error $\tilde O(\sqrt{d\zeta/2^{l-1}})$ and thus to compute the projection interval $\hat\Pi_l$ (See Section \ref{sec:projectioninterval}). Each mechanism privatizes the sum of the means $\mathbb{Z}_{l, u}$ over $u$ via adequate Gaussian noise. The weighted average of these sums yields the estimators $\tilde{\boldsymbol{\mu}}_t$, which are private by Lemma \ref{lm:dpcontinualestimator}. 

\vspace{0.3cm}
\begin{lemma}
    \label{lm:dpcontinualestimator}
    Algorithm \ref{alg:continualestimator} is $(\varepsilon, \delta)$-user-level differentially private (ULDP) for $\varepsilon, \delta \in (0,1)$.
\end{lemma}
\begin{proof}

    We start by showing that jointly releasing the projection intervals  $\hat\Pi_\ell$ for $\ell \in \{0,...,L\}$ is private. By construction of the bin covering algorithm, each user contributes observations to at most one of the bins $S_1,...,S_{K_c}$. Hence, changing one user affects at most one $Y_{\ell, k}$, which makes the guarantee in Lemma \ref{lem:projectioninterval} applicable. This means that each $\hat \Pi_\ell$ is $(\varepsilon^\prime, \delta^\prime)$-ULDP. There are $L+1$ of them and by basic composition releasing $(\hat\Pi_0,...,\hat\Pi_L)$ is $(\varepsilon/2, \delta/2)$-ULDP. 
    
    Given that the projection intervals are private, by basic composition it remains to show that the outputs $(\widehat{\boldsymbol{\mu}}_1,...,\widehat{\boldsymbol{\mu}}_T)$ are private. A change in one user can affect all $L+1$ Matrix Factorization mechanisms, but only one $\mathbf{z}_{\ell, j}$ per mechanism. We recenter the projected $\mathbf{z}_{\ell, j}$ to bound their sensitivity. Let the midpoint of $\hat\Pi_\ell$ be $\hat{\mathbf{m}}_\ell \in \mathbb{R}^d$. For all $\ell \in \{0,...,L\}$ 
    \begin{align}
        \|\Pi_{\hat\Pi_\ell}(\mathbf{z}_{\ell,j}) - \hat{\mathbf{m}}_\ell\|_2
        \leq \sqrt d \cdot \|\Pi_{\hat\Pi_\ell}(\mathbf{z}_{\ell,j}) - \hat{\mathbf{m}}_\ell\|_\infty
        \leq {\sqrt d}/{2} \cdot \operatorname{diam}_{L_\infty}(\hat \Pi_\ell). 
    \end{align}
    
    Each mechanism $\ell$ is thus $(\varepsilon^\prime, \delta^\prime)$-ULDP by the privacy of the Gaussian mechanism. Basic composition over the $L+1$ mechanisms yields the $(\varepsilon/2, \delta/2)$-ULDP guarantee of $(\widehat{\boldsymbol{\mu}}_1,...,\widehat{\boldsymbol{\mu}}_T)$.
\end{proof}

\subsection{Diversity Conditions}

The diversity conditions in Assumption \ref{ass:diversity} and \ref{ass:sufficientdiversity} below are adaptations of that in \cite{george2024continual}. For the timestep $t \in [n]$ they hold, the conditions ensure that users have provided enough observations to activate all the mechanisms necessary to obtain the utility guarantees of Algorithm \ref{alg:continualestimator} in Lemma \ref{lm:squarederror} and \ref{lm:averagemse2} and Theorem \ref{thm:rmse}.

\begin{assumption}{(Diversity).}
    \label{ass:diversity}
    Let $K_c := K_c(d, \varepsilon^{\prime\prime}, \delta^{\prime\prime})$ in Algorithm \ref{alg:continualestimator} be the smallest $K \geq 4$ s.t.\
    \begin{align}
        \frac{2LKe^{\varepsilon^{\prime\prime}}}{\delta^{\prime\prime}} \exp\left(-\frac{\varepsilon^{\prime\prime} K}{4\cdot16}\right) + 2L\exp\left(-\frac{K}{8\cdot16^2}\right)
        \leq \gamma. 
    \end{align}

    The assumption holds at time $t$ if $\exists S_1,...,S_{K_c} \subseteq [b]$ disjoint s.t.\ $\sum_{u \in S_\kappa} k_{t,u} \geq k_t/2$ for all $\kappa \in [K_c]$. 
        
\end{assumption}

Assumption \ref{ass:diversity} above requires that we can solve a bin covering problem optimally. This problem is NP-hard. We can provide a sufficient condition which is lose as it corresponds to a greedy constant-factor approximation algorithm (See Corollary \ref{cor:sufficiency}).

\begin{assumption}{(Sufficient Diversity).}
    \label{ass:sufficientdiversity}
    Let $\varepsilon^{\prime\prime}, \delta^{\prime\prime}, d, K_c$ be as in Assumption \ref{ass:diversity}. This condition holds at time $t \in [n]$ if 
    \begin{align}
        \sum_{u \in [b]} k_{t,u} \wedge k_t/2
        \geq k_t/2 \cdot 2K_c. 
    \end{align}
        
\end{assumption}

Assumption \ref{ass:diversity} may only hold for $b \geq K_c$ and Assumption \ref{ass:sufficientdiversity} for $b \geq 2K_c$. The diversity conditions also cannot be fulfilled for the first $t < K_c$ and $t < 2K_c$ timesteps, respectively.

\subsection{Greedy Bin Covering}

\begin{algorithm}
    \caption{GreedyBinCovering$(k_1,..., k_n , m)$:}
    \label{alg:greedybincovering}
    \begin{algorithmic}[1]
        \Require $k_1,...,k_b, m \geq 0$,
        \State $\kappa \gets 1$
        \ForAll{$u \in [b]$}
            \State $S_\kappa \gets S_\kappa \cup \{u\}$
            \If{$\sum_{u \in S_\kappa} k_u \wedge m \geq m$}
                \State $\kappa \gets \kappa + 1$
            \EndIf
        \EndFor
        \State \textbf{Return:} $S := \{S_1,...,S_{\kappa-1}\}$
    \end{algorithmic}
\end{algorithm}

\begin{corollary}
    \label{cor:sufficiency}
    Let $k_1,...,k_b \geq 0$, $m > 0$. By means of running Algorithm \ref{alg:greedybincovering} we have the following implication: 
    \begin{align}
        \sum_{i \in [b]} k_i \wedge m \geq m \cdot 2K
        \quad \Rightarrow \quad 
        \text{$\exists S_1,...,S_K \subseteq [b]$ s.t.\ $\sum_{i \in S_\kappa} k_i \geq m$ for all $\kappa \in [K]$}. 
    \end{align}
    
\end{corollary}

\begin{proof}

    Let the partition produced in Algorithm \ref{alg:greedybincovering} be $S_1,...,S_{\kappa-1} \subseteq [b]$ and let $S_{\kappa} = [b] \setminus \bigcup_{l \in [\kappa-1]} S_l$ be a potential last unclosed bin. By construction $S_1,...,S_{\kappa-1}, S_\kappa$ are disjoint and for all $l \in [\kappa-1]$ we have $\sum_{u \in S_l} k_u \wedge m \geq m$. It remains to prove that $K \leq \kappa-1$. For all $l \in [\kappa-1]$ we have 
    \begin{align}
        m 
        \leq \sum_{u \in S_l} k_u \wedge m
        < 2m.
    \end{align}

    The first inequality holds as it is the condition necessary to close a bin $S_\kappa$. The second inequality holds, as $k_{u^\star} \wedge m \leq m$ for the last item $u^\star \in [b]$ added to $S_\kappa$ and the partial sum $\sum_{u \in S_\kappa \setminus \{u^\star\}} k_u \wedge m < m$. 
    
    For $S_{\kappa}$ we have $\sum_{u \in S_\kappa} k_u \wedge m < m$. By assumption and since the $S_1,...,S_{\kappa-1}, S_\kappa$ are disjoint, 
    \begin{align}
        2Km
        \leq \sum_{u \in [b]} k_u \wedge m
        = \sum_{l \in [\kappa-1]}\sum_{u \in S_l} k_u \wedge m + \sum_{u \in S_\kappa} k_u \wedge m 
        < (\kappa-1) 2m + m. 
    \end{align}

    Dividing by $m$ yields $2K < 2\kappa-1$. Since $2K, 2\kappa \in \mathbb{N}$, this implies $K \leq \kappa-1$. 
\end{proof}

\subsection{Private Projection Interval}
\label{sec:projectioninterval}

In this subsection we introduce the projection interval algorithm used to limit the sensitivity of (potentially) unbounded observations within Algorithm \ref{alg:continualestimator}. The projection interval procedure in Algorithm \ref{alg:projectioninterval} is adjusted from \cite{avellamedina2025meanest}. The results herein rely on the notion of $(\tau, \gamma)$-concentration that was introduced by \cite{levy2021uldp}. 

\begin{assumption}
    \label{ass:taugammaconc}

    A matrix $\mathbf{X} := (\mathbf{x}_1,...,\mathbf{x}_n)^\top \in \mathbb{R}^{n \times d}$ is $(\tau, \gamma)$-concentrated around $\boldsymbol{\mu} \in \mathbb{R}^d$ if for all $i \in [n]$ with probability at least $1-\gamma/n$, $\|{\mathbf{x}_i - \boldsymbol{\mu}}\|_\infty \leq \tau$. We call $\tau$ the concentration radius.
\end{assumption}

\begin{algorithm}
    \caption{ProjectionInterval$(\mathbf{Y}, \tau^\prime, \tau, \varepsilon, \delta)$:}
    \label{alg:projectioninterval}
    \begin{algorithmic}[1]
        \Require $\mathbf{Y} \in \mathbb{R}^{K \times d}$, $\tau^\prime, \tau > 0$, $\varepsilon, \delta \in (0,1)$
        \State $(\varepsilon^\prime, \delta^\prime) \gets (\varepsilon/\sqrt{8d\ln(2/\delta)}, \delta/(2d))$
        \ForAll{$j \in [d]$}
        \State $B_k \gets (2 k \tau^\prime \pm \tau^\prime]$ \hfill $\forall k \in \mathbb{Z}$
        \State $(\dots,\tilde p_{-1},\tilde p_0,\tilde p_1,\dots) \gets \operatorname{StableHistogram} (\mathbf{Y}_{:,j}, (B_k)_{k\in \mathbb{Z}}, \varepsilon, \delta)$ \hfill \# Algorithm 2 in \citet{avellamedina2025meanest}
        \State $\hat k \gets \arg\max_{k \in \mathbb{Z}} \tilde p_k$
        \State $\hat{m}_j \gets 
            \begin{cases}
                2 \hat k \tau^\prime & \text{if $\exists k \in \mathbb{Z} : \tilde p_k > 0$} \\
                0 & \text{otherwise}
            \end{cases}$
        \State $\hat I_j \gets [\hat{m}_j \pm 2\tau^\prime + \tau]$
        \EndFor
        \State \textbf{Return:} $\hat \Pi := \hat I_1 \times ... \times \hat I_d$
    \end{algorithmic}
\end{algorithm}

\begin{lemma}{(Lemma 3.4 from \citet{avellamedina2025meanest}).}
    \label{lem:privmidpoint}
    Let $\hat{m}_j$ be the $j$-th midpoint in Algorithm \ref{alg:projectioninterval}. The midpoint $\hat{m}_j$ is $(\varepsilon^\prime, \delta^\prime)$-DP. Let $\mathbf{Y}_{:,j} \in \mathbb{R}^K$ have \iid{} rows and be $(\tau, \gamma)$-concentrated around $\boldsymbol{\mu}_j \in \mathbb{R}$ with $\gamma \in (0,1 \wedge \frac K 4)$. Then, 
    \begin{align}
        \mathbb{P}\left[\hat{m}_j \in \Big[\boldsymbol{\mu}_j \pm 2\tau^\prime \Big]\right] \geq 1 - \left(1+\frac{e^\varepsilon}{\delta}\right) K \exp\left(-\frac{\varepsilon K}{4\cdot16}\right) - 2\exp\left(-\frac{K}{8\cdot16^2}\right).
    \end{align}
\end{lemma}

\begin{lemma}
    \label{lem:projectioninterval}

    Algorithm \ref{alg:projectioninterval} is $(\varepsilon, \delta)$-DP for $\varepsilon, \delta \in (0,1)$. Let $\mathbf{Y} \in \mathbb{R}^{K \times d}$ and $\mathbf{Z} \in \mathbb{R}^{n \times d}$ have \iid{} rows and be $(\tau^\prime, \gamma)$ and $(\tau, \gamma)$-concentrated around $\boldsymbol{\mu} \in \mathbb{R}^d$ with $\gamma \in (0,1 \wedge \frac K 4)$. Then, for $\hat\Pi$ in Algorithm \ref{alg:projectioninterval} computed from $\mathbf{Y}$, 
    \begin{align}
        \mathbb{P}\left[\forall i \in [n] : \mathbf{Z}_i \in \hat \Pi\right]
        \geq 1 - 2d\gamma - \frac{2Kde^{\varepsilon^\prime}}{\delta^\prime} \exp\left(-\frac{\varepsilon^\prime K}{4\cdot16}\right) - 2d\exp\left(-\frac{K}{8\cdot16^2}\right).
    \end{align}
\end{lemma}

\begin{proof}   

    By Lemma \ref{lem:privmidpoint} each $\hat{\mathbf{m}}_j$ in Algorithm \ref{alg:projectioninterval} is $(\varepsilon^\prime, \delta^\prime)$-DP. The output of Algorithm \ref{alg:projectioninterval} that combines the $\hat{\mathbf{m}}_j$ for all $j \in [d]$ then is $(\varepsilon, \delta)$-DP by advanced composition. 

    For utility, we want to apply Lemma \ref{lem:privmidpoint} to obtain a guarantee for the $\hat{\mathbf{m}}_j$. Because $\mathbf{Y} =: (\mathbf{Y}_1,...,\mathbf{Y}_K)^\top$ is $(\tau^\prime, \gamma)$-concentrated around $\boldsymbol{\mu}$, $\mathbf{Y}_{:,j}$ is $(\tau^\prime, \gamma)$-concentrated around $\boldsymbol{\mu}_j$. We may indeed apply the lemma and for all $j \in [d]$, 
    \begin{align}
        \mathbb{P}\left[|\hat{m}_j - \boldsymbol{\mu}_j| \leq 2\tau^\prime\right]
        &\geq \mathbb{P}\left[\hat{m}_j \in \Big[\boldsymbol{\mu}_j \pm 2\tau^\prime \Big]\right] \\
        &\geq 1 - \left(1+\frac{e^{\varepsilon^\prime}}{\delta^\prime}\right) K \exp\left(-\frac{\varepsilon^\prime K}{4\cdot16}\right) - 2\exp\left(-\frac{K}{8\cdot16^2}\right) \\
        &\geq 1 - \frac{2Ke^{\varepsilon^\prime}}{\delta^\prime} \exp\left(-\frac{\varepsilon^\prime K}{4\cdot16}\right) - 2\exp\left(-\frac{K}{8\cdot16^2}\right).
    \end{align}

    Here, the last inequality uses that $e^{\varepsilon^\prime} \geq \delta^\prime$. Combining the $(\tau, \gamma)$-concentration of $Z$ around $\boldsymbol{\mu}_j$ with a union bound,
    \begin{align}
        \mathbb{P}\left[\forall i \in [n] : |\mathbf{Z}_{ij} - \boldsymbol{\mu}_j| \leq \tau\right] \geq 1-\gamma. 
    \end{align}

    By another union bound over the two events above and a triangle inequality, we get
    \begin{align}
        \mathbb{P}\left[\forall i \in [n] : \mathbf{Z}_{ij} \in \hat I_j\right]
        &= \mathbb{P}\left[\forall i \in [n] : |\mathbf{Z}_{ij} - \hat{m}_j| \leq 2\tau^\prime + \tau\right] \\
        &\geq \mathbb{P}\left(\left\{\forall i \in [n] : |\mathbf{Z}_{ij} - \boldsymbol{\mu}_j| \leq \tau\right\} \cap \left\{|\hat{m}_j - \boldsymbol{\mu}_j| \leq 2\tau^\prime\right\}\right) \\
        &\geq 1 - \gamma - \frac{2Ke^{\varepsilon^\prime}}{\delta^\prime} \exp\left(-\frac{\varepsilon^\prime K}{4\cdot16}\right) - 2\exp\left(-\frac{K}{8\cdot16^2}\right).
    \end{align}

    We obtain the result via a union bound over $j \in [d]$. 
\end{proof}

\subsection{Error Bounds}

The error bounds below all rely on the following statement on the number of users that are withheld within Algorithm \ref{alg:continualestimator}. 

\begin{lemma}{(Claim 1 from \citet{george2024continual}).}
    \label{lm:activeobservations}
    Let $\mathbf{X} := (\mathbf{x}_1,...,\mathbf{x}_T)^\top \in \mathbb{R}^{n \times d}$ have an arbitrary arrival pattern $\mathbf{U} := (u_1,...,u_n)^\top \in [b]^n$. Then, for all $t \in [n]$, Algorithm \ref{alg:continualestimator} withholds at most $t/2$ observations. 
    
\end{lemma}

\AlgTwoSquaredError*

\begin{proof}

    Fix $t \in [n]$ and define the non-private estimator 
    \begin{align}
        \boldsymbol{\mu}_t
        := \frac{1}{\sum_{l \in \mathbb A} |\mathbb I_l| |\mathbb{M}_{t, l}|} \sum_{\ell \in \mathbb A} \sum_{u \in \mathbb{M}_{t, \ell}} |\mathbb I_\ell| \mathbf{z}_{\ell, u}
        = \frac{1}{\sum_{l \in \mathbb A} |\mathbb I_l| |\mathbb{M}_{t, l}|} \sum_{\ell \in \mathbb A} \sum_{u \in \mathbb{M}_{t, \ell}} \sum_{i \in \mathbb I_\ell} ([\mathbf{X}]_u)_i. 
    \end{align}
    
    By a triangle inequality, we may decompose the squared error of the private estimator $\tilde{\boldsymbol{\mu}}_t$ into
    \begin{align}
        \|\tilde{\boldsymbol{\mu}}_t - \boldsymbol{\mu}\|_2
        \leq \|\tilde{\boldsymbol{\mu}}_t - \boldsymbol{\mu}_t\|_2 + \|\boldsymbol{\mu}_t - \boldsymbol{\mu}\|_2.
    \end{align}
    
    We start by bounding the squared error of the non-private estimator $\boldsymbol{\mu}_t$ with high probability. Define the effective sample size $N_t := \sum_{l \in \mathbb A} |\mathbb I_l| |\mathbb{M}_{t, l}|$ and note that since the $([\mathbf{X}]_u)_i$ in $\boldsymbol{\mu}_t$ are \iid{}, 
    \begin{align}
        \boldsymbol{\mu}_t 
        \overset{d}{=} \frac{1}{N_t} \sum_{i=1}^{N_t} \mathbf{x}_i
        =: \bar{\mathbf{X}}_{N_t}. 
    \end{align}
    
    Let $\bar w(t)$ be the observations not withheld at time $t$. Given Assumption \ref{ass:diversity}, we can bound the effective sample size:
    \begin{align}
        N_t
        = \sum_{l \in \mathbb A} |\mathbb I_l| |\mathbb{M}_{t, l}|
        \overset{(i)}{=} \sum_{l \in [L_t]} |\mathbb I_l| |\mathbb{M}_{t, l}|
        \overset{(ii)}{=} \bar w(t)
        \overset{(iii)}{\geq} t/2. 
    \end{align}

    Above, $(i)$ holds as the active mechanisms are $\mathbb A = [L_t]$ with $L_t = \lfloor \log_2(k_t) \rfloor$, $(ii)$ holds as all buffers $\mathbb B_l$ are empty and $(iii)$ holds by Corollary \ref{lm:activeobservations}. We combine the distributional equality and $1/N_t \leq 2/t$ with Auxiliary Result \ref{cor:concsubgaussmean} to get
    \begin{align}
        \mathbb{P}\left[\|\boldsymbol{\mu}_t - \boldsymbol{\mu}\|_2 \leq \sqrt{\frac{4d\zeta^2 \ln(2/\gamma)}{t}}\right]
        \geq \mathbb{P}\left[\|\bar{\mathbf{X}}_{N_t} - \boldsymbol{\mu}\|_2 \leq \sqrt{\frac{2d\zeta^2 \ln(2/\gamma)}{N_t}}\right]
        \geq 1-d\gamma. 
    \end{align}

    It remains to bound the privacy error. Using Lemma \ref{lem:projectioninterval}, we show that the projection intervals are good and the projections do not affect any $\mathbf{z}_{\ell, u}$ with high probability. On that event $\mathcal E$, we then control the privacy noise of the active mechanisms. 

    The $(\mathbf{Y}_{\ell, k})_j$ are $\zeta^2/|\mathbb I_\ell|$-sub-Gaussian as they are averages of at least $|\mathbb I_\ell|$ \iid{} $\zeta^2$-sub-Gaussian observations. Thus, by Auxiliary Result \ref{cor:subgausstaugamma} $\mathbf{Y}_\ell \in \mathbb{R}^{K_c \times d}$ is $(\tau^\prime_\ell, \gamma/L)$-concentrated with
    \begin{align}
        \tau^\prime_\ell 
        = \sqrt{{2\zeta^2\ln(2LKd/\gamma)}/{|\mathbb I_\ell|}}
        \leq \sqrt{{2\zeta^2\ln(2\log_2(k)db/\gamma)}/{|\mathbb I_\ell|}}.
    \end{align}

    Here, we use that $K_c \leq n$, as no user can contribute observations to two bins during the bin covering.
    
    Likewise, as they are averages of $|\mathbb I_\ell|$ \iid{} $\zeta^2$-sub-Gaussians, the averages below are $\zeta^2/|\mathbb I_\ell|$-sub-Gaussian for all $j \in [d]$: 
    \begin{align}
        (\mathbf{z}_{\ell, u})_j = \frac{1}{|\mathbb I_\ell|}\sum_{i \in \mathbb I_\ell} ([\mathbf{X}]_u)_{ij}. 
    \end{align}
    
    We collect them in $\mathbf{Z}_\ell^n := (\mathbf{z}_{\ell, 1},..., \mathbf{z}_{\ell, n})^\top \in \mathbb{R}^{n \times d}$, which by Auxiliary Result \ref{cor:subgausstaugamma} is $(\tau_\ell, \gamma/L)$-concentrated with 
    \begin{align}
        \tau_\ell 
        = \sqrt{{2\zeta^2 \ln(2Ldb/\gamma)}/{|\mathbb I_\ell|}}
        \leq \sqrt{{2\zeta^2 \ln(2\log_2(k)db/\gamma)}/{|\mathbb I_\ell|}}.
    \end{align}

    We want to show that the event $\mathcal E$ below has high probability, as on $\mathcal E$ no projection happening up to time $t$ has an effect:
    \begin{align}
        \mathcal E 
        := \left\{\forall \ell \in \mathbb A, \forall u \in \mathbb{M}_{t, \ell} : \mathbf{z}_{\ell, u} \in \hat \Pi_\ell\right\}
        \supseteq \left\{\forall \ell \in [L], \forall u \in [b] : \mathbf{z}_{\ell, u} \in \hat \Pi_\ell\right\}. 
    \end{align}

    The inclusion holds as $\mathbb A \subseteq [L]$ and $\mathbb{M}_{t, \ell} \subseteq [n]$. Combining this with a union bound over mechanisms $\ell \in [L]$ and an application of Lemma \ref{lem:projectioninterval}, for $\gamma \in (0, 1 \wedge \frac K 4)$ we get the first two inequalities below:  
    \begin{align}
        \mathbb{P}\left[\mathcal E\right]
        &\geq 
        \mathbb{P}\left[\forall \ell \in [L], \forall u \in [b] : \mathbf{z}_{\ell, u} \in \hat \Pi_{\ell}\right] \\
        &\geq 1 - 2d\gamma - \frac{2LK_cde^{\varepsilon^{\prime\prime}}}{\delta^{\prime\prime}} \exp\left(-\frac{\varepsilon^{\prime\prime} K_c}{4\cdot16}\right) - 2Ld\exp\left(-\frac{K_c}{8\cdot16^2}\right)
        \geq 1-3d\gamma. 
    \end{align}

    Above, the last inequality holds by Assumption \ref{ass:diversity}. Note also that $K \geq 4$ and thus $\gamma \in (0,1)$. Collect the user-averages in $\mathbf{Z}_\ell := (\mathbf{z}_{\ell, 1},..., \mathbf{z}_{\ell, |\mathbb{M}_{t, \ell}|}, 0,...,0)^\top \in \mathbb{R}^{n \times d}$. Since $\Pi_{\hat \Pi_\ell}(\mathbf{z}_{\ell, u}) = \mathbf{z}_{\ell, u}$ for all $\ell \in \mathbb A, u \in \mathbb{M}_{t, \ell}$ on $\mathcal E$, we simplify as 
    \begin{align}
        \|\tilde{\boldsymbol{\mu}}_t - \boldsymbol{\mu}_t\|_2
        &= \frac{1}{\sum_{l \in \mathbb A} |\mathbb I_l| |\mathbb{M}_{t, l}|} \left\|\sum_{\ell \in \mathbb A} |\mathbb I_\ell| \left(\mathbf{P}_\ell^\top {\mathbf{e}_1}_{|\mathbb{M}_{t, \ell}|} + \mathbf{\Xi}_\ell^\top \mathbf{b}_{|\mathbb{M}_{t, \ell}|}\right) - \sum_{\ell \in \mathbb A} \sum_{u \in \mathbb{M}_{t, \ell}} |\mathbb I_\ell| \mathbf{z}_{\ell, u}\right\|_2 \\
        &= \frac{1}{\sum_{l \in \mathbb A} |\mathbb I_l| |\mathbb{M}_{t, l}|} \left\|\sum_{\ell \in \mathbb A} |\mathbb I_\ell| \left(\mathbf{Z}_\ell^\top {\mathbf{e}_1}_{|\mathbb{M}_{t, \ell}|} + \mathbf{\Xi}_\ell^\top \mathbf{b}_{|\mathbb{M}_{t, \ell}|}\right) - \sum_{\ell \in \mathbb A} \sum_{u \in \mathbb{M}_{t, \ell}} |\mathbb I_\ell| \mathbf{z}_{\ell, u}\right\|_2 \\
        &= \frac{1}{\sum_{l \in \mathbb A} |\mathbb I_l| |\mathbb{M}_{t, l}|} \left\|\sum_{\ell \in \mathbb A} |\mathbb I_\ell| \left(  \sum_{u \in \mathbb{M}_{t, \ell}} \mathbf{z}_{\ell, u} + \mathbf{\Xi}_\ell^\top \mathbf{b}_{|\mathbb{M}_{t, \ell}|}\!\right) \! - \! \sum_{\ell \in \mathbb A} \sum_{u \in \mathbb{M}_{t, \ell}} \! |\mathbb I_\ell| \mathbf{z}_{\ell, u}\right\|_2 \\
        &= \frac{1}{\sum_{l \in \mathbb A} |\mathbb I_l| |\mathbb{M}_{t, l}|} \left\|\sum_{\ell \in \mathbb A} |\mathbb I_\ell| \mathbf{\Xi}_\ell^\top \mathbf{b}_{|\mathbb{M}_{t, \ell}|}\right\|_2
        \leq \frac{1}{N_t} \sum_{\ell \in \mathbb A} |\mathbb I_\ell| \|\mathbf{\Xi}_\ell^\top \mathbf{b}_{|\mathbb{M}_{t, \ell}|}\|_2 \\
        &\leq \frac{2}{t} \sum_{\ell \in \mathbb A} |\mathbb I_\ell| \|\mathbf{\Xi}_\ell^\top \mathbf{b}_{|\mathbb{M}_{t, \ell}|}\|_2.
    \end{align}

    Above, ${\mathbf{e}_1}_t$ denotes the $t$-th row of $\mathbf{E}_1$. To control the norms, note that it holds that 
    \begin{align}
        \mathbf{\Xi}_\ell^\top \mathbf{b}_{|\mathbb{M}_{t, \ell}|} \sim \mathcal N(0, \sigma_\ell^2 \|\mathbf{b}_{|\mathbb{M}_{t, \ell}|}\|_2^2 I_d)
        \quad \text{and} \quad 
        \mathbb{E}[\|\mathbf{\Xi}_\ell^\top \mathbf{b}_{|\mathbb{M}_{t, \ell}|}\|_2] \leq \sigma_\ell \|\mathbf{b}_{|\mathbb{M}_{t, \ell}|}\|_2 \sqrt d. 
    \end{align}
    
    Thus, by Gaussian Lipschitz concentration applied to the $\ell_2$-norm we have 
    \begin{align}
        \mathbb{P}&\left[\|\mathbf{\Xi}_\ell^\top \mathbf{b}_{|\mathbb{M}_{t, \ell}|}\|_2 \geq \sigma_\ell \|\mathbf{b}_{|\mathbb{M}_{t, \ell}|}\|_2 \left(\sqrt d + \sqrt{2\ln(2/\alpha)}\right)\right] \\
        &\leq \mathbb{P}\left(\|\mathbf{\Xi}_\ell^\top \mathbf{b}_{|\mathbb{M}_{t, \ell}|}\|_2 - \mathbb{E}\left[\|\mathbf{\Xi}_\ell^\top \mathbf{b}_{|\mathbb{M}_{t, \ell}|}\|_2\right] \geq \sigma_\ell \|\mathbf{b}_{|\mathbb{M}_{t, \ell}|}\|_2\sqrt{2\ln(2/\alpha)}\right)
        \leq \alpha. 
    \end{align}

    A union bound over $\ell \in \mathbb A$, using that $|\mathbb A| \leq L \leq \log_2(k)$ and setting $\alpha \gets \alpha /L$ yield that with probability at least $1-\alpha$, 
    \begin{align}
        \|\mathbf{\Xi}_\ell^\top \mathbf{b}_{|\mathbb{M}_{t, \ell}|}\|_2 
        \leq \sigma_\ell \|\mathbf{b}_{|\mathbb{M}_{t, \ell}|}\|_2 \left(\sqrt d + \sqrt{2\ln(2\log_2(k)/\alpha)}\right)
        \quad \text{for all $\ell \in \mathbb A$}. 
    \end{align}

    Hence, on $\mathcal E$ with probability at least $1-\alpha$ it holds that
    \begin{align}
        \|\tilde{\boldsymbol{\mu}}_t - \boldsymbol{\mu}_t\|_2
        &\leq \frac{2}{t} \sum_{\ell \in \mathbb A} |\mathbb I_\ell| \|\mathbf{\Xi}_\ell^\top \mathbf{b}_{|\mathbb{M}_{t, \ell}|}\|_2 \\
        &\leq \frac{2}{t} \sum_{\ell \in \mathbb A} |\mathbb I_\ell| \sigma_\ell \|\mathbf{b}_{|\mathbb{M}_{t, \ell}|}\|_2 \left(\sqrt d + \sqrt{2\ln(2\log_2(k)/\alpha)}\right) \\
        &\leq \frac{1}{t}  \sigma_{\varepsilon^\prime, \delta^\prime} \operatorname{sens}(\mathbf{C}) \!  \left(\sqrt d \! + \!\sqrt{2\ln(2\log_2(k)/\alpha)}\right) \!  \sum_{\ell \in \mathbb A} |\mathbb I_\ell| \operatorname{diam}_{\ell_2}(\hat \Pi_\ell) \|\mathbf{b}_{|\mathbb{M}_{t, \ell}|}\|_2.  
    \end{align}

    The set $\hat\Pi_\ell$ is an $L_\infty$-ball in $\mathbb{R}^d$. We can bound its diameter in $\ell_2$-norm as follows: 
    \begin{align}
        \operatorname{diam}_{\ell_2}(\hat \Pi_\ell) 
        &\leq \sqrt d \operatorname{diam}_{L_\infty}(\hat \Pi_\ell) 
        = \sqrt d \left(4\tau^\prime_\ell + 2\tau_\ell\right)
        \lesssim \sqrt{{d\zeta^2 \ln(2\log_2(k)db/\gamma)}/{|\mathbb I_\ell|}}. 
    \end{align}
    
    Thus, using that $|\mathbb I_\ell| \leq k_t$ and setting $\alpha \gets d\gamma$, on $\mathcal E$ with probability at least $1-d\gamma$, 
    \begin{align}
        \|\tilde{\boldsymbol{\mu}}_t - \boldsymbol{\mu}_t\|_2
        &\leq \frac{1}{t} \cdot \sigma_{\varepsilon^\prime, \delta^\prime} \operatorname{sens}(\mathbf{C}) \left(\sqrt d + \sqrt{2\ln(2\log_2(k)/(d\gamma))}\right) 
        \\ & \qquad \qquad \qquad \qquad \times \sum_{\ell \in \mathbb A} |\mathbb I_\ell| \operatorname{diam}_{\ell_2}(\hat \Pi_\ell) \|\mathbf{b}_{|\mathbb{M}_{t, \ell}|}\|_2 \\
        &\leq \frac{1}{t}  \sigma_{\varepsilon^\prime, \delta^\prime} \operatorname{sens}(\mathbf{C})  \left(\sqrt d + \sqrt{2\ln(2\log_2(k)db/\gamma)}\right) 
        \\ &
         \qquad \qquad \qquad \qquad \times \sum_{\ell \in \mathbb A} |\mathbb I_\ell| \operatorname{diam}_{\ell_2}(\hat \Pi_\ell) \|\mathbf{b}_{|\mathbb{M}_{t, \ell}|}\|_2 \\
        &\lesssim \frac{1}{t}  \sigma_{\varepsilon^\prime, \delta^\prime} \operatorname{sens}(\mathbf{C})  \left(\sqrt d + \sqrt{2\ln(2\log_2(k)db/\gamma)}\right)  \\ &
         \qquad \qquad \qquad \qquad \times \sum_{\ell \in \mathbb A} |\mathbb I_\ell|^{1/2} \sqrt{d\zeta^2 \ln(2\log_2(k)db/\gamma)} \|\mathbf{b}_{|\mathbb{M}_{t, \ell}|}\|_2 \\
        &\lesssim \frac{\sqrt{k_t}}{t}  \sigma_{\varepsilon^\prime, \delta^\prime} \operatorname{sens}(\mathbf{C})  \left(\sqrt{d^2\zeta^2 \ln(2\log_2(k)db/\gamma)} \! + \! \sqrt{d\zeta^2}\ln(2\log_2(k)db/\gamma)\right)  \\ &
         \qquad \qquad \qquad \qquad  \times \sum_{\ell \in \mathbb A} \|\mathbf{b}_{|\mathbb{M}_{t, \ell}|}\|_2 \\
        &\lesssim \frac{\sqrt{k_t}}{t}  \sigma_{\varepsilon^\prime, \delta^\prime} \operatorname{sens}(\mathbf{C})  \sqrt{d^2\zeta^2}\ln(2\log_2(k)db/\gamma)  \sum_{\ell \in \mathbb A} \|\mathbf{b}_{|\mathbb{M}_{t, \ell}|}\|_2.
    \end{align}

    By a union bound over $\mathcal E$ and the concentration of the $\|\mathbf{\Xi}_\ell^\top \mathbf{b}_{|\mathbb{M}_{t, \ell}|}\|_2$ we obtain
    \begin{align}
        \mathbb{P}\left[
        \|\tilde{\boldsymbol{\mu}}_t - \boldsymbol{\mu}_t\|_2
        \leq \tilde O\left(\frac{1}{t} \cdot \sqrt{d^2\zeta^2k_t} \cdot \sigma_{\varepsilon^\prime, \delta^\prime} \operatorname{sens}(\mathbf{C}) \cdot \sum_{\ell \in \mathbb A} \|\mathbf{b}_{|\mathbb{M}_{t, \ell}|}\|_2\right)\right] 
        \geq 1 - 4d\gamma.
    \end{align}

    Combining this with the bound on $\|\boldsymbol{\mu}_t - \boldsymbol{\mu}\|_2$ via a union bound recovers the statement.
\end{proof}

\AlgTwoRootSquaredError*

\begin{proof}

    We may apply Lemma \ref{lm:squarederror} because we make Assumption \ref{ass:diversity} for all $t \in \{\tau_l,...,\tau_u\}$. Combining this with $(a+b)^2 \leq 2a^2 + 2b^2$ for all $a,b \in \mathbb{R}$, for all $t \in \{\tau_l,...,\tau_u\}$ with probability at least $1-5d\gamma$,
    \begin{align}
        \|\tilde{\boldsymbol{\mu}}_t \!- \! \boldsymbol{\mu}\|_2^2
        \lesssim \frac{d\zeta^2 \ln(2/\gamma)}{t} \! + \!\frac{d^2\zeta^2 k_t}{t^2}  \sigma^2_{\varepsilon^\prime, \delta^\prime} \operatorname{sens}^2(\mathbf{C})  \ln(2\log_2(k)db/\gamma)^2  \left(\sum_{\ell \in \mathbb A} \|\mathbf{b}_{|\mathbb{M}_{t, \ell}|}\|_2 \!\right)^2 \! .
    \end{align}

    Let the two terms on the right hand side be $T_1(t, \gamma)$ and $T_2(t, \gamma)$. By a union bound over the $t \in \{\tau_l,...,\tau_u\}$ and setting $\gamma \gets \gamma/\tau$, with probability at least $1-5d\gamma$ we have
    \begin{align}
        \sum_{t = \tau_l}^{\tau_u} \|\tilde{\boldsymbol{\mu}}_t - \boldsymbol{\mu}\|_2^2
        \lesssim \sum_{t = \tau_l}^{\tau_u} T_1(t, \gamma/\tau) + T_2(t, \gamma/\tau). 
    \end{align}

    We bound both terms separately, starting with the sum over the $T_1(t, \gamma/\tau)$:     
    \begin{align}
        \sum_{t = \tau_l}^{\tau_u} T_1(t, \gamma/\tau)
        \lesssim d\zeta^2 \ln(2\tau/\gamma) \sum_{t = \tau_l}^{\tau_u} \frac{1}{t}
        \leq d\zeta^2 \ln(2\tau/\gamma) H_\tau
        \leq d\zeta^2 \ln(2\tau/\gamma) (1+\ln(\tau)). 
    \end{align} 

    Here, $H_\tau$ is the $\tau$-th Harmonic number. For the sum over the $T_2(t, \gamma/\tau)$ we get
    \begin{align}
    \sum_{t = \tau_l}^{\tau_u} T_2(t, \gamma/\tau)
        &\lesssim \sum_{t = \tau_l}^{\tau_u} \frac{d^2\zeta^2 k_t}{t^2}  \sigma^2_{\varepsilon^\prime, \delta^\prime} \operatorname{sens}^2(\mathbf{C})  \ln(2\log_2(k)db\tau/\gamma)^2  \left(\sum_{\ell \in \mathbb A} \|\mathbf{b}_{|\mathbb{M}_{t, \ell}|}\|_2\right)^2 \\
        &= d^2\zeta^2  \sigma^2_{\varepsilon^\prime, \delta^\prime} \operatorname{sens}^2(\mathbf{C})  \ln(2\log_2(k)db\tau/\gamma)^2  \sum_{t = \tau_l}^{\tau_u} \frac{k_t}{t^2}  \left(\sum_{\ell \in \mathbb A} \|\mathbf{b}_{|\mathbb{M}_{t, \ell}|}\|_2\right)^2.
    \end{align}

    Collecting terms recovers the statement. 
\end{proof}

\AlgTwoRootSquaredErrorSimplified*

\begin{proof}

    The idea of the proof is to split the analysis at a critical timestep $t_c$. 

    For bounded observations $\|\mathbf{x}_i\|_\infty \leq \zeta$ we avoid the diversity condition for $t < 2K_c$ by activating the first two mechanisms $\ell \in \{0,1\}$ at initialization of Algorithm \ref{alg:continualestimator}, setting their $\hat \Pi_\ell = [\pm \zeta]^d$ and defining the $\tau_\ell, \tau_\ell^\prime, \mathbf{\Xi}_\ell$ as usual. Then, Assumption \ref{ass:sufficientdiversity} only has to hold for $t \geq 2K_c$ and  $\mathbf{z}_{\ell,u} \in \hat \Pi_\ell$ holds deterministically for all $\ell \in \{0,1\}$ and $u \in [b]$.
    
    To ensure that $\ell \in \{2,...,L\}$ are activated in time, we still rely on Assumption \ref{ass:sufficientdiversity}. The mechanism $\ell = 2$ should be activated once $k_t = 4$. For the round-robin arrival pattern where $\mathbf{U} := (1,2,3...,b, 1,2,3,...) \in \mathbb{N}^n$ that happens when $t = t_c := 3b+1$. For $t \geq t_c$, the pattern and $b \geq 2K_c$ ensure that $k_{t,u} \geq k_t - 1 \geq k_t/2$ and we get
    \begin{align}
        \sum_{u \in [b]} k_{t,u} \wedge k_t/2
        = \sum_{u \in [b]} k_t/2
        = b \cdot k_t/2
        \geq k_t/2 \cdot 2K_c. 
    \end{align}

    Hence, Assumption \ref{ass:sufficientdiversity} holds for $t \geq t_c$. Moreover, the $\mathbf{x}_{ij} \in [\pm \zeta]$ are $\zeta^2$-sub-Gaussian due to their boundedness. Now, by reasoning analogous to the proof of Lemma \ref{lm:squarederror}, with probability $1-5d\gamma$
    \begin{align}
        \|\tilde{\boldsymbol{\mu}}_t - \boldsymbol{\mu}\|_2
        &\leq \tilde O\left(\sqrt{\frac{d\zeta^2}{t}} + \frac{d\zeta}{t} \cdot \sqrt{k_t} \cdot \sigma_{\varepsilon^\prime, \delta^\prime} \operatorname{sens}(\mathbf{C}) \cdot \sum_{\ell \in \mathbb A} \|\mathbf{b}_{|\mathbb{M}_{t, \ell}|}\|_2\right). 
    \end{align}
    
    By squaring and summing these terms like in the proof of Theorem \ref{thm:rmse}, with probability $1-5d\gamma$, 
    \begin{align}
        \sum_{t = 1}^{\tau} \|\tilde{\boldsymbol{\mu}}_t - \boldsymbol{\mu}\|_2^2
        &\leq \tilde O\left(d\zeta^2 + d^2\zeta^2 \cdot \sigma^2_{\varepsilon^\prime, \delta^\prime} \operatorname{sens}^2(\mathbf{C}) \cdot \sum_{t = 1}^{\tau} \frac{k_t}{t^2} \cdot \left(\sum_{\ell \in \mathbb A} \|\mathbf{b}_{|\mathbb{M}_{t, \ell}|}\|_2\right)^2\right) \\
        &\leq \tilde O\left(d\zeta^2 + d^2\zeta^2 \cdot \sigma^2_{\varepsilon^\prime, \delta^\prime} \operatorname{sens}^2(\mathbf{C}) \cdot \sum_{t = 1}^{\tau} \frac{k_t}{t^2} \right). 
    \end{align}

    Above, the last inequality holds by the following upper bounds, which hold by reasoning as in the proof of Lemma \ref{lem:rmse_DA1sqrt_A1sqrt}: 
    \begin{align}
        \mathrm{sens}(\mathbf{C})  
        &= \| \mathbf{C} \|_{1\to2}
        = \|\mathbf{E}_1^{1/2}\|_{1\to 2} = \sqrt{\sum_{j=0}^{b-1} r_j^2}  \sim \sqrt{\frac{\ln b}{\pi}}, \\
        \|\mathbf{b}_{|\mathbb{U}_{t, \ell}|}\|_2
        &= \sqrt{\sum_{j = 0}^{|\mathbb{U}_{t, \ell}| - 1} r_j^2}
        \leq \sqrt{\sum_{j = 0}^{b - 1} r_j^2}
        \leq \sqrt{\frac{\ln b}{\pi} + \frac{\gamma + \ln(16)}{\pi}}.
    \end{align}

    Using that $k_t = \lceil t/b\rceil$ and $n = kb$ we may simplify the sum above
    \begin{align}
        \sum_{t = 1}^\tau \frac{k_t}{t^2}
        = \sum_{t = 1}^\tau \frac{\lceil t/b\rceil}{t^2}
        &\leq \sum_{t = 1}^\tau \frac{t/b + 1}{t^2}
        = \sum_{t = 1}^\tau \frac{1}{bt} + \frac{1}{t^2}
        = \frac 1 b \sum_{t=1}^\tau \frac 1 t + \sum_{t=1}^\tau \frac{1}{t^2} \\
        &< \frac{H_\tau}{b} + \frac{\pi^2}{6}
        = O\left(\frac{\ln \tau}{b} + 1\right)
        = O\left(\frac{k\ln \tau}{n} + 1\right). 
    \end{align}

    Substituting this expression, setting $\gamma \gets \beta/d$ and dividing by $\tau$ concludes the proof.     
\end{proof}

\begin{corollary}
    \label{cor:subgausstaugamma}
    Let $\mathbf{X} := (\mathbf{x}_1,...,\mathbf{x}_n)^\top \in \mathbb{R}^{n \times d}$ fulfill Assumption \ref{ass:iid}. Then, $\mathbf{X}$ is $(\tau, \gamma)$-concentrated around $\boldsymbol{\mu} \in \mathbb{R}^d$ with 
    \begin{align}
        \tau = \sqrt{2\zeta^2 \ln(2dn/\gamma)}. 
    \end{align}
    
\end{corollary}

\begin{proof}

    By the $\zeta^2$-sub-Gaussianity of the $\mathbf{x}_{ij}$ we have 
    \begin{align}
        \mathbb{P}\left[|\mathbf{x}_{ij} - \boldsymbol{\mu}_j| \geq t\right]
        \leq 2 \exp\left(-\frac{t^2}{2\zeta^2}\right). 
    \end{align}
    
    For all $i \in [n]$, a union bound over $j \in [d]$ then results in 
    \begin{align}
        \mathbb{P}\left[\|\mathbf{x}_i - \boldsymbol{\mu}\|_\infty \geq t\right]
        \leq 2d \exp\left(-\frac{t^2}{2\zeta^2}\right).
    \end{align}
    
    We conclude by choosing $\tau = t = \sqrt{2\zeta^2 \ln(2nd/\gamma)}$ such that the right hand side above becomes $\gamma/n$. 
\end{proof}

\begin{corollary}
    \label{cor:concsubgaussmean}
    Let $\mathbf{X} := (\mathbf{x}_1,...,\mathbf{x}_n)^\top \in \mathbb{R}^{n \times d}$ fulfill Assumption \ref{ass:iid}. Then, 
    \begin{align}
        \mathbb{P}\left[\|\bar{\mathbf{X}}_n - \boldsymbol{\mu}\|_2 \leq \sqrt{\frac{2d\zeta^2 \ln(2d/\alpha)}{n}}\right] 
        \geq 1 - \alpha. 
    \end{align}

\end{corollary}

\begin{proof}

    Since the $\mathbf{x}_{ij}$ are \iid{} $\zeta^2$-sub-Gaussian with mean $\mathbb{E}[\mathbf{x}_{ij}] = \boldsymbol{\mu}_j$, their average $(\bar{\mathbf{X}}_n)_j$ is $\zeta^2/n$-sub-Gaussian with mean $\boldsymbol{\mu}_j$. From this we get that for all $j \in [d]$, 
    \begin{align}
        \mathbb{P}\left[|(\bar{\mathbf{X}}_n)_j - \boldsymbol{\mu}_j| \geq t\right]
        \leq 2 \exp\left(-\frac{t^2}{2\zeta^2/n}\right). 
    \end{align}
    Through a union bound over $j \in [d]$ and using that $\|x\|_2 \leq \sqrt d \|x\|_\infty$ for all $x \in \mathbb{R}^d$ we then obtain 
    \begin{align}
        \mathbb{P}\left[\|\bar{\mathbf{X}}_n - \boldsymbol{\mu}\|_2 \geq \sqrt d t\right]
        &\leq \mathbb{P}\left[\|\bar{\mathbf{X}}_n - \boldsymbol{\mu}\|_\infty \geq t\right] \\
        &\leq \mathbb{P}\left[\exists j \in [d] : |(\bar{\mathbf{X}}_n)_j - \boldsymbol{\mu}_j| \geq t\right]
        \leq 2d \exp\left(-\frac{t^2}{2\zeta^2/n}\right).
    \end{align}

    We recover the statement by setting $t = \sqrt{2d\zeta^2\ln(2d/\alpha)/n}$. 
\end{proof}

\vfill

\end{document}